\documentclass{article}

\PassOptionsToPackage{colorlinks=true,linkcolor=black,citecolor=black,urlcolor=blue!50!black}{hyperref}
\usepackage[preprint]{log_2026}
\usepackage[numbers,sort&compress]{natbib}

\usepackage[utf8]{inputenc}
\usepackage[T1]{fontenc}
\usepackage{url}
\usepackage{booktabs}
\usepackage{array}
\usepackage{amsfonts}
\usepackage{amsmath}
\usepackage{amssymb}
\usepackage{amsthm}
\usepackage{mathtools}
\usepackage{microtype}
\usepackage{xcolor}
\usepackage{colortbl}
\usepackage{pifont}
\usepackage{graphicx}
\usepackage{tikz}
\usetikzlibrary{arrows.meta,positioning,fit,calc,backgrounds}

\usepackage{algorithm}
\usepackage{algpseudocode}
\usepackage{enumitem}
\usepackage{float}
\usepackage{subcaption}
\usepackage{placeins}
\newtheorem{theorem}{Theorem}
\newtheorem{proposition}[theorem]{Proposition}
\newtheorem{lemma}[theorem]{Lemma}
\newtheorem{corollary}[theorem]{Corollary}

\theoremstyle{remark}
\newtheorem{remark}{Remark}

\newcommand{\R}{\mathbb{R}}
\newcommand{\diag}{\operatorname{diag}}
\newcommand{\LN}{\operatorname{LN}}
\newcommand{\clamp}{\operatorname{clamp}}
\newcommand{\ReLU}{\operatorname{ReLU}}
\newcommand{\Lip}{\operatorname{Lip}}
\definecolor{tabrowblue}{HTML}{EEF5FC}
\definecolor{tabrowblueStrong}{HTML}{DDECF8}
\definecolor{tabrowgray}{HTML}{F4F4F4}
\definecolor{tabblue}{HTML}{1F5E9C}
\definecolor{tabred}{HTML}{C93D32}
\definecolor{tabgray}{HTML}{5F6368}

\newcommand{\tabmuted}[1]{\textcolor{tabgray}{#1}}

\newcolumntype{L}[1]{>{\raggedright\arraybackslash}p{#1}}

\title{Placing Degree Scales After LayerNorm}

\author{%
  Yash Vardhan Tomar \\
  Purdue University \\
  \texttt{tomar4@purdue.edu} \\
  \And
  Aryav Das \\
  Park Tudor High School \\
  \texttt{aryav30das@gmail.com} \\
}

\begin{document}

\maketitle

\begin{abstract}
Graph neural networks (GNNs) are widely used to learn node-selection policies on graphs, and most stack graph attention (GAT) blocks with LayerNorm. On degree-sensitive tasks, LayerNorm tends to remove the degree signal these models need to rank nodes. Much recent work addresses this by redesigning normalizers or aggregators, which changes what these components compute but does not ask where, relative to LayerNorm, a degree scale should be applied. In this paper, we show that the answer follows from a single algebraic fact about LayerNorm. When a positive per-node scale is applied before LayerNorm, LayerNorm divides it out, and it never reaches the model's output. Applied after LayerNorm, the same scale comes through and reaches the score head as magnitude. From this placement rule we derive PostDeg, a parameter-free inverse-degree scale that we add as the single change to a fixed GAT backbone. PostDeg multiplies each node's normalized representation by an inverse function of its degree, and we compare it against controls in the same position. At every evaluation size, PostDeg improves over the LayerNorm backbone on influence maximization, network dismantling, and maximum independent set, and these controls show where the improvement comes from. The same scale before LayerNorm stays at the backbone, as the absorption identity predicts, and a constant scale after LayerNorm stays there too on all but the most heavy-tailed graphs, so the effect needs both the position after LayerNorm and a degree-dependent scale. The exact form of the scale matters much less, so we recommend PostDeg, which needs no tuning.
\end{abstract}

\section{Introduction}
\label{sec:intro}

Graph neural networks (GNNs) are a standard tool for learning node-selection policies on graphs such as influence maximization (InfluMax), network dismantling (Dismantle), maximum independent set (MIS), and epidemic containment (Epidemic) \citep{khalil2017learning,kool2019attention,karalias2020erdos,cappart2023combinatorial}. Most of these models stack residual graph attention network (GAT) blocks \citep{velickovic2018gat,brody2022gatv2} with LayerNorm \citep{ba2016layernorm}. LayerNorm normalizes each node to stabilize training and is a standard part of modern deep GNNs.

On degree-sensitive node selection, these policies should rank nodes by topology, since high-degree hubs spread an epidemic faster and low-degree nodes are easier to add to an independent set. LayerNorm tends to suppress the degree signal they need. One line of GNN-normalization work has diagnosed this and answered it with better normalizers. GraphNorm \citep{cai2021graphnorm} rescales graph-wise feature statistics to speed up training. PairNorm \citep{zhao2020pairnorm} holds the total pairwise feature distance fixed against oversmoothing, and NodeNorm \citep{zhou2021nodenorm} rescales each node's features toward the same end. Another line of work builds degree into message passing directly, as in Principal Neighbourhood Aggregation (PNA) \citep{corso2020pna}, which scales each aggregator by a function of node degree.

All of these redesign \emph{what} the normalizer or aggregator computes. None asks \emph{where}, inside the residual block, a degree signal can get through normalization. Whether a positive per-node scalar survives or is absorbed depends on which side of LayerNorm it is on, and, to our knowledge, this \emph{placement} question has gone unexamined for GNNs. We found this gap surprising, since both LayerNorm and degree features are staples of modern GNNs. Recent work on large language models raises the same placement question in a different domain~\citep{wang2026scalevec}.

The answer follows from a single algebraic fact about LayerNorm. When a positive multiplier $a_i$ is applied before LayerNorm, it is divided out, $\mathrm{LN}(a_i z_i)\approx \mathrm{LN}(z_i)$ \citep{ba2016layernorm,salimans2016weightnorm,press2017weighttying}, so it does not reach a score head at the output. The same multiplier applied after LayerNorm reaches it as magnitude. To reach the score head, a degree signal must enter after LayerNorm (Figure~\ref{fig:placement-summary}). From this rule we derive PostDeg, a parameter-free inverse-degree scale that we add as the single change to a fixed GAT backbone. PostDeg improves over the LayerNorm backbone on every node-selection task and at every evaluation size, and controls in the same position show where the improvement comes from. The same scale placed before LayerNorm stays at the backbone, as the identity predicts, and two learned variants (PostDeg-L-FG, PostDeg-L-Adaptive) test whether the slot needs added capacity.

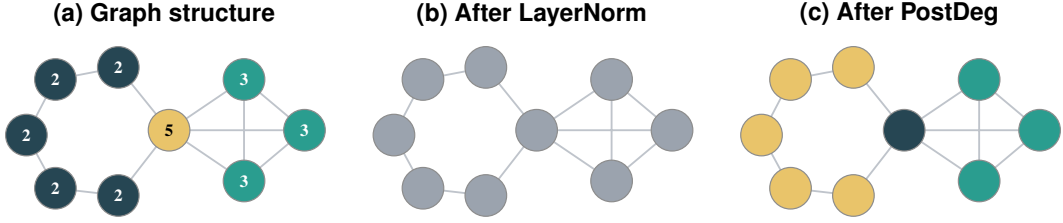
\begin{figure}[t]
\centering
\definecolor{magLo}{HTML}{264653}   
\definecolor{magMid}{HTML}{2A9D8F}  
\definecolor{magHi}{HTML}{E9C46A}   
\definecolor{magFlat}{HTML}{9BA3AE} 
\definecolor{edgegrey}{HTML}{BFC4CB}
\newcommand{\panelcoords}{%
  \coordinate (H)  at (0,0);
  \coordinate (P1) at (-1.05,1.30);
  \coordinate (P2) at (-2.35,1.05);
  \coordinate (P3) at (-2.95,-0.10);
  \coordinate (P4) at (-2.35,-1.20);
  \coordinate (P5) at (-1.05,-1.35);
  \coordinate (C1) at (1.55,1.05);
  \coordinate (C2) at (2.80,0.0);
  \coordinate (C3) at (1.55,-1.05);%
}
\newcommand{\paneledges}{%
  \draw[edgegrey,line width=1.1pt]
    (H)--(P1) (P1)--(P2) (P2)--(P3) (P3)--(P4) (P4)--(P5) (P5)--(H)
    (H)--(C1) (H)--(C2) (H)--(C3) (C1)--(C2) (C2)--(C3) (C3)--(C1);%
}
\resizebox{\linewidth}{!}{%
\begin{tikzpicture}[
  nd/.style={circle, draw=black!45, line width=0.7pt, minimum size=8.6mm, inner sep=0pt},
  ttl/.style={font=\sffamily\Large\bfseries},
  lbl/.style={font=\normalsize\bfseries}]
  \begin{scope}[shift={(0,0)}]
    \panelcoords \paneledges
    \foreach \p in {P1,P2,P3,P4,P5}{\node[nd,fill=magLo,text=white,lbl] at (\p) {2};}
    \foreach \p in {C1,C2,C3}{\node[nd,fill=magMid,text=white,lbl] at (\p) {3};}
    \node[nd,fill=magHi,text=black,lbl] at (H) {5};
    \node[ttl] at (-0.1,2.4) {(a)~Graph structure};
  \end{scope}
  \begin{scope}[shift={(7.6,0)}]
    \panelcoords \paneledges
    \foreach \p in {P1,P2,P3,P4,P5,C1,C2,C3,H}{\node[nd,fill=magFlat] at (\p) {};}
    \node[ttl] at (-0.1,2.4) {(b)~After LayerNorm};
  \end{scope}
  \begin{scope}[shift={(15.2,0)}]
    \panelcoords \paneledges
    \foreach \p in {P1,P2,P3,P4,P5}{\node[nd,fill=magHi] at (\p) {};}
    \foreach \p in {C1,C2,C3}{\node[nd,fill=magMid] at (\p) {};}
    \node[nd,fill=magLo] at (H) {};
    \node[ttl] at (-0.1,2.4) {(c)~After PostDeg};
  \end{scope}
\end{tikzpicture}%
}
\vspace{3mm}
\caption{\textbf{Why placement matters (schematic).} The same graph is shown three times, with node positions fixed and colour encoding, schematically, the representation magnitude that would reach the score head; the colours illustrate the placement argument and do not show measured norms. \textbf{(a)}~magnitude tracks degree, and the numbers on the nodes give those degrees. \textbf{(b)}~LayerNorm normalizes each node to a single magnitude, so the degree contrast the scorer needs is gone. \textbf{(c)}~multiplying by $s_i=(\widehat c_i+\varepsilon)^{-1/2}$ restores a magnitude that grows as $1/\sqrt{d_i}$, so the low-degree nodes again have the largest magnitudes.}
\label{fig:placement-summary}
\end{figure}

\paragraph{Contributions.}
\begin{itemize}[leftmargin=*,topsep=2pt,itemsep=1pt]
  \item We identify the placement rule for a positive per-node scalar in a LayerNorm block. Applied \emph{before} LayerNorm, it is divided out in the forward pass ($\mathrm{LN}(a_i z_i)\approx\mathrm{LN}(z_i)$, exact as $\varepsilon_{\rm LN}\to 0$, Proposition~\ref{prop:absorption-exact}), so a degree signal must enter \emph{after} LayerNorm to reach the score head (Section~\ref{sec:placement}).
  \item We confirm the rule with a same-position control. The same parameter-free inverse-degree operator is inert before LayerNorm and improves over the LN backbone after it, on every node-selection task and evaluation size, and a constant post-LayerNorm scale shows the effect also needs the degree contrast (Section~\ref{sec:results-main}).
  \item We introduce PostDeg, the parameter-free post-LayerNorm instance, report its per-task margins and where they are size-dependent, and show the exact form of the scale matters little (Sections~\ref{sec:results-main},~\ref{sec:results-ablation}).
\end{itemize}

\section{Related Work}
\label{sec:related}

BatchNorm, LayerNorm, InstanceNorm, GraphNorm, PairNorm, and NodeNorm answer a feature-statistic question about which activations should be centered, scaled, or constrained \citep{ioffe2015batchnorm,ba2016layernorm,ulyanov2016instance,cai2021graphnorm,zhao2020pairnorm,zhou2021nodenorm}. GraphNorm is our closest reference point in GNN normalization. It normalizes graphwise feature statistics for optimization, whereas PostDeg acts on magnitude after LayerNorm. The controlled main grid uses BatchNorm1d over node features within each processed graph, Extra LayerNorm in the post-block slot, InstanceNorm, scalar-shift GraphNorm, PairNorm, GraphScalar, PostDeg, PostDeg-L-FG, PostDeg-L-Adaptive, and the LN backbone (Table~\ref{tab:main_results}). NodeNorm appears only as a supplemental diagnostic because its comparison pipeline uses a different baseline configuration (Appendix~\ref{app:cross-backbone}). DiffGroupNorm \citep{zhou2020dgn} alters the per-block normalization slot with learned group-wise feature statistics, and we use it only for positioning in Table~\ref{tab:normalizer_design}. Closest to PostDeg in \emph{motivation} is ResNorm \citep{liang2023resnorm}, which reshapes the node-wise standard-deviation distribution during normalization specifically to help low-degree tail nodes. The distinction is positional, since ResNorm rescales \emph{within} the normalization step, whereas PostDeg multiplies the already-normalized output by a closed-form degree function in the post-LayerNorm slot. Degree-specific architectures such as DEMO-Net \citep{wu2019demonet} likewise treat degree as a modeling axis inside the layer, whereas PostDeg acts on post-normalization magnitude. Closest to the ``which normalization, and where'' framing is the learned combination of node-, adjacency-, graph-, and batch-level normalizations of~\citep{chen2020learngraphnorm}. PostDeg instead fixes one closed-form degree scale and studies only its placement relative to LayerNorm. Post-2021 graph-adaptive normalizers such as GRANOLA \citep{eliasof2024} learn a graph-conditioned normalization, which again changes what the normalizer computes. The closest work outside GNNs is \citep{wang2026scalevec}, on scale-vector placement around linear maps in language models, and PostDeg is the graph analogue with a closed-form degree scale in the post-LayerNorm slot.

Degree-aware aggregators add degree inside message passing. The graph convolutional network (GCN) uses the symmetric normalization $D^{-1/2}AD^{-1/2}$ of the adjacency matrix $A$ by the degree matrix $D$ \citep{kipf2017gcn}. GraphSAGE (SAGE) averages over neighborhoods in its mean aggregator \citep{hamilton2017graphsage}, and PNA scales each aggregator by $\delta(d_i)=\log(d_i+1)/\bar\delta$ inside the message-passing step \citep{corso2020pna}. PNA's degree scaler changes messages before they are combined and normalized, so the LayerNorm in our backbone sees a representation with degree already blended into the features. PostDeg leaves aggregation untouched and acts on per-node magnitudes after LayerNorm. The PNA test confirms this, since PostDeg + PNA is paired-equivalent to PNA without PostDeg on InfluMax and Dismantle ($|\Delta\%|<1\%$, Section~\ref{sec:results-pna} and Appendix~\ref{app:cross-backbone}, Table~\ref{tab:pna-redundancy}).

Structural encodings add degree, centrality, shortest-path, random-feature, subgraph, or transformer positional information as feature content \citep{dehmamy2019understanding,sato2020random,bouritsas2022gsn,ying2021graphormer,rampasek2022gps}. Content features such as $\log d_i$ enter before LayerNorm and are re-centered and re-scaled with the rest of the representation. PostDeg instead changes the magnitude after LayerNorm, on an already-encoded representation. Appendix~\ref{app:content-vs-magnitude} formalizes the content-versus-magnitude distinction and reports the $\log d_i$ baseline.

We ask where a positive per-node scalar survives LayerNorm. The algebraic identity in Eq.~\eqref{eq:intro-ln-absorption} is standard in the literature on LayerNorm, RMSNorm (root-mean-square normalization), and weight tying \citep{ba2016layernorm,zhang2019rmsnorm,xiong2020transformerln}. We use it as a placement rule for GNN policies trained on combinatorial optimization tasks \citep{khalil2017learning,kool2019attention,karalias2020erdos,joshi2022learning,cappart2023combinatorial}. Task-specific learned methods exist for these problems, such as ToupleGDD \citep{chen2022touplegdd} and DeepIM \citep{ling2023deepim} for influence maximization. They are complementary to our work, since we study a normalization slot inside a fixed backbone. Independent of GNNs, degree-based influence-maximization heuristics show advantages that rise with degree heterogeneity \citep{xie2023degree}, consistent with our within-task finding (\S\ref{sec:mechanism}). All main experiments use a fixed GAT backbone \citep{velickovic2018gat,brody2022gatv2} so the normalization slot is the variable under test.

\section{Method}
\label{sec:method}

The method has two parts. First, we identify the normalization slot where a positive per-node scalar survives LayerNorm. Second, we place a minimal degree scale in that slot and test whether additional capacity is necessary. Figure~\ref{fig:method-overview} shows the controlled setup. We compute graph-level statistics once, use the same post-LayerNorm slot in each GAT layer, and change only that slot while the task heads, budgets, data generator, and seeds stay fixed.

\begin{figure}[t]
\centering
\includegraphics[width=\linewidth]{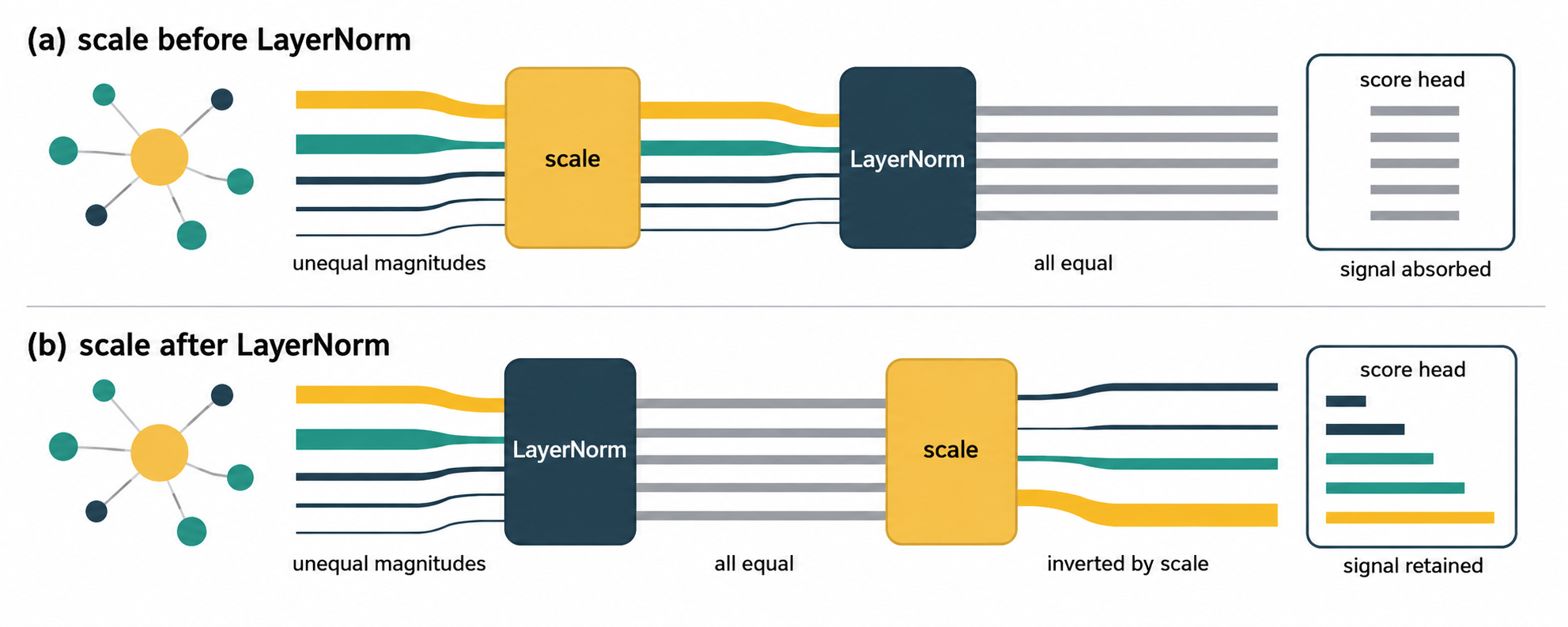}
\vspace{1.5mm}
\caption{\textbf{Method overview.} All runs share one backbone (node state $h_i$, residual GAT block, LayerNorm, score head), and the only thing that changes is where a positive per-node scalar multiplies the representation. \textbf{(a)}~Placed before LayerNorm, the scalar $a_i$ is divided out by the normalizer in the forward pass, so $\mathrm{LN}(a_i z_i)\approx\mathrm{LN}(z_i)$ (Eq.~\ref{eq:intro-ln-absorption}, exact as $\varepsilon_{\rm LN}\to 0$), and it does not reach the score head as magnitude, so the per-node magnitudes arriving at the scorer stay flat. \textbf{(b)}~Placed after LayerNorm, the scalar $s_i$ multiplies the normalized representation and reaches the head as representation magnitude, so a positive per-node $s_i$ becomes a usable degree-conditioned signal. The backbone, budgets, data generator, and seeds stay fixed across the main grid, and only this post-LayerNorm slot varies. The variants tested there (LN backbone, GraphScalar, Extra LayerNorm, PostDeg, PostDeg-L-FG, PostDeg-L-Adaptive) are defined in Section~\ref{sec:variants}.}
\label{fig:method-overview}
\end{figure}

\subsection{Why placement is post-LayerNorm}
\label{sec:placement}

With LayerNorm stabilizer $\varepsilon_{\rm LN}$, per-coordinate gain/bias $(g,b)$, elementwise product $\odot$, and $\mu(z_i),\sigma^2(z_i)$ the feature-wise mean and variance of a node representation $z_i$, the LayerNorm of a per-node-multiplied input $a_i z_i$ is
\begin{equation}
  \LN(a_i z_i) = g\odot \frac{z_i - \mu(z_i)}{\sqrt{\sigma^2(z_i) + \varepsilon_{\rm LN}/a_i^2}} + b,
  \label{eq:intro-ln-absorption}
\end{equation}
so multiplying a node representation by $a_i>0$ before LayerNorm changes only the stabilizer term, and $\LN(a_i z_i)=\LN(z_i)$ exactly in the zero-stabilizer limit \citep{ba2016layernorm}. With $\varepsilon_{\rm LN}>0$, the relative deviation between the two is at most $\varepsilon_{\rm LN}\lvert 1-a_i^{-2}\rvert/(2\sigma^2(z_i))$, negligible for any activation variance bounded away from zero (Proposition~\ref{prop:absorption-exact}, Appendix~\ref{app:variance}). Within the LayerNorm block in Eq.~\eqref{eq:intro-ln-absorption}, a positive degree multiplier survives on the post-LN side. Interestingly, this settles the placement question with arithmetic already standard for LayerNorm. The remaining question is what operator should go in that slot.

\paragraph{Placement rule.}
The absorption identity applies to positive per-node scalar multipliers such as degree, PageRank, $k$-core, clustering, or learned positive attention scores. Content features behave differently, since a concatenated $\log d_i$ feature is re-centered and re-scaled with the rest of the representation, so we treat content baselines separately in Appendix~\ref{app:content-vs-magnitude}. The rule is therefore positional, and a positive degree scalar belongs after LayerNorm if the score head should read it, whatever form the scalar takes.

\subsection{The operators we compare}
\label{sec:variants}
\paragraph{Notation.} We write $d_i$ for the degree of node $i$ and $\widehat c_i=\max(d_i,1)/\max_j\max(d_j,1)\in(0,1]$ for its clipped normalized degree. The learned variants write $c_i=d_i/\max_j d_j$ clipped to $[10^{-6},1]$, which matches $\widehat c_i$ except at isolated nodes. $\LN$ denotes the residual-block LayerNorm (LN), and $s_i>0$ the per-node scale the slot applies to the LayerNorm output. $\lambda_2$ and $\lambda_{\max}$ are the second-smallest and largest eigenvalues of the normalized Laplacian, and $\varepsilon=10^{-8}$, $\varepsilon'=10^{-4}$ are small numerical stabilizers for the degree and spectral denominators. All methods below share one architecture and one training setup and differ only in this slot.

\paragraph{PostDeg.} The simplest choice, and the one we adopt as the default, is a fixed inverse-degree scale with exponent $1/2$, the symmetric graph-convolution value, and it needs no tuning. Given the residual output of layer $\ell-1$,
\begin{equation}
  H^{(\ell)} = (\widehat c+\varepsilon)^{-1/2}\odot \LN\!\left(H^{(\ell-1)}+\mathrm{GAT}_\ell(H^{(\ell-1)},A)\right),\qquad \varepsilon=10^{-8},
  \label{eq:postdeg-formula}
\end{equation}
where $\odot$ broadcasts $s_i=(\widehat c_i+\varepsilon)^{-1/2}$ across the feature dimension and isolated nodes take degree $1$. PostDeg has no learned parameters and adds $\mathcal{O}(|V|)$ work per layer, within a $30$--$35$ min walltime envelope against the LN backbone's $30$ min (Appendix~\ref{app:reproducibility}). We kept PostDeg deliberately minimal, so that any measured effect can be credited to the slot itself.

\paragraph{Variants and controls (all in the same slot).}
\begin{itemize}[leftmargin=*,topsep=2pt,itemsep=1pt]
  \item \textbf{LN backbone:} identity in the slot (the residual block still contains its own LayerNorm).
  \item \textbf{PostDeg-L-FG / PostDeg-L-Adaptive:} learned scales $s_i=\alpha(c_i+\varepsilon)^{-(\beta+\delta\lambda_2)}(\lambda_{\max}+\varepsilon')^{-1/2}$ with learned coefficients $\alpha,\beta,\delta$, blended with the identity through a gate $m_i\in(0,1)$, so that $\widetilde h_i=\bigl(m_i\,s_i+1-m_i\bigr)\,\bar h_i$, where $\bar h_i=\LN(h_i)$ is the LayerNorm output. FG uses one scalar gate. Adaptive uses a per-node gate $m_i=\sigma(\mathrm{MLP}(c_i))$, with $\sigma$ the logistic sigmoid and $\mathrm{MLP}$ a small multi-layer perceptron. Parameter domains, clamps, and initialization are in Appendix~\ref{app:theory}, Eq.~\eqref{eq:dsn-formula}.
  \item \textbf{GraphScalar:} a single per-graph multiplier $s=(\lambda_{\max}+\varepsilon')^{-1/2}$, with no per-node dependence.
  \item \textbf{Extra LayerNorm:} a second LayerNorm in the slot (added normalization, no degree signal).
  \item \textbf{Feature-statistic baselines:} GraphNorm, PairNorm, InstanceNorm, and BatchNorm placed in the slot.
\end{itemize}
Two clarifications matter for the equivalence claim in Section~\ref{sec:results-ablation}. First, PostDeg is the \emph{limiting form} of the learned family ($\alpha{=}1$, $\beta{=}1/2$, $\delta{=}0$, gate$\to 1$), but the two are implemented as separate operators (PostDeg an ungated scale, the learned variants a gated blend), and the learned variants train freely from their own initialization. Second, the per-graph factor $(\lambda_{\max}+\varepsilon')^{-1/2}$ is applied \emph{after} LayerNorm, where for a fixed graph it rescales each node equally and cannot change node ranking, but because it varies across graphs it stays separate from the shared LayerNorm gain. The implementation clamps and the inactive $\varepsilon$ floor are in Appendix~\ref{app:reproducibility}.

\subsection{What the placement rule implies}
\label{sec:method-predictions}

Formal statements and proofs are in Appendix~\ref{app:theory}, and the experiments need only the degree-separation ratio. For two nodes with $\widehat c_i>\widehat c_j$ and exponent $\gamma>0$,
\begin{equation}
  R_{ij}:=\frac{s_j}{s_i}=\left(\frac{\widehat c_i+\varepsilon}{\widehat c_j+\varepsilon}\right)^{\!\gamma}>1,
  \label{eq:rij}
\end{equation}
so low-degree nodes get larger post-LayerNorm magnitude, the slope of $\log s_i$ against $\log d_i$ is negative, and the effect weakens as degree heterogeneity vanishes. These relations imply concrete expectations that we test. Pre-LayerNorm scalars are absorbed, and PostDeg helps only when degree-conditioned magnitude aligns with the reward and degrees are heterogeneous enough to vary. Per-graph scalars cannot recover node ranking, and a backbone that already carries degree (PNA) leaves little for PostDeg to add.

\section{Experiments}
\label{sec:experiments}

\subsection{Experimental Setup}
\label{sec:setup}
\textbf{Tasks.} We evaluate four node-selection tasks and one graph-classification null task. InfluMax runs influence maximization \citep{kempe2003maximizing} on Barab\'asi--Albert (BA) graphs \citep{barabasi1999}, Dismantle runs network dismantling \citep{braunstein2016network} on stochastic block model (SBM) graphs \citep{holland1983}, Epidemic on SBM contact graphs, MIS on SBM graphs, and DD graph classification on the TU protein benchmark \citep{morris2020tudataset}. Per-task degree coefficient of variation (CV) ranges from $0.33$ to $0.81$ (Table~\ref{tab:task-graph-stats}).

\begin{table}[h]
\centering
\caption{Per-task degree distribution summary statistics on the training-graph distribution (200 graphs per task), with the PostDeg gain over the LN backbone paired by 10 seeds at the largest evaluation size. The gain does \emph{not} track degree dispersion across tasks. MIS and DD have almost the same coefficient of variation ($0.33$) and skewness ($0.35$ versus $0.27$), yet MIS carries the largest gain and DD the only negative one, and across the three positive tasks the rank correlation between dispersion and gain is negative (Spearman $-0.5$, \S\ref{sec:mechanism}). A heavier-tail trend appears only \emph{within} MIS, by graph family. Epidemic's value is at the largest size. It is $+3.7\%$ at the smallest evaluation size and fades with size (Table~\ref{tab:size-profile}).}
\label{tab:task-graph-stats}
\footnotesize
\setlength{\tabcolsep}{6pt}
\renewcommand{\arraystretch}{1.10}
\begin{tabular}{lccccc}
\toprule
Task (graph family) & mean $d$ & CV & max $d$ & skewness & PostDeg $\Delta\%$ \\
\midrule
\rowcolor{tabrowblueStrong}
\textbf{InfluMax (BA)}   & 5.82  & 0.81 & 45  & \textbf{2.90} & $+3.5\%$ \\
\textbf{Dismantle (SBM)} & 6.32  & 0.42 & 20  & 0.46          & $+2.5\%$ \\
\textbf{MIS (SBM)}       & 9.95  & 0.33 & 27  & 0.35          & $+5.6\%$ \\
\midrule
\rowcolor{tabrowgray}
Epidemic (SBM contact)   & 6.29  & 0.41 & 20  & 0.45          & $+0.1\%$ \\
\rowcolor{tabrowgray}
DD (TU proteins)         & 4.99  & 0.33 & 15  & \tabmuted{0.27} & $-0.7\%$ \\
\bottomrule
\end{tabular}
\end{table}

\textbf{Same-slot comparison.} Each method uses the same slot $\mathrm{Norm}(\cdot):\R^{n\times d}\!\to\!\R^{n\times d}$ applied to the residual GAT output $Z^{(\ell)}=H^{(\ell-1)}+\mathrm{GAT}_{\ell}(H^{(\ell-1)},A)$. The LN backbone is the identity in this post-block slot, and the residual block still contains LayerNorm before the slot under study. GraphScalar is a graphwise multiplier $1/\sqrt{\lambda_{\max}+\varepsilon'}$ derived from the normalized Laplacian (no per-node degree dependence, so it tests whether a graph-level spectral term explains the gain). PNA and NodeNorm relax this slot and are reported separately in Appendix~\ref{app:cross-backbone}.

\textbf{Architecture.} All methods use a 3-layer GAT with 4 attention heads, hidden dimension 128, dropout 0.1, and a task-specific MLP head.

\textbf{Training.} 200 epochs with online graph generation using each task reward as the training signal (spread for InfluMax, fragmentation for Dismantle, negative infected fraction for Epidemic, and valid selected-set size for MIS). Training sizes are $n=100$ for InfluMax and $n=150$ for the others. 10 independent seeds per method and task. Selection budgets are $K=\lfloor 0.1n\rfloor$ for InfluMax, $\lfloor 0.2n\rfloor$ for Dismantle and Epidemic, and $\lfloor 0.3n\rfloor$ for MIS.

\textbf{Evaluation.} The learned score head is used greedily (Algorithm~\ref{alg:greedy}, Appendix~\ref{app:reproducibility}). Evaluation sizes extend to $n=300$ for the transfer check. We report mean $\pm$ seed standard deviation, paired Wilcoxon tests, and two one-sided tests (TOST) for equivalence in Appendix~\ref{app:experiments}. At 10 seeds, the paired Wilcoxon test detects effects of size $\sim 0.7\sigma$ at $\alpha=0.05$ (one-sided), and TOST detects equivalence within $\pm 1\%$ relative. Cross-backbone numbers on GCN, the graph isomorphism network (GIN), GraphSAGE, and PNA at training-graph size are in Appendix~\ref{app:cross-backbone}.

\subsection{The effect and its controls}
\label{sec:results-main}

PostDeg improves over the LN backbone on every node-selection task and at every evaluation size. The improvement is unanimous across the 10 seeds, except at the two middle MIS-SBM sizes, and two controls in the same position show where it comes from (Table~\ref{tab:placement-main}). The same operator applied before LayerNorm stays at the backbone with no cell significant, which confirms the forward-pass absorption of Proposition~\ref{prop:absorption-exact}. The informative control is a constant scale after LayerNorm, which keeps the slot but drops the degree contrast. It stays at the backbone on InfluMax, Dismantle, and MIS-SBM, so on those families the effect is the degree contrast, and it recovers a third to a half of the advantage on the heavy-tailed MIS-BA family, so part of that effect is uniform amplification (Appendix~\ref{app:review-controls}). The before-LayerNorm result is a forward-pass statement, and a pre-LayerNorm scalar still rescales gradients and the effective learning rate during training \citep{vanlaarhoven2017,arora2019,lyle2024}. In our view this is the central result of the paper, since it holds on every node-selection task and size and survives the same-slot controls.

\begin{table}[t]
\centering
\caption{Placement decides the effect. Relative change over the LN backbone (\%, range across evaluation sizes, 10 seeds, paired) for one parameter-free inverse-degree operator applied before LayerNorm, a constant scale applied after LayerNorm, and PostDeg (the same operator after LayerNorm). The before-LayerNorm arm stays at the backbone on every task and size, with no cell significant. PostDeg improves on every task and size, unanimous across seeds except at the two middle MIS-SBM sizes. The constant scale stays at the backbone except on the heavy-tailed MIS-BA family, so the effect needs both the post-LayerNorm slot and the degree contrast. Full per-size numbers with wins are in Appendix Table~\ref{tab:review-placement}.}
\label{tab:placement-main}
\footnotesize
\setlength{\tabcolsep}{8pt}
\renewcommand{\arraystretch}{1.15}
\begin{tabular}{lccc}
\toprule
Task (family) & before LN & constant, after LN & PostDeg (after LN) \\
\midrule
InfluMax (BA)   & $+0.1$ to $+0.2$ & $-0.2$ to $+0.7$ & $+2.2$ to $+3.3$ \\
Dismantle (SBM) & $-0.1$ to $+0.6$ & $-0.1$ to $+0.3$ & $+2.4$ to $+7.7$ \\
MIS (SBM)       & $-0.4$ to $-0.1$ & $-0.1$ to $+0.3$ & $+0.5$ to $+5.8$ \\
MIS (BA)        & $-0.2$ to $-0.1$ & $+2.7$ to $+3.6$ & $+5.6$ to $+9.7$ \\
\bottomrule
\end{tabular}
\end{table}

\begin{table}[h]
\centering
\caption{Relative improvement over the LN backbone on node-selection tasks (\%, mean\,$\pm$\,seed std, 10 seeds). The parameter-free \textbf{PostDeg} row (top) is the recommended default, and the TOST column marks methods statistically equivalent to it on these pooled evaluations. Epidemic is shown in gray. These are largest-size values, where the effect is near zero, and it reaches $+3.7\%$ at the smallest evaluation size (Table~\ref{tab:size-profile}). $W_{10}$ is paired-seed wins out of 10 (mean across the three positive tasks).}
\label{tab:ablation_improvement}
\footnotesize
\setlength{\tabcolsep}{3pt}
\renewcommand{\arraystretch}{1.15}
\begin{tabular}{lcccccc}
\toprule
Method & InfluMax & Dismantle & \tabmuted{Epidemic} & MIS & $W_{10}$ & TOST \\
\midrule
\rowcolor{tabrowblueStrong}
\textbf{PostDeg (ours)} & +3.50\,$\pm$\,0.55 & +2.53\,$\pm$\,0.61 & \tabmuted{+0.08\,$\pm$\,0.08} & +5.58\,$\pm$\,0.85 & 10/10 & \ding{52}~(ref.) \\
PostDeg-L-FG & +3.50\,$\pm$\,0.60 & +2.65\,$\pm$\,0.35 & \tabmuted{+0.17\,$\pm$\,0.15} & +6.09\,$\pm$\,1.09 & 10/10 & \ding{52} \\
PostDeg-L-Adaptive & +3.61\,$\pm$\,0.90 & +2.61\,$\pm$\,0.30 & \tabmuted{+0.15\,$\pm$\,0.13} & +5.85\,$\pm$\,1.12 & 10/10 & \ding{52} \\
\midrule
\multicolumn{7}{l}{\emph{Graphwise controls}} \\
GraphScalar & +0.31\,$\pm$\,0.74 & $-$0.90\,$\pm$\,0.34 & \tabmuted{+0.01\,$\pm$\,0.11} & $-$0.57\,$\pm$\,1.15 & 6/10 & \ding{56} \\
GraphNorm & +2.58\,$\pm$\,0.60 & $-$1.34\,$\pm$\,1.19 & \tabmuted{$-$0.36\,$\pm$\,0.31} & +4.27\,$\pm$\,1.32 & 7/10 & \ding{56} \\
\addlinespace[3pt]
\multicolumn{7}{l}{\emph{Capacity-only control}} \\
Extra LayerNorm & +0.06\,$\pm$\,0.58 & $-$0.07\,$\pm$\,0.28 & \tabmuted{+0.02\,$\pm$\,0.11} & $-$0.07\,$\pm$\,1.07 & 4/10 & \ding{56} \\
\addlinespace[3pt]
\multicolumn{7}{l}{\emph{Feature-statistic baselines}} \\
PairNorm & +0.80\,$\pm$\,0.83 & $-$1.86\,$\pm$\,1.22 & \tabmuted{$-$0.58\,$\pm$\,0.27} & +2.02\,$\pm$\,1.08 & 6/10 & \ding{56} \\
InstanceNorm & +2.58\,$\pm$\,0.91 & $-$3.57\,$\pm$\,1.66 & \tabmuted{$-$0.37\,$\pm$\,0.19} & +4.02\,$\pm$\,1.52 & 6/10 & \ding{56} \\
BatchNorm & +1.02\,$\pm$\,0.73 & $-$1.41\,$\pm$\,0.70 & \tabmuted{$-$0.04\,$\pm$\,0.14} & +2.89\,$\pm$\,1.29 & 6/10 & \ding{56} \\
\bottomrule
\end{tabular}
\end{table}

Table~\ref{tab:ablation_improvement} reports the size of the post-LayerNorm effect against the same-slot baselines, each measured against the LN backbone and paired by seed. PostDeg gains $+3.5\%$ on InfluMax, $+2.5\%$ on Dismantle, and $+5.6\%$ on MIS at the largest evaluation size, on 10 of 10 seeds. The gain is size-dependent on two of the tasks. On epidemic containment, where the reward is the negative infected fraction under susceptible-infected-recovered (SIR) dynamics, it reaches $+3.7\%$ at the smallest evaluation size and falls to $+0.08\%$ at the largest, and Dismantle follows the same decreasing profile while staying clearly positive (Table~\ref{tab:size-profile}).

Two facts frame these numbers. Degree is excluded from the inputs, so PostDeg reinjects it as post-LayerNorm magnitude, and a log-degree input feature recovers only part of the lift in the same pipeline ($+1.6\%$ against PostDeg's $+2.1\%$ on InfluMax at 5 seeds, Appendix~\ref{app:content-vs-magnitude}). On MIS the trained policy stays $11$ to $23\%$ below a greedy construction \citep{schuetz2022,angelini2023}, so we read MIS as a comparison of normalization slots. Per-seed reward differences are in Table~\ref{tab:per-seed-reward}.

The same table separates the effect from other same-slot changes. GraphNorm is negative on Dismantle ($-1.34\%$ versus $+2.53\%$ for PostDeg), and GraphScalar and Extra LayerNorm stay at or below the backbone, so a graphwise spectral term and extra normalization do not explain the improvement. PairNorm, InstanceNorm, and BatchNorm match PostDeg on some tasks but collapse to the majority-class predictor on DD (Figure~\ref{fig:appendix-transfer-heatmap}). On MIS the margin over the best competing normalizer is size-dependent, and PostDeg is competitive with these tuned normalizers at zero added parameters (Appendix Table~\ref{tab:mis-by-size}). Absolute metrics and per-baseline tests are in Appendix Table~\ref{tab:main_results}.

\subsection{The exact form of the scale matters less}
\label{sec:results-ablation}

Parameter-free PostDeg is the method we recommend, and, to our mild surprise, the exact form of the scale matters little. Two learned same-slot variants (PostDeg-L-FG, PostDeg-L-Adaptive) test whether added capacity in the slot helps, and they are statistically equivalent to PostDeg on the pooled and SBM evaluations (TOST rejects a difference beyond $\pm1\%$ on all tasks, Table~\ref{tab:tost}, and a paired Wilcoxon does not reject equality, Table~\ref{tab:significance}). These variants also carry weight spectral normalization that PostDeg does not, so this is not a clean capacity ablation (Appendix~\ref{app:reproducibility}). The equivalence breaks only when MIS is split by graph family, where the learned exponent beats the fixed $1/2$ on the heavy-tailed BA family, a split that also mixes distribution shift and the added weight normalization (Appendix~\ref{app:stats}, Table~\ref{tab:tost-family}). Our direct $\gamma$-sweep on the parameter-free operator (Appendix~\ref{app:gamma-sweep}) varies the exponent alone. It helps on its own, clearest on MIS, but its BA gain ($+0.7$ to $+1.0\%$, not significant at $10$ seeds) is small, and the learned exponent barely moves from its initialization (Appendix~\ref{app:why-beta}). We keep $\gamma=1/2$ as the no-tuning default, since it matches the GCN-symmetric exponent in $D^{-1/2}AD^{-1/2}$ and needs no search.

\subsection{Does the post-LayerNorm scale reach the score head?}
\label{sec:mechanism}

By construction, the operator's own output is a function of node degree. With $s_i=(\widehat c_i+\varepsilon)^{-\gamma}$, log-log fits of $\ln s_i$ against $\ln d_i$ are uniformly negative with strong correlations, and a linear fit returns a task-dependent correlation between $-0.38$ and $-0.76$ (Appendix Tables~\ref{tab:pearson-all} and~\ref{tab:scale-extreme}). Both are computed on the operator's own scalar factors, upstream of the representation that reaches the score head, so the informative quantity is whether degree still tracks the \emph{representation magnitude} there, after the residual gate, clamp, and learned message passing. We measure it with forward hooks on the trained models. At the score head input, the per-node representation norm shows a clear inverse-degree slope for PostDeg, near $-0.49$ with correlation near $-0.98$ on all three tasks. It stays flat for the LN backbone, a constant post-LayerNorm scale, and a pre-LayerNorm scale (Appendix~\ref{app:review-hooks}). The pre-LayerNorm scale is the direct placement check, since its LayerNorm-input variance is degree-dependent while its score-head norm stays flat. Because the score-head magnitude is nearly a function of degree, PostDeg pushes the policy toward degree-weighted selection, and on Dismantle and Epidemic the trained policy selects much like the degree heuristic and outperforms it (Appendix Table~\ref{tab:policy-vs-heuristic}).

\paragraph{Where the heterogeneity effect appears.} We split MIS by graph family. MIS trains on SBM, so the BA family is out-of-distribution here, and this contrast mixes heterogeneity with distribution shift. Training MIS on BA removes the shift, and PostDeg still gains ($+1.7$ to $+3.9\%$, 10 of 10 seeds, Appendix~\ref{app:review-controls}), so shift does not fully explain it. Across the three positive tasks, though, degree dispersion does not order the improvement, since MIS has the lowest CV ($0.33$) and skewness ($0.35$) yet the largest advantage, and the cross-task rank correlation is negative (Spearman $-0.5$, Table~\ref{tab:task-graph-stats}). We therefore treat the heterogeneity account as a hypothesis, supported only by the within-task contrast, with the MIS cross-task gain left unexplained.

\subsection{The effect across evaluation sizes}
\label{sec:results-transfer}
PostDeg stays above the LN backbone across the tested $2$--$3\times$ train-to-test range on all four tasks (Figure~\ref{fig:transfer}). It is roughly flat with size on InfluMax, decays with size on Dismantle and Epidemic, and is non-monotone on MIS, with a dip to near zero at $n{=}200$ that we cannot account for. Across each (task, size) pair, paired Wilcoxon tests reject equality between the PostDeg policy and the canonical classical heuristic at $\alpha=0.05$, and PostDeg has the smallest train-evaluation gap of the methods tested (Appendix Table~\ref{tab:full-transfer}, Table~\ref{tab:policy-vs-heuristic}, Figures~\ref{fig:appendix-rel-vs-size} and~\ref{fig:appendix-gengap}).

\begin{figure}[t]
  \centering
  \includegraphics[width=\linewidth]{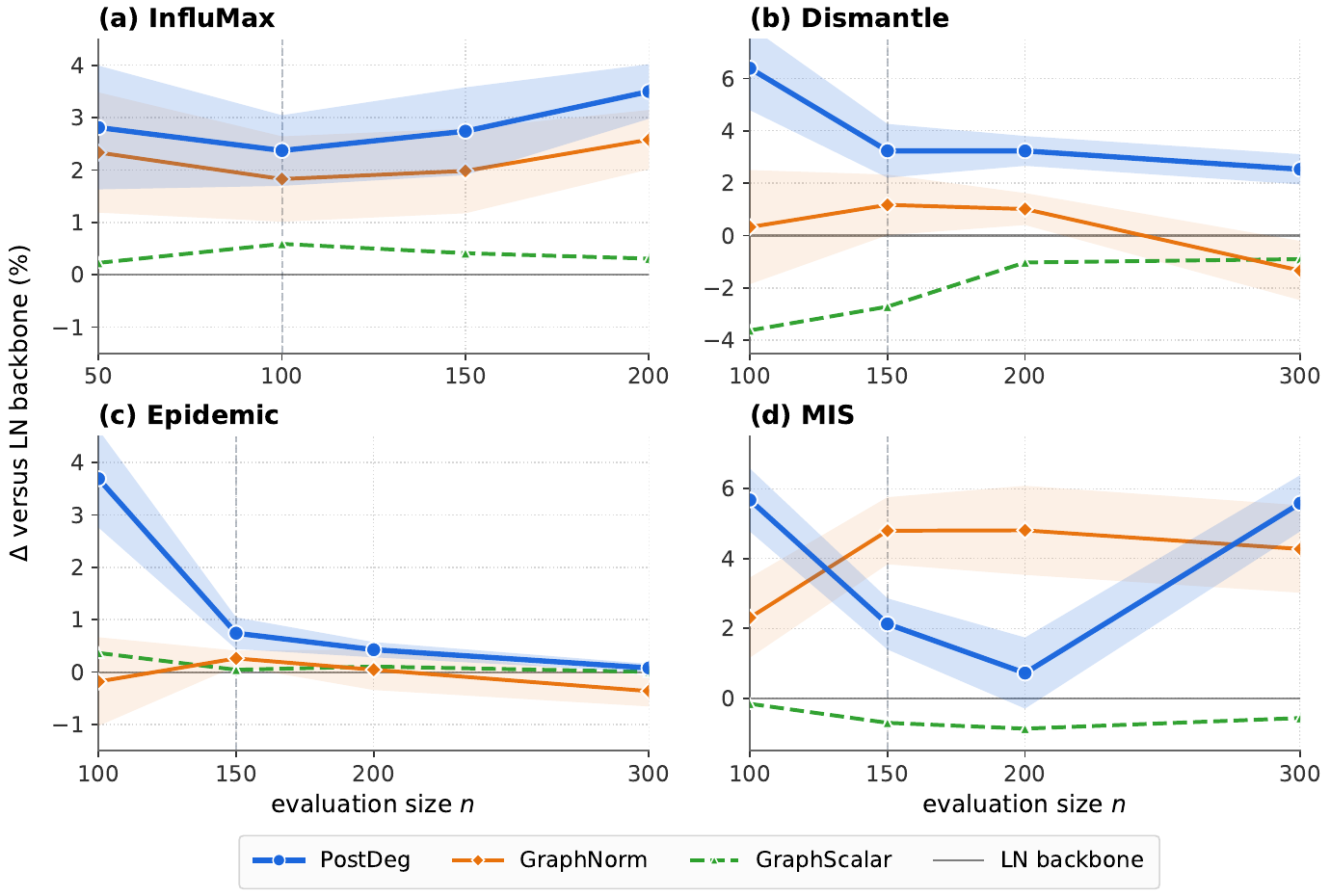}
  \caption{PostDeg gain over the backbone across evaluation sizes. Improvement over the LN backbone at each evaluation size (mean $\pm$ seed std over 10 seeds, bands at $\pm 1$ seed std, dashed vertical line at the training size $n_{\rm train}$). PostDeg (blue) stays ahead on InfluMax, Dismantle, and MIS, while GraphNorm (orange) and GraphScalar (green) do not. On Epidemic, PostDeg's advantage is largest at the smallest evaluation size and shrinks toward zero as graphs grow, while the other methods stay near zero. Learned same-slot variants are in Appendix Figure~\ref{fig:appendix-rel-vs-size}, and the within-MIS trend across graph families is in Appendix Figure~\ref{fig:appendix-mis}.}
  \label{fig:transfer}
\end{figure}

\subsection{No gain on DD, and a pooling-head artifact}
\label{sec:results-boundary}

As expected, on DD PostDeg does not improve over the backbone, and both remain near $80\%$. DD is a graph-classification benchmark, so the placement rule, which targets degree-sensitive node selection, does not apply to it. Low degree heterogeneity does not explain this null, because MIS has almost the same coefficient of variation and skewness (Table~\ref{tab:task-graph-stats}) and yet shows the largest improvement in the paper. PairNorm ($58.7\%$), InstanceNorm ($58.7\%$), and BatchNorm ($58.8\pm0.2\%$) fall to the majority-class rate ($58.65\%$) on every fold, while PostDeg, Extra LayerNorm, GraphScalar, and the backbone stay near $80\%$. This collapse is an artifact of the graph-classification pipeline, and the same normalizers match PostDeg on several node-selection tasks. The normalizer is the last operation before the graph embedding is formed by averaging over nodes, and a normalizer that centers each feature across a graph's nodes makes that average a constant vector, which can only feed the majority-class predictor. It appears at the first epoch on every fold, and Appendix~\ref{app:dd-collapse} gives the per-method arithmetic. GraphNorm has two DD numbers from two protocols ($78.5\%$ main 10-seed, $77.4\%$ supplemental 5-fold) and NodeNorm is at $78.8\%$ in the supplemental run (Appendix Table~\ref{tab:dd-extended}), and we report the two protocols separately.

\subsection{When degree is already in aggregation, PostDeg adds nothing}
\label{sec:results-pna}

PNA \citep{corso2020pna} is the closest published built-in degree normalizer. Its aggregator scales each message by $\delta(d_i)=\log(d_i+1)/\bar\delta$ while the messages are still being aggregated, so the degree term is already embedded in the representation before it reaches LayerNorm. Adding PostDeg on top of a PNA backbone returns paired-equivalent results ($|\Delta\%|<1\%$ on InfluMax and Dismantle, Appendix Table~\ref{tab:pna-redundancy}). These PNA groups run at 5 and 3 seeds versus 10 elsewhere and equivalence is power-sensitive, so we read this as a power-limited indication for our setup.

\section{Discussion and Conclusion}
\label{sec:discussion}

We began from a simple question about where a per-node degree signal can reach the score head inside a LayerNorm residual block. Only the slot after LayerNorm carries a positive per-node scalar to the score head as magnitude, because anything applied earlier is divided out in the forward pass. The pre-LayerNorm arm matches the backbone at evaluation (\S\ref{sec:results-main}), even though a scalar applied before LayerNorm can still shape training through the backward pass. PostDeg puts a parameter-free inverse-degree scale in that slot, and our controlled comparisons separate two effects. Placement is necessary for any improvement, while the exact form of the operator matters much less. Its small extra effect is clearest for the learned variant on heavy-tailed BA graphs (\S\ref{sec:results-ablation}). On InfluMax, Dismantle, and MIS, PostDeg improves over the LayerNorm backbone by $+3.5\%/+2.5\%/+5.6\%$ on all 10 seeds at the largest size. Its epidemic advantage is smaller and fades with size, DD shows no gain, and adding PostDeg on a degree-aware PNA backbone is paired-equivalent at the reduced seed counts we ran there (\S\ref{sec:results-pna}). To let a score head see the degree signal, the scale should go in the post-LayerNorm slot, with $\gamma=1/2$ as a default that needs no tuning. PostDeg gains nothing where degree already enters through aggregation or where heterogeneity is low.

\paragraph{Contextualization and impact.} Prior GNN-normalization work redesigns the normalizer (GraphNorm, PairNorm, NodeNorm) or the aggregator (PNA) to recover a lost degree signal. Our result is complementary and simpler, since the signal is often lost only because the scalar is on the wrong side of the normalizer, and moving it adds no parameters. The absorption identity holds for normalizers that divide each node by its own scale (LayerNorm, RMSNorm). Normalizers that pool across nodes, such as BatchNorm, are outside it, which is also why the node-centering normalizers behave differently in the DD pipeline (\S\ref{sec:results-boundary}). The placement check applies to a LayerNorm or RMSNorm residual block that feeds a positive per-node scalar to a downstream head, shown here only on the synthetic node-selection tasks we test.
\paragraph{Limitations.} We performed no per-method tuning, so a tuned baseline could shrink the margins. The positive results are all synthetic (BA/SBM), and the one real dataset, DD, shows no improvement, so practical relevance stays untested. The current dense-adjacency implementation also bounds the tested sizes to a few hundred nodes. Our baselines stop at GraphNorm/PairNorm and one degree-aware aggregator, without post-2021 normalizers or degree-aware attention. The theory is single-layer and covers positive scalars only (App.~\ref{app:theory-limitations}).

\clearpage
\appendix
\setcounter{table}{0}
\renewcommand{\thetable}{A\arabic{table}}
\setcounter{figure}{0}
\renewcommand{\thefigure}{A\arabic{figure}}
\setcounter{equation}{0}
\renewcommand{\theequation}{A\arabic{equation}}
\setcounter{section}{0}
\renewcommand{\thesection}{A.\arabic{section}}

\section{Notation and statistical conventions}
\label{app:reading-guide}

This appendix collects the proofs, the full statistical apparatus, and the supplementary runs.

\paragraph{Notation.} Table~\ref{tab:glossary} collects the symbols, definitions, and default values used throughout the appendix.
\begin{table}[H]
\centering
\caption{Glossary of symbols, definitions, default values, and the section or equation where each is introduced.}
\label{tab:glossary}
\footnotesize
\setlength{\tabcolsep}{5pt}
\resizebox{\textwidth}{!}{%
\begin{tabular}{llll}
\toprule
Symbol & Meaning & Value & Defined at \\
\midrule
$d_i$ & raw degree of node $i$ & n/a & Section~\ref{sec:variants} \\
$d_{\max}$ & max degree on a graph & n/a & Section~\ref{sec:variants} \\
$\widehat c_i$ & PostDeg normalized degree, $\max(d_i,1)/\max_j\max(d_j,1)$, in $(0,1]$ & n/a & Eq.~\ref{eq:postdeg-formula} \\
$c_i$ & learned-variant normalized degree, $d_i/d_{\max}$ clipped to $[10^{-6},1]$ & n/a & Eq.~\ref{eq:dsn-formula} \\
$\varepsilon$ & PostDeg floor in $(c_i+\varepsilon)$ & $10^{-8}$ & Eq.~\ref{eq:postdeg-formula} \\
$\varepsilon'$ & spectral floor in $\sqrt{\lambda_{\max}+\varepsilon'}$ & $10^{-4}$ & Eq.~\ref{eq:dsn-formula} \\
$\varepsilon_{\rm LN}$ & LayerNorm stabilizer & $10^{-5}$ & Eq.~\ref{eq:intro-ln-absorption} \\
$\bar h_i$ & $\LN(h_i)$, the LayerNorm output & n/a & Section~\ref{sec:variants} \\
$s_i$ & post-LayerNorm scale factor & $(\widehat c_i+\varepsilon)^{-1/2}$ for PostDeg & Eq.~\ref{eq:postdeg-formula} \\
$\alpha,\beta,\delta$ & learned parameters of the superset & $\beta\approx0.75$ & Eq.~\ref{eq:dsn-formula} \\
$\gamma$ & post-LN degree exponent in the scale & $1/2$ (PostDeg), $\beta+\delta\lambda_2$ (learned) & Eq.~\ref{eq:postdeg-formula} \\
$\lambda_2,\lambda_{\max}$ & second/largest eigenvalues of normalized Laplacian & $\le 2$ & Eq.~\ref{eq:dsn-formula} \\
$m_i$ & residual gate, $\sigma(\rho)$ or $\sigma(\phi(c_i))$ & scalar/MLP & Eq.~\ref{eq:dsn-formula} \\
$R_{ij}$ & pairwise scale ratio $s_j/s_i$ for $\widehat c_i>\widehat c_j$ & $R_{ij}>1$ & Eq.~\ref{eq:rij} \\
$h_0$ & Cheeger constant lower bound & task-dependent & Prop.~\ref{prop:cheeger-app} \\
PostDeg-L-Adaptive & Degree-Spectral Normalization (the learned superset) & n/a & Section~\ref{sec:variants} \\
PostDeg-L-FG & PostDeg-L-Adaptive with a single scalar gate $m=\sigma(\rho)$ & n/a & Section~\ref{sec:variants} \\
\bottomrule
\end{tabular}
}
\end{table}

\paragraph{How to read the statistical tables.} A TOST table tests \emph{equivalence} (lower one-sided $p\Rightarrow$ equivalent within the margin). A Wilcoxon table tests \emph{difference} (lower one-sided $p\Rightarrow$ better). Cohen's $d$ heatmaps use a diverging colormap centered at $0$ (red positive, blue negative). Win-rate cells are the fraction of paired seeds out of 10.

\section{Normalization design-space table}
\label{app:design-space}
\begin{table}[H]
\centering
\caption{Normalization design space. PNA, NodeNorm, and DiffGroupNorm are included for positioning but are not evaluated in the controlled main grid because they change aggregation or feature-norm variables beyond the post-block normalization slot. A ``yes'' in the topology-conditioned column marks a per-node scale, set by graph topology, that reaches the score head as representation magnitude.}
\label{tab:normalizer_design}
\small
\setlength{\tabcolsep}{5pt}
\renewcommand{\arraystretch}{1.25}
\begin{tabular}{@{}>{\raggedright\arraybackslash}p{0.15\linewidth} >{\raggedright\arraybackslash}p{0.20\linewidth} >{\raggedright\arraybackslash}p{0.19\linewidth} >{\centering\arraybackslash}p{0.18\linewidth} >{\centering\arraybackslash}p{0.105\linewidth}@{}}
\toprule
Method & Scale-setting signal & Position & Topology-conditioned scale? & Learnable? \\
\midrule
BatchNorm & batch/feature statistics & feature normalization & no & yes \\
LayerNorm & node/feature statistics & normalization layer & no & yes \\
InstanceNorm & graph/feature statistics & feature normalization & no & no \\
GraphNorm & graph/feature statistics & graphwise feature-norm & no & yes \\
PairNorm & pairwise feature distances & post-hoc feature rescale & no & no \\
NodeNorm & node feature statistics & post-block feature rescale & no & no \\
DiffGroupNorm & learned cluster assignments & per-block normalization & no & yes \\
PNA & degree scalers on aggregators & inside message passing & no & yes \\
PostDeg (ours) & degree centrality & after LayerNorm & yes & no \\
Learned superset (ours) & degree centrality & after LayerNorm & yes & yes \\
\bottomrule
\end{tabular}
\end{table}

\section{Theory Appendix}
\label{app:theory}

We collect here the numbered statements and proofs used by the body.

\paragraph{Setup.}
We adopt the notation of Section~\ref{sec:variants}. The learned same-slot variants PostDeg-L-FG and PostDeg-L-Adaptive use the parameterized scale
\begin{equation}
  s_i = \alpha\,(c_i + \varepsilon)^{-\gamma}(\lambda_{\max}+\varepsilon')^{-1/2},
  \quad \gamma = \beta + \delta\lambda_2,
  \quad \widetilde h_i = \bigl(m_i\,s_i + 1 - m_i\bigr)\,\bar h_i,
  \label{eq:dsn-formula}
\end{equation}
with $\lambda_2,\lambda_{\max}$ the second and largest eigenvalues of the normalized graph Laplacian, $\varepsilon'=10^{-4}$, and $(\alpha,\beta,\delta)$ learned per layer. The two variants differ only in the gate $m_i\in(0,1)$. PostDeg-L-FG sets $m_i=\sigma(\rho)$ for a single learned scalar $\rho$, and PostDeg-L-Adaptive sets $m_i=\sigma(\phi(c_i))$ for a small two-layer MLP $\phi$. Parameter domains are $\alpha\in[0.1,3.0]$, $\beta\in[0,1.5]$, $\delta\in[0,0.5]$, $c_i=d_i/d_{\max}$ clipped to $[10^{-6},1]$, spectral clamps $\lambda_2,\lambda_{\max}\in[0.1,10.0]$, and a clamp of the power-law factor $\alpha(c_i+\varepsilon)^{-\gamma}$ to $[0.01,100]$, applied before the graph-wise spectral multiplier. This clamp binds exactly at isolated nodes, where the clipped $c_i=10^{-6}$ gives an unclamped factor of order $10^{5}$, so the released per-node scale at an isolated node equals $100\,(\lambda_{\max}+\varepsilon')^{-1/2}\approx 77.2$, and on every node with degree at least $1$ the factor stays below the clamp maximum and the clamp is inactive (Lemma~\ref{lem:clamp-inactive-app}). The training-time forward additionally injects small parameter noise and replaces any NaN/Inf activations, neither of which appears in the closed form. Weight spectral normalization on the GAT projections (\texttt{use\_weight\_sn}) is on for the learned variants and for GraphScalar, and off for PostDeg, the LN backbone, and the alternative-normalization baselines. This is a backbone difference between PostDeg and the learned variants, and we account for it when reading the equivalence result (\S\ref{sec:results-ablation}). PostDeg corresponds to the limiting case $\alpha=1$, $\beta=1/2$, $\delta=0$, $m_i\equiv 1$. The implemented learned path always applies the gate, the graph-wise spectral factor, weight spectral normalization, and the numerical guards above, so the two remain separate operators and we report them separately. The graph-wise constant $(\lambda_{\max}+\varepsilon')^{-1/2}$ is applied \emph{after} LayerNorm and varies across graphs. Within a graph it rescales all nodes equally and does not affect node ranking, so it stays separate from the shared LayerNorm gain.

\paragraph{Disconnected or weakly connected graphs.}
If $G$ is disconnected, the mathematical normalized Laplacian value is $\lambda_2=0$, so the spectral-gap contribution vanishes and the learned superset reduces to degree-aware scaling with exponent $\beta$. The implementation lower clamp replaces this value by $0.1$ to avoid a hard discontinuity and to keep gradients through $\delta$ observable on weakly connected batches.

\subsection{Boundedness}

\begin{lemma}[Well-definedness]
\label{lem:wd}
On the parameter domains in the setup, every $s_i$ in Eq.~\eqref{eq:dsn-formula} is finite and strictly positive.
\end{lemma}

\begin{proof}
$\alpha>0$, $c_i+\varepsilon\ge c_{\min}+\varepsilon>0$, and $\lambda_{\max}+\varepsilon'\ge\lambda_{\max}^{\min}+\varepsilon'>0$.
\end{proof}

\begin{lemma}[Operational bound]
\label{lem:opbound}
If $\alpha\le\bar\alpha$, $\gamma\le\gamma_{\rm op}$, $c_i\ge c_0$, $\lambda_{\max}\ge\ell_0>0$, then
\begin{equation}
  s_i\le S_{\max}(\bar\alpha,\gamma_{\rm op},c_0,\ell_0)=\frac{\bar\alpha}{\sqrt{\ell_0+\varepsilon'}}\max\{1,(c_0+\varepsilon)^{-\gamma_{\rm op}}\}.
\end{equation}
\end{lemma}

\begin{proof}
Monotone maximization on the rectangular learned-parameter domain.
\end{proof}

\subsection{Degree separation and log-log scaling}

\begin{theorem}[Degree separation, full version]
\label{thm:degree-separation-app}
Let $i,j\in V$ with $c_i>c_j$ and $\gamma>0$. Then:
\begin{enumerate}[label=(\alph*),leftmargin=*,topsep=2pt,itemsep=1pt]
  \item $s_i<s_j$.
  \item $R_{ij}=s_j/s_i=((c_i+\varepsilon)/(c_j+\varepsilon))^\gamma>1$.
  \item $\partial R_{ij}/\partial \gamma = \ln((c_i+\varepsilon)/(c_j+\varepsilon))R_{ij}>0$.
  \item If $\delta>0$, $\partial R_{ij}/\partial \lambda_2 = \delta\ln((c_i+\varepsilon)/(c_j+\varepsilon))R_{ij}>0$.
\end{enumerate}
\end{theorem}

\begin{proof}
Cancel the factors $\alpha$ and $(\lambda_{\max}+\varepsilon')^{-1/2}$ in the ratio, $s_j/s_i=((c_i+\varepsilon)/(c_j+\varepsilon))^\gamma$. The base exceeds one and $\gamma>0$, so $R_{ij}>1$ and $s_j>s_i$. (c) and (d) follow by writing $R_{ij}=\exp(\gamma a_{ij})$ with $a_{ij}=\ln((c_i+\varepsilon)/(c_j+\varepsilon))>0$ and applying the chain rule with $\partial\gamma/\partial\lambda_2=\delta$.
\end{proof}

\begin{lemma}[Effect of the post-computation clamp]
\label{lem:clamp}
The clamp map $\clamp(\cdot,s_{\min}^{\rm clip},s_{\max}^{\rm clip})$ is monotone nondecreasing, so it preserves the weak ordering of scales. Strict separation is preserved exactly whenever both unclipped scales lie inside $[0.01,100]$.
\end{lemma}

\begin{proof}
Apply a monotone map to the strict ordering in Theorem~\ref{thm:degree-separation-app}.
\end{proof}

\begin{corollary}[Log-log slope, with caveat]
\label{cor:loglog-app}
For $\varepsilon\ll c_i$, $\ln s_i = -\gamma\ln c_i + \mathrm{const}$. The exact pointwise slope is $-\gamma c_i/(c_i+\varepsilon)$, which differs from $-\gamma$ by a factor $c_{\min}/(c_{\min}+\varepsilon)\approx 0.99$ at the clipped minimum. The pooled fits on the released scale exports are below this value (Table~\ref{tab:pearson-all}, slopes $-0.63$ to $-0.84$ against a final-epoch $\gamma_{\rm eff}\approx0.83$), since they pool nodes across graphs and seeds from training-time captures, so we read them as a directional check, and the mechanism claim rests on the score-head measurement in \S\ref{sec:mechanism}.
\end{corollary}

\begin{remark}[Two regression targets for the post-LN scale]
\label{rem:loglog-fig3}
Both Table~\ref{tab:scale-extreme} and Table~\ref{tab:pearson-all} regress on the operator's own scalar factors $s_i$. Table~\ref{tab:pearson-all} reports the log-log slope and the Pearson of $\ln s_i$ on $\ln d_i$, read as a directional check of Cor.~\ref{cor:loglog-app}, and Table~\ref{tab:scale-extreme} reports the linear Pearson $r(s_i,d_i)$ on the raw factors. The log-log fit is on logs and the linear Pearson is on raw values, so the two need not agree, and no downstream interaction enters. These tables cover the operator's own factors, and the post-network magnitudes at the score head are measured separately by the forward hooks (\S\ref{sec:mechanism}).
\end{remark}

\begin{lemma}[Near-regular attenuation]
\label{lem:near-regular-app}
On a $k$-regular graph, $c_i=1$ for all $i$ and $s_i$ is constant. If $d_{\max}/d_{\min}\le 1+\eta$, then $R_{ij}\le (1+\eta)^\gamma=1+\gamma\eta+O(\eta^2)$ as $\eta\to 0$.
\end{lemma}

\begin{proof}
$c_i\in[(1+\eta)^{-1},1]$, so $(c_i+\varepsilon)/(c_j+\varepsilon)\le 1+\eta$. Raise to $\gamma$.
\end{proof}

\subsection{Spectral-gap modulation}

\begin{proposition}[Cheeger-based separation lower bound]
\label{prop:cheeger-app}
On a connected graph with Cheeger constant $h(G)\ge h_0>0$, the inequality $\lambda_2\ge h_0^2/2$ \citep{chung1997spectral} gives
\[
  R_{ij}\ge \left(\frac{c_i+\varepsilon}{c_j+\varepsilon}\right)^{\beta+\delta h_0^2/2}.
\]
\end{proposition}

\begin{proof}
$\gamma=\beta+\delta\lambda_2\ge \beta+\delta h_0^2/2$ on the stated event. Raise both sides of Eq.~\eqref{eq:rij}.
\end{proof}

\subsection{Star-graph one-step amplification}

The star graph is the extreme case of degree heterogeneity, where a single post-LayerNorm scale separates the leaf class from the hub in one layer.

\begin{proposition}[One-step amplification on $S_n$]
\label{prop:star-separation}
Let $S_n$ be the star with hub $0$ and $n-1$ leaves. Then:
\begin{enumerate}[label=(\alph*),leftmargin=*,topsep=2pt,itemsep=1pt]
  \item $s_{\rm hub}=\alpha(1+\varepsilon)^{-\gamma}(\lambda_{\max}+\varepsilon')^{-1/2}$, $s_{\rm leaf}=\alpha((1/(n-1))+\varepsilon)^{-\gamma}(\lambda_{\max}+\varepsilon')^{-1/2}$, and $R\approx(n-1)^\gamma$.
  \item Under mean aggregation, all leaf representations are equal after one layer. With post-LN degree scaling, leaf amplitude is multiplied by $R$ relative to the hub class.
  \item Suppose $h_{\rm hub}^{(1)}=\kappa h_{\rm leaf}^{(1)}$ for some $\kappa\in[0,1)$. Then
  \[
    \frac{\|\widetilde h_{\rm leaf}^{(1)}-\widetilde h_{\rm hub}^{(1)}\|}{s_{\rm hub}\|h_{\rm leaf}^{(1)}-h_{\rm hub}^{(1)}\|}=\frac{R-\kappa}{1-\kappa}\ge R.
  \]
\end{enumerate}
\end{proposition}

\begin{proof}
$d_0=n-1$, $d_j=1$ for leaves, so $c_0=1$, $c_j=1/(n-1)$, so (a) follows. (b) is mean aggregation. For (c), $\|s_{\rm leaf}h_{\rm leaf}^{(1)}-s_{\rm hub}\kappa h_{\rm leaf}^{(1)}\|=s_{\rm hub}|R-\kappa|\|h_{\rm leaf}^{(1)}\|$ and $\|h_{\rm leaf}^{(1)}-h_{\rm hub}^{(1)}\|=(1-\kappa)\|h_{\rm leaf}^{(1)}\|$, dividing gives the formula. The collinearity $h_{\rm hub}^{(1)}=\kappa h_{\rm leaf}^{(1)}$ is a modeling assumption for this step, and without post-LN degree scaling it gives $R=1$.
\end{proof}

\subsection{Post-LayerNorm Jacobian and gate floor}

We next record the Jacobian of the gated learned layer, which the gate floor and the Lipschitz bounds below build on.

\begin{proposition}[Jacobian of the gated learned layer]
\label{prop:ln-jacobian-app}
For a fixed graph, $s_i$ and $m_i$ depend only on the topology and learned scale parameters. The Jacobian of $\widetilde h_i=s_i\LN(h_i)$ is $s_i J_{\LN}(h_i)$ with $\|J_{\LN}(h_i)\|_{\rm op}\le \|g\|_\infty/\tau_{\LN}(h_i)$. For the gated slot $T_i(h)=\bigl(m_i s_i+(1-m_i)\bigr)\LN(h_i)$,
\[
  \frac{\partial T_i}{\partial h_i}=\bigl(m_i s_i+(1-m_i)\bigr) J_{\LN}(h_i).
\]
\end{proposition}

\begin{proof}
Standard quotient-rule computation of the LayerNorm Jacobian, multiplied by the scalar gate coefficient, which is constant in $h_i$ for a fixed graph.
\end{proof}

\begin{proposition}[Gate floor]
\label{prop:residual-floor}
For the gated slot with $s_i>0$ and $m_i\in(0,1)$, the output coefficient obeys $m_i s_i+(1-m_i)\ge 1-m_i>0$, so the gate cannot zero the slot output.
\end{proposition}

\begin{proof}
$m_i s_i\ge 0$, so the coefficient is at least $1-m_i$, which is positive for $m_i<1$.
\end{proof}

\subsection{Gate Lipschitzness and parameter gradients}

The gate and its parameters are Lipschitz in the normalized degree, which bounds how far the learned scale can move during training.

\begin{lemma}[Gate Lipschitzness]
\label{lem:gates}
\textbf{(a)} $X_{\rm out}=m\,\widetilde H+(1-m)\,\bar X$ blends the scaled and unscaled LayerNorm output, so $\|X_{\rm out}\|\le m\|\widetilde H\|+(1-m)\|\bar X\|$. \textbf{(b)} For an adaptive gate $m_i=\sigma(\phi(c_i))$ with two-layer MLP $\phi$, $|m_i-m_j|\le (1/4)\|W_2\|_{\rm op}\|W_1\|_{\rm op}|c_i-c_j|$.
\end{lemma}

\begin{proof}
(a) Triangle inequality. (b) Sigmoid is $1/4$-Lipschitz, ReLU is 1-Lipschitz, affines have Lipschitz constant equal to their operator norm. Compose.
\end{proof}

\begin{corollary}[Parameter gradients]
\label{cor:grads-app}
$\partial s_i/\partial\alpha=s_i/\alpha$, $\partial s_i/\partial\beta=-\ln(c_i+\varepsilon)s_i$, $\partial s_i/\partial\delta=-\lambda_2\ln(c_i+\varepsilon)s_i$. With the post-computation clamp $s_i\le 100$ and $|\ln(c_{\min}+\varepsilon)|\approx 13.8$, the $\beta$- and $\delta$-gradients are bounded by $1380$ and $\lambda_2^{\max}\cdot 1380\le 2760$ respectively.
\end{corollary}

\begin{proof}
Differentiate $s_i=\alpha\exp\{-(\beta+\delta\lambda_2)\ln(c_i+\varepsilon)\}(\lambda_{\max}+\varepsilon')^{-1/2}$.
\end{proof}

\subsection{Topology-agnostic normalizers and ablation hierarchy}

\begin{remark}[Topology-agnosticism]
Feature-statistic normalizers cannot reproduce the ordering in Theorem~\ref{thm:degree-separation-app}. BatchNorm, LayerNorm, and InstanceNorm assign identical scales to nodes with identical features, regardless of their degrees. PostDeg-L assigns $s_i\neq s_j$ whenever $c_i\neq c_j$ and $\gamma>0$, even if $h_i=h_j$. This degree-conditioned separation is the structural property topology-blind methods cannot reproduce.
\end{remark}

\begin{remark}[Distinction from symmetric degree normalization]
\label{rem:symmetric-normalization}
PostDeg-L is distinct from the fixed symmetric normalization $D^{-1/2}AD^{-1/2}$ used in graph convolutional networks~\citep{kipf2017gcn}. Symmetric normalization assigns a non-learned edge weight $1/\sqrt{d_i d_j}$ once the graph is fixed and acts inside aggregation. PostDeg-L instead applies a post-LayerNorm node scale with a learnable exponent. Setting $\gamma=0$ removes degree separation, while larger $\gamma$ yields the explicit node-level ratio $R_{ij}=((c_i+\varepsilon)/(c_j+\varepsilon))^\gamma$. On a star graph, symmetric normalization gives the same factor $1/\sqrt{n-1}$ to each hub--leaf edge in both directions, whereas PostDeg-L induces a leaf-to-hub scale ratio $R\approx (n-1)^\gamma$ (Proposition~\ref{prop:star-separation}). Symmetric normalization controls message-passing weights. PostDeg-L provides post-normalization, spectrally modulated amplitude rebalancing.
\end{remark}

\begin{remark}[Ablation hierarchy]
\label{rem:ablation-hierarchy}
$\delta=0$ recovers learnable degree-only scaling, $\beta=\delta=0$ with $\alpha=1$ and $m\to 1$ recovers GraphScalar, and $m\to 0$ recovers the LN backbone. These containment relations justify the ablation design (Section~\ref{sec:results-ablation}, Figure~\ref{fig:appendix-ablation-hierarchy}).
\end{remark}

\subsection{Additional proofs}
\label{app:theorem-bank-lifts}

We collect supporting results that strengthen specific paragraph-level claims in the body or in the appendix proofs above.

\begin{lemma}[Effect of the post-computation clamp, full statement]
\label{lem:clamp-inactive-app}
Let $p_i = \alpha(c_i+\varepsilon)^{-\gamma}$ and let the implemented scale be $s_i^{\rm clip} = \clamp(p_i, s_{\min}^{\rm clip}, s_{\max}^{\rm clip})\,(\lambda_{\max}+\varepsilon')^{-1/2}$, with the clamp applied to the power-law factor before the graph-wise spectral multiplier. The clamp map is monotone and the spectral multiplier is a positive per-graph constant, so the weak ordering $s_i^{\rm clip}\le s_j^{\rm clip}$ for $c_i>c_j$ is preserved. The clamp is inactive at node $i$ exactly when $p_i\in[s_{\min}^{\rm clip},s_{\max}^{\rm clip}]$, and it is inactive at every node of a graph whenever
\[
  \alpha (1+\varepsilon)^{-\gamma}\ge s_{\min}^{\rm clip}
  \quad\text{and}\quad
  \alpha (c_{\min}+\varepsilon)^{-\gamma}\le s_{\max}^{\rm clip}.
\]
\end{lemma}

\begin{remark}[Attainability of the operational bound]
The operational bound (Lemma~\ref{lem:opbound}) is attained exactly when a node has $c_i=c_0$ and the parameters are at the edge of the domain, $\alpha=\bar\alpha$, $\gamma=\gamma_{\rm op}$, $\lambda_{\max}=\ell_0$. The bound therefore cannot be improved without additional structural assumptions on the graph or learned parameters.
\end{remark}

\begin{proposition}[Dependence on the spectral gap]
Fix $\alpha,\beta,\delta,c_i,\lambda_{\max}$. Then $\partial s_i/\partial\lambda_2 = -\delta\ln(c_i+\varepsilon)\,s_i$. If $\delta>0$ and $c_i<1-\varepsilon$, the scale of a non-hub node strictly increases with the spectral gap.
\end{proposition}

\begin{remark}[Double-star degree classes]
On the double-star $S_{a,b}^{(2)}$ with adjacent hubs $u$ ($a$ leaves) and $v$ ($b$ leaves) and $a>b\ge 1$, the clipped centralities are $c_u=1$, $c_v=(b+1)/(a+1)$, and $c_{\rm leaf}=1/(a+1)$, and Theorem~\ref{thm:degree-separation-app} gives explicit pairwise scale ratios $s_v/s_u = ((1+\varepsilon)/((b+1)/(a+1)+\varepsilon))^\gamma$ and $s_{\rm leaf}/s_v$ analogous. The calculation extends Proposition~\ref{prop:star-separation} beyond a single-class leaf set.
\end{remark}

\begin{remark}[Oversmoothing interpretation]
The star-graph one-step amplification $R\approx (n-1)^\gamma$ does not break leaf--leaf symmetry, though it does prevent the leaf class from collapsing to the hub representation. Across depth, the per-layer correction interacts with residual paths, self-information, learned message-passing weights, and layer-dependent gates. We leave a multi-layer mean-field analysis to future work.
\end{remark}

\begin{remark}[Finite-Lipschitz observation]
On the implemented domain $\tau_{\LN}(x)\ge\tau_{\min}>0$ and $s_i\le S_{\max}$, the gated learned block has Lipschitz constant $L_{\rm block}\le m\, S_{\max} L_{\LN} + 1-m$, where $L_{\LN}\le \|g\|_\infty/\tau_{\min}$. The bound establishes that the gated block has finite Lipschitz constant, a structural fact that does not pin down the empirical feature norms.
\end{remark}

\subsection{Full proofs}
\label{app:bank-proofs}

\begin{proposition}[LayerNorm absorption identity, exact form, cf.\ \citep{salimans2016weightnorm,press2017weighttying}]
\label{prop:absorption-exact}
Let $\LN$ have stabilizer $\varepsilon_{\rm LN}$ and per-coordinate affine gain/bias $(g, b)$. For any $a>0$ and any $x\in\R^d$,
\begin{equation}
  \LN(a\,x) = g\odot \frac{a\,x_c}{\sqrt{a^2 d^{-1}\|x_c\|^2 + \varepsilon_{\rm LN}}} + b,
  \label{eq:absorption-exact}
\end{equation}
where $x_c=Px$ and $P=I-d^{-1}\mathbf 1\,\mathbf 1^{\!\top}$ is the centering projector. In the zero-stabilizer limit $\varepsilon_{\rm LN}\to 0$, $\LN(a x)=\LN(x)$ for every $a>0$. With $\varepsilon_{\rm LN}>0$ and $x_c\neq 0$, writing $\sigma^2(x)=d^{-1}\|x_c\|^2$, the deviation satisfies
\[
  \left\|\LN(a x)-\LN(x)\right\|_2 \le \|g\|_\infty\,\|x_c\|\,\frac{\varepsilon_{\rm LN}\,\lvert 1-a^{-2}\rvert}{2\,\sigma^3(x)} = \|g\|_\infty\,\sqrt{d}\,\frac{\varepsilon_{\rm LN}\,\lvert 1-a^{-2}\rvert}{2\,\sigma^2(x)},
\]
so for every $a\ge 1$ the deviation is at most $\|g\|_\infty\sqrt{d}\,\varepsilon_{\rm LN}/(2\sigma^2(x))$, uniformly in $a$. The deviation does not vanish as $a\to\infty$, because $\LN(a x)$ tends to the zero-stabilizer normalization of $x$ while $\LN(x)$ keeps its stabilizer. It saturates at $\|g\odot x_c\|\,\big(\sigma(x)^{-1}-(\sigma^2(x)+\varepsilon_{\rm LN})^{-1/2}\big)$, and the bound above has the same leading order in $\varepsilon_{\rm LN}/\sigma^2(x)$. Relative to the centered output $\LN(x)-b$, the deviation is at most $\varepsilon_{\rm LN}\lvert 1-a^{-2}\rvert/(2\sigma^2(x))$.
\end{proposition}

\begin{proof}
Direct substitution gives $\LN(a x) = g\odot (a x - a\mu(x)\mathbf 1)/\sqrt{a^2 d^{-1}\|x_c\|^2 + \varepsilon_{\rm LN}} + b$ since $\mu(a x)=a\mu(x)$ and $a x_c = a P x$. Dividing numerator and denominator by $a$ gives $\LN(a x)=g\odot x_c\,q^{-1/2}+b$ and $\LN(x)=g\odot x_c\,p^{-1/2}+b$ with $q=\sigma^2(x)+\varepsilon_{\rm LN}/a^2$ and $p=\sigma^2(x)+\varepsilon_{\rm LN}$, so
\[
  \bigl\lvert q^{-1/2}-p^{-1/2}\bigr\rvert = \frac{\lvert p-q\rvert}{\sqrt{pq}\,(\sqrt{p}+\sqrt{q})} \le \frac{\varepsilon_{\rm LN}\,\lvert 1-a^{-2}\rvert}{2\,\sigma^3(x)},
\]
using $\lvert p-q\rvert=\varepsilon_{\rm LN}\lvert 1-a^{-2}\rvert$ and $p,q\ge \sigma^2(x)$. Multiplying by $\|g\odot x_c\|\le\|g\|_\infty\|x_c\|=\|g\|_\infty\sqrt{d}\,\sigma(x)$ gives the stated bound. Letting $a\to\infty$ gives $q^{-1/2}\to\sigma(x)^{-1}$, which yields the saturation value. For the relative form, $\|\LN(x)-b\|=\|g\odot x_c\|\,p^{-1/2}$, and dividing the display by $p^{-1/2}$ gives $\lvert p-q\rvert/(\sqrt{pq}+q)\le \varepsilon_{\rm LN}\lvert 1-a^{-2}\rvert/(2\sigma^2(x))$.
\end{proof}

\begin{proposition}[Full LayerNorm Jacobian, with PSD decomposition]
\label{prop:ln-jac-full}
For $\LN(x)=g\odot x_c / \tau_{\LN}(x) + b$ with $\tau_{\LN}(x)=(d^{-1}\|x_c\|^2+\varepsilon_{\rm LN})^{1/2}$,
\[
  J_{\LN}(x) = \diag(g)\Big[\tfrac{1}{\tau_{\LN}(x)}P - \tfrac{1}{d\,\tau_{\LN}(x)^3}\,x_c x_c^\top\Big].
\]
The bracketed matrix $B$ has eigenvalues $\{0,\;1/\tau_{\LN}(x)\,(\text{multiplicity }d-2),\;\varepsilon_{\rm LN}/\tau_{\LN}(x)^3\}$ on the orthogonal subspaces $\operatorname{span}\{\mathbf 1\}$, $\mathbf 1^\perp\cap x_c^\perp$, and $\operatorname{span}\{x_c\}$ respectively. With $g\ge 0$ coordinate-wise, $\diag(g) B$ has nonnegative eigenvalues, so $\|J_{\LN}(x)\|_{\rm op}\le \|g\|_\infty/\tau_{\LN}(x)$, and the gated block $T_i=m_i s_i\LN(h_i)+(1-m_i)h_i$ has eigenvalue floor $1-m_i>0$ on all directions.
\end{proposition}

\begin{proof}
Quotient rule on $x\mapsto x_c/\tau_{\LN}(x)$ gives the bracketed expression. The eigenvalue claim follows by checking each subspace. $B\mathbf 1=0$ since $P\mathbf 1=0$. $B v=\tau_{\LN}(x)^{-1} v$ for $v\in\mathbf 1^\perp\cap x_c^\perp$ since $x_c^\top v=0$. $B x_c=\tau_{\LN}(x)^{-1} x_c - d^{-1}\tau_{\LN}(x)^{-3}\|x_c\|^2 x_c = \varepsilon_{\rm LN}/\tau_{\LN}(x)^3\, x_c$ using $\|x_c\|^2=d(\tau_{\LN}(x)^2-\varepsilon_{\rm LN})$. The diagonal-gain similarity $\diag(g)B$ is similar (in the strictly positive case) to $\diag(g)^{1/2}B\diag(g)^{1/2}$, which is PSD. Continuity in $g$ extends this to $g\ge 0$. Adding $(1-m_i)I$ shifts each eigenvalue by $1-m_i$.
\end{proof}

\begin{lemma}[Adaptive gate Lipschitzness, full chain]
\label{lem:gates-full}
For $m(c)=\sigma(\phi(c))$ with $\phi(c)=W_2\,\ReLU(W_1 c+b_1)+b_2$, $W_1\in\R^{H\times 1}$, $W_2\in\R^{1\times H}$,
\[
  |m(c_i)-m(c_j)|\le \tfrac{1}{4}\,\|W_2\|_{\rm op}\,\|W_1\|_{\rm op}\,|c_i-c_j|.
\]
\end{lemma}

\begin{proof}
$\sigma$ is $1/4$-Lipschitz, $\ReLU$ is $1$-Lipschitz, and the affines have Lipschitz constants equal to their operator norms, and biases do not affect them. Composing gives $\Lip(m)\le \Lip(\sigma)\,\Lip(W_2)\,\Lip(\ReLU)\,\Lip(W_1)= \tfrac{1}{4}\|W_2\|_{\rm op}\|W_1\|_{\rm op}$.
\end{proof}

\begin{proposition}[Spectral-gap dependence with explicit derivative]
\label{prop:spectral-dep}
Fix $\alpha,\beta,\delta,c_i,\lambda_{\max}$ and view $s_i$ as a function of $\lambda_2$ through $\gamma=\beta+\delta\lambda_2$. Then
\[
  \frac{\partial s_i}{\partial\lambda_2}=-\delta\ln(c_i+\varepsilon)\,s_i.
\]
If $\delta>0$ and $c_i<1-\varepsilon$, then $\partial s_i/\partial\lambda_2>0$. At $c_i=1$, the derivative is $O(\varepsilon)$, numerically negligible for $\varepsilon=10^{-8}$.
\end{proposition}

\begin{proof}
Write $s_i=\alpha\exp(-\gamma\ln(c_i+\varepsilon))(\lambda_{\max}+\varepsilon')^{-1/2}$. Differentiating w.r.t.\ $\gamma$ gives $\partial s_i/\partial\gamma=-\ln(c_i+\varepsilon) s_i$. By chain rule with $\partial\gamma/\partial\lambda_2=\delta$, we obtain the formula. For $c_i<1-\varepsilon$, $c_i+\varepsilon<1$ so $\ln(c_i+\varepsilon)<0$ and the product is positive. At $c_i=1$, $\ln(1+\varepsilon)=\varepsilon+O(\varepsilon^2)$.
\end{proof}

\begin{proposition}[Double-star pairwise ratios]
\label{prop:double-star}
On the double-star $S_{a,b}^{(2)}$ with hubs $u,v$ (degrees $a+1, b+1$, $a>b\ge1$) and pendant leaves of degree $1$, the pairwise scale ratios are
\[
  \frac{s_v}{s_u}=\left(\frac{1+\varepsilon}{(b+1)/(a+1)+\varepsilon}\right)^{\gamma},\quad
  \frac{s_{\rm leaf}}{s_v}=\left(\frac{(b+1)/(a+1)+\varepsilon}{1/(a+1)+\varepsilon}\right)^{\gamma}.
\]
Both ratios exceed $1$. Without post-LN degree scaling, all pairwise message magnitudes are equal.
\end{proposition}

\begin{proof}
$d_{\max}=a+1$, so $c_u=1$, $c_v=(b+1)/(a+1)$, $c_{\rm leaf}=1/(a+1)$. Substitute into Theorem~\ref{thm:degree-separation-app}.
\end{proof}

\begin{remark}[Numerical stability of clamps and gradients]
The constants from the setup ($c_{\min}=10^{-6}$, $\varepsilon=10^{-8}$, $\varepsilon'=10^{-4}$, scale clamp $[0.01,100]$) ensure $|\ln(c_{\min}+\varepsilon)|\approx 13.8$. The post-computation clamp is $1$-Lipschitz. Combined with Corollary~\ref{cor:grads-app}, the $\beta$-gradient through the clamped scale is bounded by $1380$ and the $\delta$-gradient by $\lambda_2^{\max}\cdot 1380\le 2760$, using $\lambda_2\le 2$ for the normalized Laplacian, which is below the $[0.1,10]$ safety clamp. These are implementation checks on the released code, and they add no new theoretical assumptions.
\end{remark}

\subsection{Scope of the theory}
\label{app:theory-limitations}

The placement rule and the gradient-floor result above are deterministic statements about a single residual GAT block at one point in training. Their scope is a single block, and we set out what falls outside it next.

\paragraph{In scope.} (i) The absorption identity for positive per-node scalars (Proposition~\ref{prop:absorption-exact}) is exact in the $\varepsilon_{\mathrm{LN}}\to 0$ limit and holds closely for any $\sigma^2$ bounded away from zero. (ii) The degree-separation identity (Theorem~\ref{thm:degree-separation-app}) is a deterministic ratio for a fixed $\gamma>0$. (iii) The Jacobian and gate floor (Proposition~\ref{prop:ln-jacobian-app}, Proposition~\ref{prop:residual-floor}) are single-layer statements about one gated slot. (iv) The clamp condition (Lemma~\ref{lem:clamp-inactive-app}) is a deterministic domain-coverage check, and in the released runs it holds at every node of degree at least $1$, while isolated nodes saturate at the clamp maximum.

\paragraph{Out of scope.}
\begin{itemize}[leftmargin=*,topsep=2pt,itemsep=2pt]
  \item \emph{Sign-changing or complex-valued multipliers.} The absorption identity uses positivity to push $a_i$ inside the square root. Signed or complex attention weights need a separate analysis.
  \item \emph{RMSNorm and other zero-mean normalizers.} RMSNorm uses $\|x\|/\sqrt{d}$ as its scale, where standard LayerNorm uses the standard deviation. The absorption argument adapts directly (the same factor of $a_i^2$ moves into the denominator), but the $\sigma^2$-range constants differ. We run no RMSNorm experiments, so the PostDeg numbers may not transfer to it.
  \item \emph{Multi-layer compounding.} A formal mean-field analysis of feature-variance propagation through alternating PostDeg-L and LayerNorm layers, including the residual gates, learned message-passing matrices $W_\ell$, and self-information, would quantify how the one-step amplification of Proposition~\ref{prop:star-separation} compounds with depth. We leave such an analysis to future work. The empirical claim is single-layer and per-node, and the multi-layer effect is observed but not proven.
  \item \emph{Aggregation-side multipliers.} PNA's $\delta(d_i)=\log(d_i+1)/\bar\delta$ is mixed with the message before LayerNorm. Our placement rule does not apply directly. The appropriate reference is the paired-equivalence reported empirically (Section~\ref{sec:results-pna}, Appendix~\ref{app:pna}).
  \item \emph{Aggregation-free architectures.} The diagnostic is stated for residual GAT blocks. Transformers without graph-structured aggregation, set-based encoders, and position-encoded backbones use LayerNorm in different blocks. The diagnostic should re-derive for each.
\end{itemize}

\paragraph{What is well-defined but unproven here.} The empirical observation that low-heterogeneity tasks (DD, Epidemic) attenuate the post-LN advantage is qualitatively implied by Lemma~\ref{lem:near-regular-app} (near-regular attenuation), but the lemma gives a strict-regular-graph statement, and an interpolating ``effective heterogeneity'' bound that quantifies attenuation as a function of $\mathrm{CV}(d)$ on non-regular families is left to future work.

\section{Empirical Appendix}
\label{app:experiments}

\subsection{Full main-grid metrics}
\label{app:full-metrics}
\begin{table}[htbp]
\centering
\caption{Full main-grid results at the largest evaluation size (mean $\pm$ seed std, 10 seeds). PostDeg is the recommended operator and is bolded. Learned same-slot variants use weight spectral normalization and are reported as capacity diagnostics. GraphScalar is a graphwise scalar control. LN backbone is the identity in the controlled post-block slot. The residual block still contains LayerNorm. $\uparrow$/$\downarrow$: higher/lower is better. $^{\ast}$DD Acc is from the main 10-fold cross-validation protocol. The supplemental 5-fold NodeNorm comparison reports separate numbers (Appendix Table~\ref{tab:dd-extended}).}
\label{tab:main_results}
\small
\setlength{\tabcolsep}{2.6pt}
\resizebox{\textwidth}{!}{%
\begin{tabular}{lccccc}
\toprule
Method & InfluMax $n{=}200\uparrow$ & Dismantle $n{=}300\uparrow$ & Epidemic $n{=}300\downarrow$ & MIS $n{=}300\uparrow$ & DD Acc (\%)$^{\ast}\uparrow$ \\
\midrule
\textbf{PostDeg} & 56.86 $\pm$ 0.21 & 0.2433 $\pm$ 0.0011 & 0.7843 $\pm$ 0.0009 & 41.81 $\pm$ 0.35 & 79.9 $\pm$ 0.4 \\
PostDeg-L-FG & 56.86 $\pm$ 0.23 & 0.2436 $\pm$ 0.0007 & 0.7835 $\pm$ 0.0011 & 42.02 $\pm$ 0.37 & 80.1 $\pm$ 0.4 \\
PostDeg-L-Adaptive & 56.92 $\pm$ 0.25 & 0.2435 $\pm$ 0.0007 & 0.7837 $\pm$ 0.0008 & 41.92 $\pm$ 0.37 & 80.3 $\pm$ 0.2 \\
\midrule
Extra LayerNorm & 54.97 $\pm$ 0.19 & 0.2371 $\pm$ 0.0009 & 0.7847 $\pm$ 0.0005 & 39.57 $\pm$ 0.50 & 80.3 $\pm$ 0.5 \\
GraphScalar & 55.11 $\pm$ 0.28 & 0.2351 $\pm$ 0.0006 & 0.7848 $\pm$ 0.0006 & 39.38 $\pm$ 0.53 & 80.7 $\pm$ 0.4 \\
LN backbone & 54.94 $\pm$ 0.31 & 0.2373 $\pm$ 0.0008 & 0.7849 $\pm$ 0.0006 & 39.60 $\pm$ 0.34 & 80.4 $\pm$ 0.4 \\
GraphNorm & 56.35 $\pm$ 0.14 & 0.2341 $\pm$ 0.0026 & 0.7877 $\pm$ 0.0021 & 41.29 $\pm$ 0.42 & 78.5 $\pm$ 0.7 \\
\midrule
PairNorm & 55.38 $\pm$ 0.25 & 0.2329 $\pm$ 0.0027 & 0.7894 $\pm$ 0.0021 & 40.40 $\pm$ 0.48 & 58.7 $\pm$ 0.0 \\
InstanceNorm & 56.35 $\pm$ 0.30 & 0.2288 $\pm$ 0.0039 & 0.7878 $\pm$ 0.0016 & 41.19 $\pm$ 0.48 & 58.7 $\pm$ 0.0 \\
BatchNorm & 55.50 $\pm$ 0.25 & 0.2339 $\pm$ 0.0014 & 0.7852 $\pm$ 0.0008 & 40.75 $\pm$ 0.45 & 58.8 $\pm$ 0.2 \\
\bottomrule
\end{tabular}%
}
\end{table}

\paragraph{Margins against the strongest baseline.} The body's $+3.5\%/+2.5\%/+5.6\%$ are relative to the LN backbone. Against the \emph{best} competing normalizer in Table~\ref{tab:main_results} the margins are narrower, about $0.9$ points on InfluMax (PostDeg $+3.50$ versus GraphNorm $+2.58$) and $1.3$ on MIS ($+5.58$ versus $+4.27$), and on Dismantle the separation is unambiguous because each feature-statistic normalizer falls below the backbone. These are largest-size margins, and the MIS comparison is size-dependent across the grid (Table~\ref{tab:mis-by-size}). We report paired Wilcoxon tests per comparison, which bottom out at $p=0.002$ at 10 seeds when each seed agrees.

\subsection{Reproducibility and setup}
\label{app:reproducibility}

\begin{algorithm}[!htbp]
\caption{PostDeg layer (wrapper around LayerNorm in a residual GNN block).}
\label{alg:postdeg}
\begin{algorithmic}[1]
\Require Node features $H \in \R^{n \times d}$, edge index $E$, normalized degree $\widehat c \in \R^n$, GAT block $\mathrm{GAT}_\ell$, stabilizer $\varepsilon = 10^{-8}$
\State $Z \gets H + \mathrm{GAT}_\ell(H, E)$ \Comment{residual GAT block}
\State $\bar H \gets \LN(Z)$ \Comment{per-node LayerNorm}
\State $\widetilde H_i \gets (\widehat c_i + \varepsilon)^{-1/2}\, \bar H_i$ for each node $i$ \Comment{post-LN scale}
\State \Return $\widetilde H$
\end{algorithmic}
\end{algorithm}

\begin{algorithm}[!htbp]
\small
\caption{Greedy node selection at evaluation. The decoder differs by task. On Dismantle and Epidemic the features are recomputed after each selection (the re-encode shown inside the loop), whereas on InfluMax and MIS the features are encoded once before the loop and each pick is a masked top-$K$ from the static scores.}
\label{alg:greedy}
\begin{algorithmic}[1]
\Require Graph $G=(V,E)$, policy $\pi_\theta$, budget $K$
\State $S \gets \emptyset$, $C \gets V$
\For{$k=1,\dots,K$}
  \State $H \gets \mathrm{GNN}_\theta(G, S)$ \Comment{Dismantle/Epidemic re-encode here. InfluMax/MIS compute $H$ once \emph{before} the loop and reuse it}
  \State $\hat v \gets \arg\max_{v\in C}\,\pi_\theta(H_v)$
  \State $S \gets S\cup\{\hat v\}$, $\quad C \gets C\setminus\{\hat v\}$
\EndFor
\State \Return $S$
\end{algorithmic}
\end{algorithm}

\paragraph{Code and data.}
The supplementary package contains the experiment code, configuration shell scripts, rendered tables and figures, and the result files used in the paper. Synthetic graph data are generated from the scripts, and DD is downloaded from the TU graph-kernel dataset distribution during preprocessing, used under its license terms and not redistributed as a modified dataset. The supplementary documentation maps each result file to the table or figure it produces. In brief, the original completed-run logs produce the main tables, while a separate set of supplemental runs produces the cross-backbone, NodeNorm, PNA, and supplemental DD tables.

\paragraph{Hyperparameters.}
\begin{table}[htbp]
\centering
\caption{Consolidated training hyperparameters. All methods share the same backbone, optimizer, training budget, seeds, and evaluation protocol. Hyperparameters were fixed from pilot runs and not tuned per method.}
\label{tab:hyperparams}
\footnotesize
\setlength{\tabcolsep}{6pt}
\resizebox{\textwidth}{!}{%
\begin{tabular}{ll}
\toprule
Hyperparameter & Value \\
\midrule
Backbone & 3-layer GAT \\
Hidden dim & 128 \\
Attention heads & 4 \\
Dropout & 0.1 \\
Optimizer & Adam \\
Learning rate & $10^{-4}$ (node-selection); $10^{-3}$ (DD) \\
Weight decay & $10^{-4}$ \\
Gradient clipping & 1.0 \\
Epochs & 200 (node-selection); 100 (DD) \\
Batch size & 64 graphs $\times$ 5 batches (InfluMax/MIS); 16 $\times$ 20 (Dismantle/Epidemic); 32 (DD) \\
Train sizes & $n=100$ (InfluMax); $n=150$ (Dismantle, Epidemic, MIS) \\
Eval sizes & $n\in\{50,100,150,200\}$ (InfluMax); $\{100,150,200,300\}$ (others) \\
Seeds (10) & $\{42, 123, 7, 2024, 314, 999, 555, 8080, 1337, 31337\}$ \\
$\varepsilon$ (PostDeg floor) & $10^{-8}$ \\
$\varepsilon'$ (PostDeg-L-Adaptive spectral floor) & $10^{-4}$ \\
$\varepsilon_{\rm LN}$ (LayerNorm) & $10^{-5}$ \\
PostDeg-L-Adaptive domain & $\alpha\in[0.1,3.0]$, $\beta\in[0,1.5]$, $\delta\in[0,0.5]$ \\
PostDeg-L-Adaptive clamps & spectral $[0.1,10]$, scale $[0.01,100]$, $c_i\in[10^{-6},1]$ \\
DD CV & stratified 10-fold with fixed shuffled folds per seed \\
\bottomrule
\end{tabular}
}
\end{table}

\paragraph{Forward pass.}
The PostDeg forward pass applies Eq.~\eqref{eq:postdeg-formula} once per layer. At layer $\ell$, $Z^{(\ell)}=H^{(\ell-1)}+\mathrm{GAT}_\ell(H^{(\ell-1)},A)$ and $H^{(\ell)}=(\widehat c+\varepsilon)^{-1/2}\odot\LN(Z^{(\ell)})$. The normalized degree vector $\widehat c$ is computed once and cached per graph. The learned same-position superset (Eq.~\eqref{eq:dsn-formula}) keeps the same placement and substitutes the parameterized scale, and we call these learned variants PostDeg-L-FG and PostDeg-L-Adaptive. Algorithm~\ref{alg:greedy} records the greedy node-selection procedure. The decoder differs by task. Dismantle and Epidemic recompute features after each selection, whereas InfluMax and MIS encode once and take a masked top-$K$ from the static scores. Within each task the decoder is identical across the compared methods.

\paragraph{Tasks and rewards.}
InfluMax uses Barab\'{a}si--Albert graphs with training size $n=100$, attachment parameter $m=3$, evaluation sizes $n\in\{50,100,150,200\}$, budget $K=\lfloor0.1n\rfloor$, and independent-cascade probability $p=0.1$ for $10$ rounds, averaged over $10$ Monte Carlo simulations per reward/evaluation. Network dismantling uses stochastic block model graphs with target training size $n=150$, $5$--$7$ communities, community-size jitter of $\pm20\%$, $p_{\mathrm{in}}=0.15$ with per-community uniform noise in $[-0.05,0.05]$ clipped to $[0.1,0.9]$, $p_{\mathrm{out}}=0.02$, evaluation sizes $n\in\{100,150,200,300\}$, and budget $K=\lfloor0.2n\rfloor$. Its reward is fragmentation $1-|\mathrm{LCC}|/n$ after removing the selected nodes, where $\mathrm{LCC}$ is the largest connected component. Epidemic containment uses the same SBM contact-graph generator and budget as Dismantle. Selected nodes are vaccinated before an SIR simulation with infection probability $\beta=0.3$, recovery probability $\gamma=0.1$, $50$ steps, $5$ initial infections, and $10$ Monte Carlo samples during training or $20$ during final evaluation. MIS uses SBM graphs with training size $n=150$, $5$--$7$ communities, $p_{\mathrm{in}}=0.3$, $p_{\mathrm{out}}=0.02$, evaluation sizes $n\in\{100,150,200,300\}$, and budget $K=\lfloor0.3n\rfloor$.

\paragraph{Excluded inputs.}
Raw node degree and eigenvector centrality are excluded from all input features.

\paragraph{Included structural features.}
InfluMax and MIS use clustering coefficient, normalized average-neighbor degree, selection indicator, and remaining-budget fraction. Dismantle and Epidemic use clustering coefficient, normalized $k$-core number, PageRank, and removal/vaccination indicator. On BA training graphs at $n=100$, PageRank, $k$-core, and average-neighbor degree have Pearson $|r|\ge 0.78$ with raw degree, so they are strong content surrogates. The PostDeg gain over the LN backbone holds even though these degree-correlated features are already in the input.

\paragraph{DD setup.}
DD graph classification uses the TU DD protein dataset with structural features (clustering, normalized $k$-core, PageRank, constant bias) concatenated with node labels/attributes when available, global mean pooling, and stratified $10$-fold cross-validation with fixed shuffled folds per seed.

\paragraph{Compute and runtime.}
\begin{table}[htbp]
\centering
\caption{Per-method runtime envelope, learnable parameters per layer, and asymptotic per-layer cost. The four post-LN scale operators add $\mathcal{O}(|V|)$ work per layer. PostDeg adds zero learnable parameters. Walltime per (method, task, seed) on a $2$-CPU $8$-GB worker.}
\label{tab:runtime}
\small
\setlength{\tabcolsep}{4pt}
\begin{tabular}{lccc}
\toprule
Method & Walltime & Learnable params/layer & Asymptotic cost/layer \\
\midrule
PostDeg & 30--35 min & 0 & $\mathcal{O}(|V|)$ \\
PostDeg-L-FG & 35--40 min & 4 & $\mathcal{O}(|V|)+\mathcal{O}(\lambda\text{ solver})$ \\
PostDeg-L-Adaptive & 40--50 min & 5 + MLP gate & $\mathcal{O}(|V|)+\mathcal{O}(\lambda\text{ solver})$ \\
GraphScalar & 30--35 min & 0 & $\mathcal{O}(\lambda\text{ solver})$ \\
LN backbone & 30 min & 0 & 0 \\
GraphNorm & 30--35 min & $d$ (gain/bias) & $\mathcal{O}(|V|d)$ \\
Extra LayerNorm & 30 min & $d$ (gain/bias) & $\mathcal{O}(|V|d)$ \\
BatchNorm & 30 min & $d$ (gain/bias) & $\mathcal{O}(|V|d)$ \\
InstanceNorm & 30 min & 0 & $\mathcal{O}(|V|d)$ \\
PairNorm & 35 min & 0 & $\mathcal{O}(|V|d)$ (Frobenius norm) \\
\bottomrule
\end{tabular}
\end{table}

\noindent
Each sweep job used one worker with $2$ CPU cores and $8$~GB RAM. Training one (method, task, seed) combination takes roughly $30$--$90$ minutes. The reported main grid contains $10$ methods $\times$ $10$ seeds $\times$ $5$ tasks $=500$ completed runs and took approximately $400$ worker-hours.

\paragraph{Numerical clamps.}
$\varepsilon=10^{-8}$ in Eq.~\eqref{eq:postdeg-formula}, $\varepsilon'=10^{-4}$ in the spectral term of Eq.~\eqref{eq:dsn-formula}, and $\varepsilon_{\rm LN}=10^{-5}$ in LayerNorm. PostDeg uses $\widehat c_i=\max(d_i,1)/\max_j\max(d_j,1)$. The learned variants clip $c_i$ to $[10^{-6},1]$. Spectral quantities are clamped to $[0.1,10.0]$. The learned-variant power-law factor $\alpha(c_i+\varepsilon)^{-\gamma}$ is clamped to $[0.01,100]$ before the graph-wise spectral multiplier, and this clamp saturates at isolated nodes (Lemma~\ref{lem:clamp-inactive-app}).

\paragraph{$\varepsilon$ ablation.}
Table~\ref{tab:eps-ablation} shows, by construction, that the $\varepsilon$ floor never binds. The $c_i$ clip at $10^{-6}$ dominates $\varepsilon$ for every $\varepsilon\le 10^{-6}$, so the operator output is independent of $\varepsilon$ in this range. This is an analytic argument, so a measured $\varepsilon$-sweep is unnecessary.

\begin{table}[htbp]
\centering
\caption{$\varepsilon$-floor ablation for PostDeg. Rows are the floor value, columns are evaluation sizes. The scores are identical across rows \emph{by construction}: the lower clamp on $c_i$ (set to $10^{-6}$) dominates $\varepsilon$ for every $\varepsilon\le10^{-6}$, so the floor never binds and the operator output does not depend on $\varepsilon$ in this range. This is an analytic argument, and we did not run a measured $\varepsilon$ sweep.}
\label{tab:eps-ablation}
\small
\setlength{\tabcolsep}{6pt}
\begin{tabular}{lcccc}
\toprule
$\varepsilon$ & InfluMax n50 & InfluMax n100 & InfluMax n150 & InfluMax n200 \\
\midrule
$10^{-10}$ & 13.30\,$\pm$\,0.08 & 27.79\,$\pm$\,0.13 & 42.25\,$\pm$\,0.19 & 56.86\,$\pm$\,0.20 \\
$10^{-8}$ & 13.30\,$\pm$\,0.08 & 27.79\,$\pm$\,0.13 & 42.25\,$\pm$\,0.19 & 56.86\,$\pm$\,0.20 \\
$10^{-6}$ & 13.30\,$\pm$\,0.08 & 27.79\,$\pm$\,0.13 & 42.25\,$\pm$\,0.19 & 56.86\,$\pm$\,0.20 \\
\bottomrule
\end{tabular}
\end{table}

\subsection{Released data}
\label{app:data-dict}

For each table and figure we release the underlying result files. These cover per-(backbone, task, mode, seed, size) evaluation rewards together with their heuristic comparators, per-(seed, epoch) learned-variant parameters for each layer, per-node scale factors with the associated degrees and learned coefficients, per-graph degree distributions and their summaries, and per-(seed, epoch) training and evaluation curves. The main tables are produced from the completed-run evaluation logs. The cross-backbone, PNA, NodeNorm, and supplemental DD tables (Tables~\ref{tab:cross-backbone}, \ref{tab:pna-redundancy}, \ref{tab:predicted-baselines-run}, \ref{tab:dd-extended}) are produced from a separate set of supplemental runs, seed-aggregated and, for DD, reported as 5-fold cross-validation, separate from the paper's main 10-fold DD runs. The supplementary README records how to recompute the tables from the released files.

\FloatBarrier
\subsection{Pre-LayerNorm multipliers are absorbed}\label{app:variance}

\paragraph{Absorption bound.}
The absorption identity (Proposition~\ref{prop:absorption-exact}) is exact in the $\varepsilon_{\rm LN}\to 0$ limit, and for a positive per-node multiplier $a_i$ the relative deviation between $\LN(a_i z_i)$ and $\LN(z_i)$ is at most $\varepsilon_{\rm LN}\lvert 1-a_i^{-2}\rvert/(2\sigma^2(z_i))$, so at most $\varepsilon_{\rm LN}/(2\sigma^2(z_i))$ for every $a_i\ge 1$, uniformly in $a_i$. The deviation saturates at this order for large $a_i$ instead of vanishing, because $\LN(a_i z_i)$ approaches the zero-stabilizer normalization while $\LN(z_i)$ keeps its stabilizer. With $\varepsilon_{\rm LN}=10^{-5}$ and any activation variance bounded away from zero (for a unit-scale representation $\sigma^2(z_i)=O(1)$), the bound is below $10^{-4}$. We did not export per-node $\sigma^2(z_i)$ from the training runs, so we state the bound analytically. The identity itself is standard for LayerNorm.

\subsection{Full numerical results}
\label{app:full-results}

Table~\ref{tab:full-transfer} reports mean $\pm$ seed std at each (method, task, evaluation size) combination. We pair it with two follow-on artifacts. The heatmap in Figure~\ref{fig:appendix-transfer-heatmap} compresses the rightmost column of the table into a single page-width grid so that the sign and magnitude of each (method, task) cell is readable at a glance, with the lower-is-better Epidemic column flipped. Per-seed reward at the largest evaluation size sits in Table~\ref{tab:per-seed-reward}, where each row is a single seed. The four per-task panels (Figures~\ref{fig:appendix-full-influmax}--\ref{fig:appendix-full-epidemic}) consolidate convergence, transfer, evaluation distribution, and per-seed comparison on a single page each.

\begin{table}[htbp]
\centering
\caption{Full transfer table: mean $\pm$ seed std over 10 seeds for each method, each node-selection task, each evaluation size. The horizontal rule separates the post-LN scale family from the baselines.}
\label{tab:full-transfer}
\small
\setlength{\tabcolsep}{2.6pt}
\medskip
\noindent\textbf{InfluMax} ($\uparrow$)\par
\begin{tabular}{lcccc}
\toprule
Method & $n{=}50$ & $n{=}100$ & $n{=}150$ & $n{=}200$ \\
\midrule
PostDeg & 13.30\,$\pm$\,0.08 & 27.79\,$\pm$\,0.13 & 42.25\,$\pm$\,0.19 & 56.86\,$\pm$\,0.20 \\
PostDeg-L-FG & 13.27\,$\pm$\,0.08 & 27.87\,$\pm$\,0.09 & 42.37\,$\pm$\,0.17 & 56.86\,$\pm$\,0.22 \\
PostDeg-L-Adaptive & 13.28\,$\pm$\,0.08 & 27.82\,$\pm$\,0.07 & 42.39\,$\pm$\,0.24 & 56.92\,$\pm$\,0.23 \\
Extra LayerNorm & 12.89\,$\pm$\,0.19 & 27.19\,$\pm$\,0.24 & 41.17\,$\pm$\,0.15 & 54.97\,$\pm$\,0.18 \\
\midrule
GraphScalar & 12.96\,$\pm$\,0.26 & 27.30\,$\pm$\,0.19 & 41.29\,$\pm$\,0.23 & 55.11\,$\pm$\,0.26 \\
LN backbone & 12.93\,$\pm$\,0.09 & 27.14\,$\pm$\,0.15 & 41.12\,$\pm$\,0.24 & 54.94\,$\pm$\,0.30 \\
GraphNorm & 13.23\,$\pm$\,0.10 & 27.64\,$\pm$\,0.12 & 41.94\,$\pm$\,0.14 & 56.35\,$\pm$\,0.13 \\
PairNorm & 13.13\,$\pm$\,0.07 & 27.30\,$\pm$\,0.14 & 41.41\,$\pm$\,0.15 & 55.38\,$\pm$\,0.24 \\
InstanceNorm & 13.17\,$\pm$\,0.07 & 27.57\,$\pm$\,0.11 & 42.01\,$\pm$\,0.23 & 56.35\,$\pm$\,0.29 \\
BatchNorm & 12.96\,$\pm$\,0.13 & 27.43\,$\pm$\,0.16 & 41.59\,$\pm$\,0.15 & 55.50\,$\pm$\,0.24 \\
\bottomrule
\end{tabular}\par
\medskip
\noindent\textbf{Dismantle} ($\uparrow$)\par
\begin{tabular}{lcccc}
\toprule
Method & $n{=}100$ & $n{=}150$ & $n{=}200$ & $n{=}300$ \\
\midrule
PostDeg & 0.3894\,$\pm$\,0.0028 & 0.3076\,$\pm$\,0.0014 & 0.2744\,$\pm$\,0.0011 & 0.2433\,$\pm$\,0.0010 \\
PostDeg-L-FG & 0.3878\,$\pm$\,0.0056 & 0.3075\,$\pm$\,0.0010 & 0.2745\,$\pm$\,0.0005 & 0.2436\,$\pm$\,0.0007 \\
PostDeg-L-Adaptive & 0.3887\,$\pm$\,0.0046 & 0.3079\,$\pm$\,0.0015 & 0.2745\,$\pm$\,0.0007 & 0.2435\,$\pm$\,0.0007 \\
Extra LayerNorm & 0.3612\,$\pm$\,0.0049 & 0.2959\,$\pm$\,0.0017 & 0.2658\,$\pm$\,0.0009 & 0.2371\,$\pm$\,0.0008 \\
\midrule
GraphScalar & 0.3528\,$\pm$\,0.0027 & 0.2899\,$\pm$\,0.0029 & 0.2631\,$\pm$\,0.0005 & 0.2351\,$\pm$\,0.0006 \\
LN backbone & 0.3661\,$\pm$\,0.0062 & 0.2980\,$\pm$\,0.0035 & 0.2658\,$\pm$\,0.0013 & 0.2373\,$\pm$\,0.0008 \\
GraphNorm & 0.3671\,$\pm$\,0.0031 & 0.3015\,$\pm$\,0.0015 & 0.2685\,$\pm$\,0.0020 & 0.2341\,$\pm$\,0.0024 \\
PairNorm & 0.3577\,$\pm$\,0.0057 & 0.2924\,$\pm$\,0.0022 & 0.2637\,$\pm$\,0.0011 & 0.2329\,$\pm$\,0.0026 \\
InstanceNorm & 0.3710\,$\pm$\,0.0038 & 0.3011\,$\pm$\,0.0016 & 0.2667\,$\pm$\,0.0018 & 0.2288\,$\pm$\,0.0037 \\
BatchNorm & 0.3559\,$\pm$\,0.0041 & 0.2877\,$\pm$\,0.0025 & 0.2607\,$\pm$\,0.0018 & 0.2339\,$\pm$\,0.0014 \\
\bottomrule
\end{tabular}\par
\medskip
\noindent\textbf{Epidemic} ($\downarrow$)\par
\begin{tabular}{lcccc}
\toprule
Method & $n{=}100$ & $n{=}150$ & $n{=}200$ & $n{=}300$ \\
\midrule
PostDeg & 0.6269\,$\pm$\,0.0038 & 0.7167\,$\pm$\,0.0021 & 0.7467\,$\pm$\,0.0014 & 0.7843\,$\pm$\,0.0009 \\
PostDeg-L-FG & 0.6292\,$\pm$\,0.0038 & 0.7179\,$\pm$\,0.0019 & 0.7470\,$\pm$\,0.0017 & 0.7835\,$\pm$\,0.0010 \\
PostDeg-L-Adaptive & 0.6319\,$\pm$\,0.0118 & 0.7171\,$\pm$\,0.0026 & 0.7464\,$\pm$\,0.0014 & 0.7837\,$\pm$\,0.0008 \\
Extra LayerNorm & 0.6477\,$\pm$\,0.0032 & 0.7215\,$\pm$\,0.0019 & 0.7499\,$\pm$\,0.0015 & 0.7847\,$\pm$\,0.0005 \\
\midrule
GraphScalar & 0.6485\,$\pm$\,0.0033 & 0.7218\,$\pm$\,0.0024 & 0.7491\,$\pm$\,0.0014 & 0.7848\,$\pm$\,0.0005 \\
LN backbone & 0.6509\,$\pm$\,0.0037 & 0.7221\,$\pm$\,0.0011 & 0.7499\,$\pm$\,0.0009 & 0.7849\,$\pm$\,0.0006 \\
GraphNorm & 0.6521\,$\pm$\,0.0067 & 0.7202\,$\pm$\,0.0014 & 0.7496\,$\pm$\,0.0025 & 0.7877\,$\pm$\,0.0020 \\
PairNorm & 0.6537\,$\pm$\,0.0038 & 0.7224\,$\pm$\,0.0016 & 0.7521\,$\pm$\,0.0024 & 0.7894\,$\pm$\,0.0020 \\
InstanceNorm & 0.6532\,$\pm$\,0.0045 & 0.7197\,$\pm$\,0.0017 & 0.7497\,$\pm$\,0.0014 & 0.7878\,$\pm$\,0.0015 \\
BatchNorm & 0.6470\,$\pm$\,0.0040 & 0.7246\,$\pm$\,0.0013 & 0.7514\,$\pm$\,0.0018 & 0.7852\,$\pm$\,0.0008 \\
\bottomrule
\end{tabular}\par
\medskip
\noindent\textbf{MIS} ($\uparrow$)\par
\begin{tabular}{lcccc}
\toprule
Method & $n{=}100$ & $n{=}150$ & $n{=}200$ & $n{=}300$ \\
\midrule
PostDeg & 25.96\,$\pm$\,0.11 & 32.96\,$\pm$\,0.22 & 37.64\,$\pm$\,0.28 & 41.81\,$\pm$\,0.34 \\
PostDeg-L-FG & 26.09\,$\pm$\,0.15 & 33.19\,$\pm$\,0.19 & 37.88\,$\pm$\,0.19 & 42.02\,$\pm$\,0.35 \\
PostDeg-L-Adaptive & 26.15\,$\pm$\,0.16 & 33.11\,$\pm$\,0.26 & 37.77\,$\pm$\,0.29 & 41.92\,$\pm$\,0.35 \\
Extra LayerNorm & 24.57\,$\pm$\,0.09 & 32.19\,$\pm$\,0.30 & 37.29\,$\pm$\,0.38 & 39.57\,$\pm$\,0.47 \\
\midrule
GraphScalar & 24.52\,$\pm$\,0.12 & 32.04\,$\pm$\,0.32 & 37.04\,$\pm$\,0.52 & 39.38\,$\pm$\,0.51 \\
LN backbone & 24.57\,$\pm$\,0.22 & 32.27\,$\pm$\,0.34 & 37.37\,$\pm$\,0.51 & 39.60\,$\pm$\,0.32 \\
GraphNorm & 25.13\,$\pm$\,0.26 & 33.82\,$\pm$\,0.34 & 39.16\,$\pm$\,0.24 & 41.29\,$\pm$\,0.40 \\
PairNorm & 24.70\,$\pm$\,0.18 & 32.72\,$\pm$\,0.24 & 38.04\,$\pm$\,0.47 & 40.40\,$\pm$\,0.46 \\
InstanceNorm & 25.02\,$\pm$\,0.27 & 33.67\,$\pm$\,0.35 & 39.04\,$\pm$\,0.26 & 41.19\,$\pm$\,0.45 \\
BatchNorm & 25.03\,$\pm$\,0.26 & 33.47\,$\pm$\,0.33 & 38.75\,$\pm$\,0.21 & 40.74\,$\pm$\,0.43 \\
\bottomrule
\end{tabular}\par
\end{table}

The heatmap below is the same data as Table~\ref{tab:full-transfer} at the largest evaluation size, with sign flipped for the lower-is-better Epidemic column. The DD column is also included so that the collapse of PairNorm/InstanceNorm/BatchNorm is visible against the same color scale.

\begin{figure}[htbp]
  \centering
  \includegraphics[width=\linewidth]{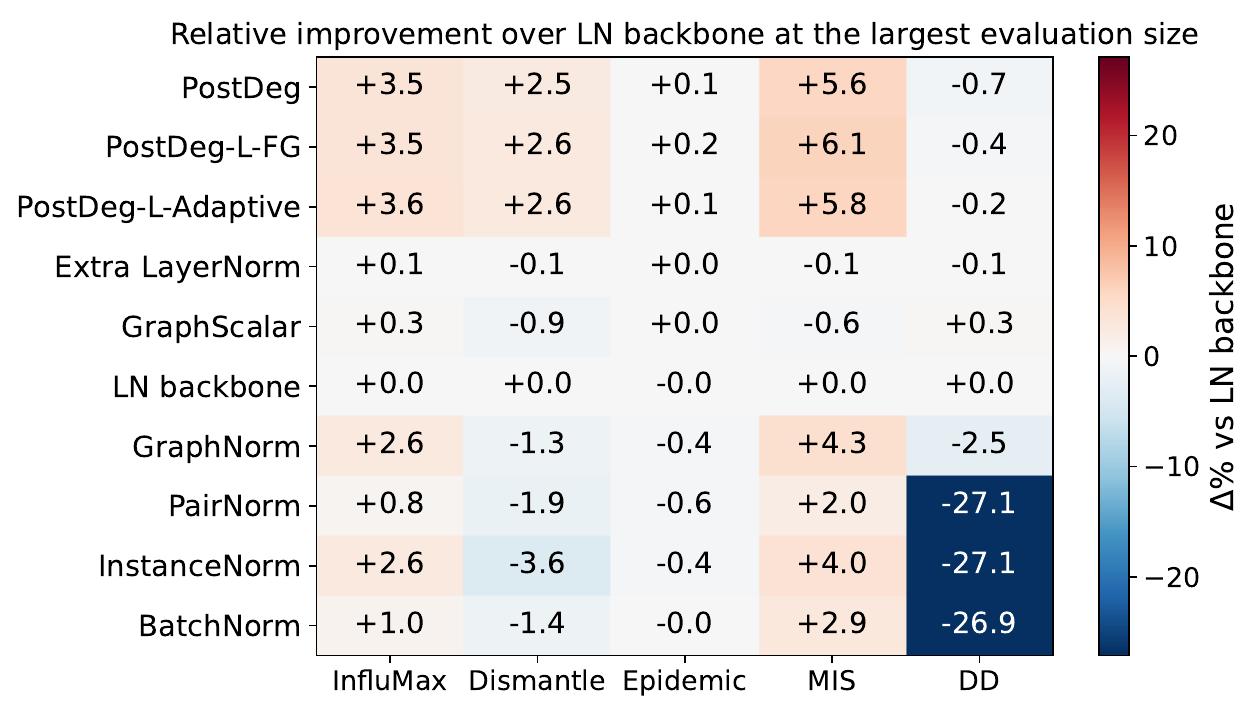}
  \caption{Heatmap version of Table~\ref{tab:full-transfer} at the largest evaluation size, with relative improvement over LN backbone per (method, task), with sign flipped for lower-is-better tasks. PostDeg and its learned variants form the only column with consistently positive entries. PairNorm, InstanceNorm, and BatchNorm are positive on three node-selection tasks and drop to the majority-class rate on DD.}
  \label{fig:appendix-transfer-heatmap}
\end{figure}

Three patterns are visible in the heatmap. The top three rows (PostDeg, PostDeg-L) carry positive entries on InfluMax, Dismantle, and MIS, near-zero entries on Epidemic, and small negative entries on DD. This is the cross-task signature of a placement effect that attenuates as the degree distribution becomes more regular. The middle block (LayerNorm, GraphScalar, LN backbone) remains near zero on all columns, because none of these operators introduces a per-node degree contrast in the slot. The bottom block (PairNorm, InstanceNorm, BatchNorm) has the most variable behavior. These three feature-statistic normalizers are positive on InfluMax and MIS, mixed on Dismantle and Epidemic, and drop to the majority-class rate on DD, which is the pooling-head artifact described in Section~\ref{sec:results-boundary}.

The DD column is also the only place where PostDeg and its learned variants deviate from LN backbone by a perceptible amount, namely $-0.7\%$ for PostDeg, $-0.4\%$ for PostDeg-L-FG, and $-0.2\%$ for PostDeg-L-Adaptive. These small negatives sit within seed std, in line with the reading in \S\ref{sec:results-boundary}. DD is a graph-classification task whose head reads a mean-pooled graph embedding, so a per-node magnitude scale has no selection target there, and accuracy stays at the backbone level.

The per-seed reward table that follows lists the per-seed evidence behind the 10/10 win counts. It has ten rows per task, one for each seed, with PostDeg and its learned variants on the left and the LN backbone baseline on the right. Reading any row, the gap between PostDeg and LN backbone matches the per-task summary in Section~\ref{sec:results-main}.

\begin{table}[htbp]
\centering
\caption{Per-seed reward for the post-LN family and the unnormalized baseline at the largest evaluation size. Each row is a single seed. The per-task 10/10 win counts can be read directly off these rows.}
\label{tab:per-seed-reward}
\small
\setlength{\tabcolsep}{4pt}
\medskip
\noindent\textbf{InfluMax}\par
\begin{tabular}{lcccc}
\toprule
Seed & PostDeg & PostDeg-L-FG & PostDeg-L-Adaptive & LN backbone \\
\midrule
7 & 56.584 & 56.640 & 56.974 & 55.046 \\
42 & 56.967 & 56.464 & 56.764 & 54.649 \\
123 & 57.106 & 56.764 & 56.798 & 55.028 \\
314 & 56.827 & 56.850 & 57.016 & 54.481 \\
555 & 57.066 & 56.894 & 56.682 & 55.366 \\
999 & 56.779 & 57.110 & 56.572 & 55.020 \\
1337 & 56.586 & 56.651 & 57.084 & 54.967 \\
2024 & 56.997 & 57.152 & 56.842 & 55.027 \\
8080 & 56.629 & 56.932 & 57.444 & 54.473 \\
31337 & 57.051 & 57.112 & 57.023 & 55.319 \\
\bottomrule
\end{tabular}\par
\medskip
\noindent\textbf{Dismantle}\par
\begin{tabular}{lcccc}
\toprule
Seed & PostDeg & PostDeg-L-FG & PostDeg-L-Adaptive & LN backbone \\
\midrule
7 & 0.2422 & 0.2432 & 0.2427 & 0.2371 \\
42 & 0.2429 & 0.2426 & 0.2422 & 0.2362 \\
123 & 0.2443 & 0.2438 & 0.2431 & 0.2372 \\
314 & 0.2413 & 0.2446 & 0.2446 & 0.2385 \\
555 & 0.2446 & 0.2448 & 0.2436 & 0.2379 \\
999 & 0.2427 & 0.2431 & 0.2439 & 0.2363 \\
1337 & 0.2435 & 0.2434 & 0.2429 & 0.2372 \\
2024 & 0.2431 & 0.2427 & 0.2441 & 0.2386 \\
8080 & 0.2442 & 0.2438 & 0.2440 & 0.2368 \\
31337 & 0.2440 & 0.2437 & 0.2437 & 0.2371 \\
\bottomrule
\end{tabular}\par
\medskip
\noindent\textbf{Epidemic}\par
\begin{tabular}{lcccc}
\toprule
Seed & PostDeg & PostDeg-L-FG & PostDeg-L-Adaptive & LN backbone \\
\midrule
7 & 0.7847 & 0.7830 & 0.7832 & 0.7847 \\
42 & 0.7827 & 0.7852 & 0.7841 & 0.7843 \\
123 & 0.7841 & 0.7842 & 0.7829 & 0.7846 \\
314 & 0.7857 & 0.7851 & 0.7845 & 0.7853 \\
555 & 0.7834 & 0.7826 & 0.7853 & 0.7846 \\
999 & 0.7846 & 0.7835 & 0.7832 & 0.7847 \\
1337 & 0.7845 & 0.7832 & 0.7829 & 0.7845 \\
2024 & 0.7835 & 0.7822 & 0.7832 & 0.7846 \\
8080 & 0.7840 & 0.7824 & 0.7846 & 0.7850 \\
31337 & 0.7854 & 0.7837 & 0.7832 & 0.7864 \\
\bottomrule
\end{tabular}\par
\medskip
\noindent\textbf{MIS}\par
\begin{tabular}{lcccc}
\toprule
Seed & PostDeg & PostDeg-L-FG & PostDeg-L-Adaptive & LN backbone \\
\midrule
7 & 41.450 & 41.875 & 41.730 & 39.390 \\
42 & 42.015 & 41.955 & 41.800 & 40.295 \\
123 & 41.580 & 41.785 & 42.105 & 39.340 \\
314 & 42.545 & 42.605 & 42.105 & 39.820 \\
555 & 42.110 & 42.625 & 42.530 & 39.480 \\
999 & 41.560 & 41.715 & 41.320 & 39.300 \\
1337 & 41.870 & 42.010 & 42.030 & 39.500 \\
2024 & 41.515 & 41.685 & 41.560 & 39.230 \\
8080 & 41.985 & 42.300 & 42.340 & 39.915 \\
31337 & 41.510 & 41.595 & 41.680 & 39.765 \\
\bottomrule
\end{tabular}\par
\end{table}

\paragraph{Per-task panels.}
The four figures that follow give one diagnostic per task. Each combines, as panels (a)--(d), the convergence curve of each method, relative improvement over the LN backbone across evaluation sizes, the score distribution at the largest evaluation size, and a paired per-seed scatter of PostDeg against the LN backbone. The order is InfluMax, Dismantle, MIS, Epidemic. The gap between PostDeg and the LN backbone shrinks as the degree distribution becomes more regular.

InfluMax is the most heterogeneous task we run. In Figure~\ref{fig:appendix-full-influmax}, PostDeg and PostDeg-L separate from all baselines within the first $50$ training epochs, keep the lead across all four evaluation sizes, and produce a per-seed scatter that sits entirely above the diagonal.

\begin{figure}[htbp]
  \centering
  \includegraphics[width=\linewidth]{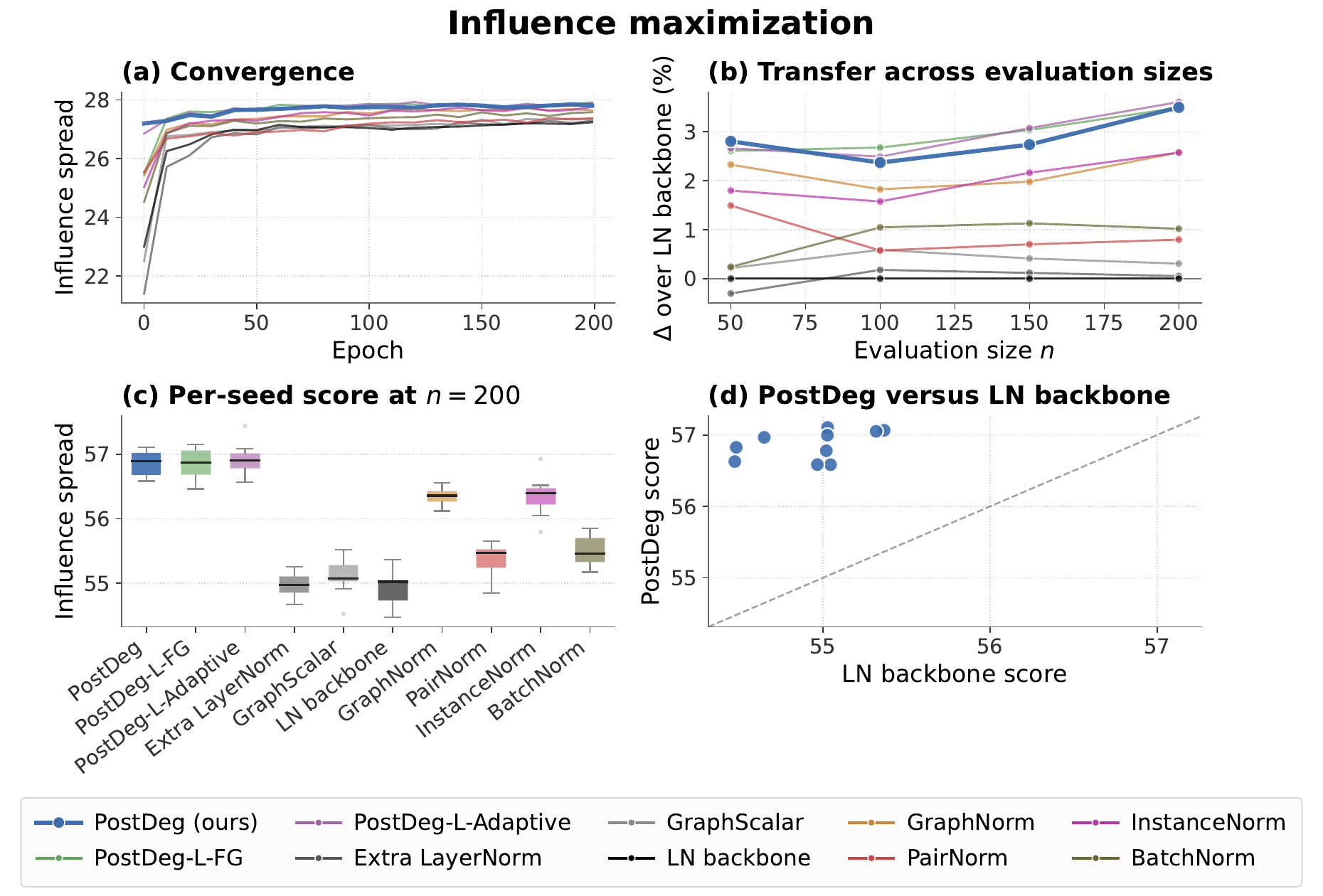}
  \caption{InfluMax per-task diagnostic (mean $\pm$ seed std, 10 seeds). \textbf{(a)}~per-method convergence. \textbf{(b)}~relative improvement over the LN backbone across evaluation sizes. \textbf{(c)}~per-seed score distribution at $n=200$. \textbf{(d)}~paired per-seed scores, PostDeg against the LN backbone (dashed identity line).}
  \label{fig:appendix-full-influmax}
\end{figure}

Figure~\ref{fig:appendix-full-dismantle} shows the same separation pattern for network dismantling, and it adds one feature. PairNorm, InstanceNorm, and BatchNorm fall \emph{below} LN backbone at the largest evaluation size, visible as the negative-slope lines in panel (b). PostDeg and its learned variants stay above zero on each evaluation size.

\begin{figure}[htbp]
  \centering
  \includegraphics[width=\linewidth]{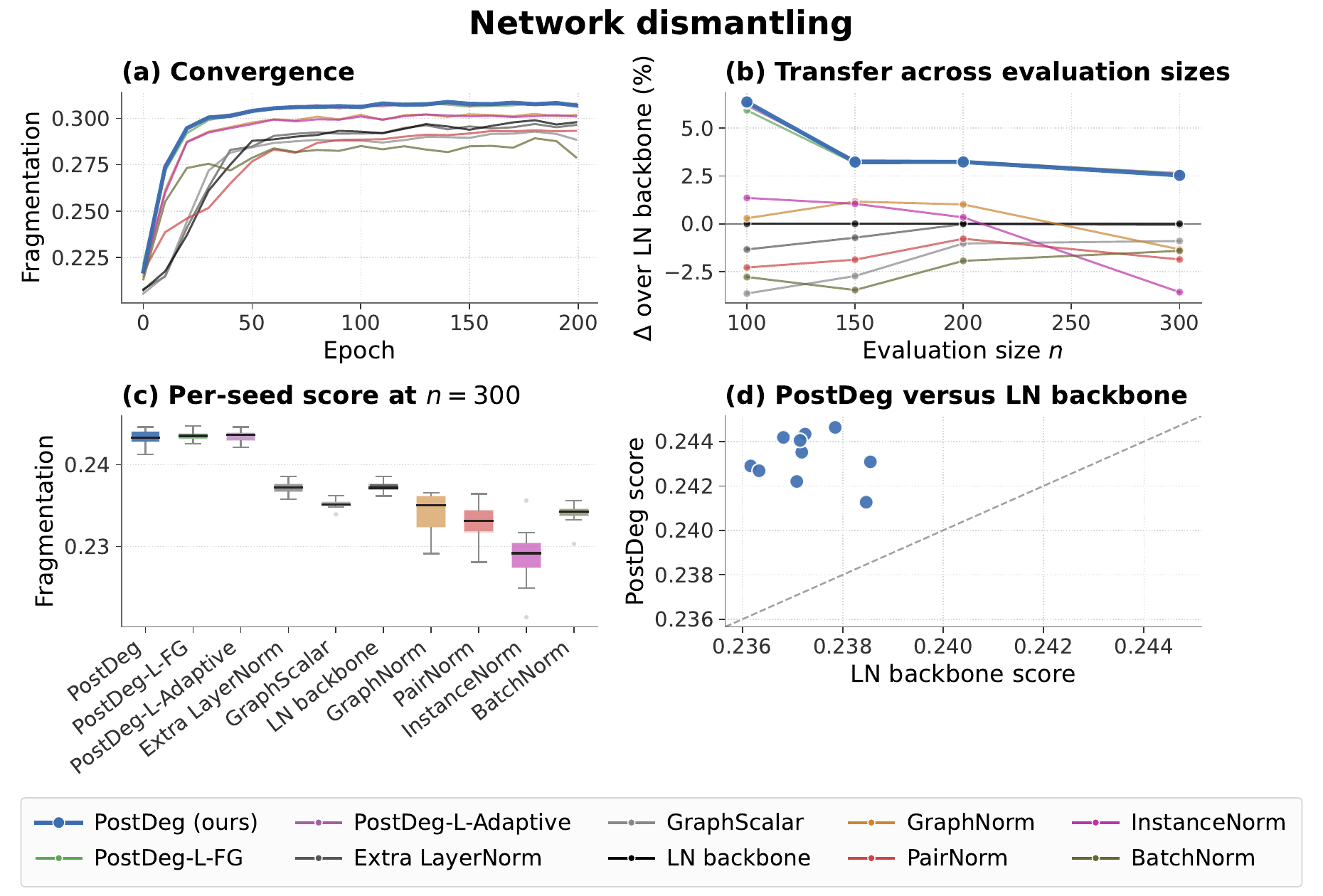}
  \caption{Network dismantling per-task diagnostic (mean $\pm$ seed std, 10 seeds). \textbf{(a)}~per-method convergence. \textbf{(b)}~relative improvement over the LN backbone across evaluation sizes, where PairNorm, InstanceNorm, and BatchNorm fall below zero. \textbf{(c)}~per-seed score distribution at $n=300$. \textbf{(d)}~paired per-seed scores, PostDeg against the LN backbone (dashed identity line).}
  \label{fig:appendix-full-dismantle}
\end{figure}

MIS trains on SBM graphs only. The BA family enters solely as cross-distribution evaluation (Appendix~\ref{app:boundary}). Figure~\ref{fig:appendix-full-mis} shows the largest absolute reward gap of any node-selection task, about $2.2$ at $n=300$ (an in-distribution SBM size). GraphNorm and InstanceNorm capture about $70\%$ of this gap. The rest of the baselines stay close to LN backbone.

\begin{figure}[htbp]
  \centering
  \includegraphics[width=\linewidth]{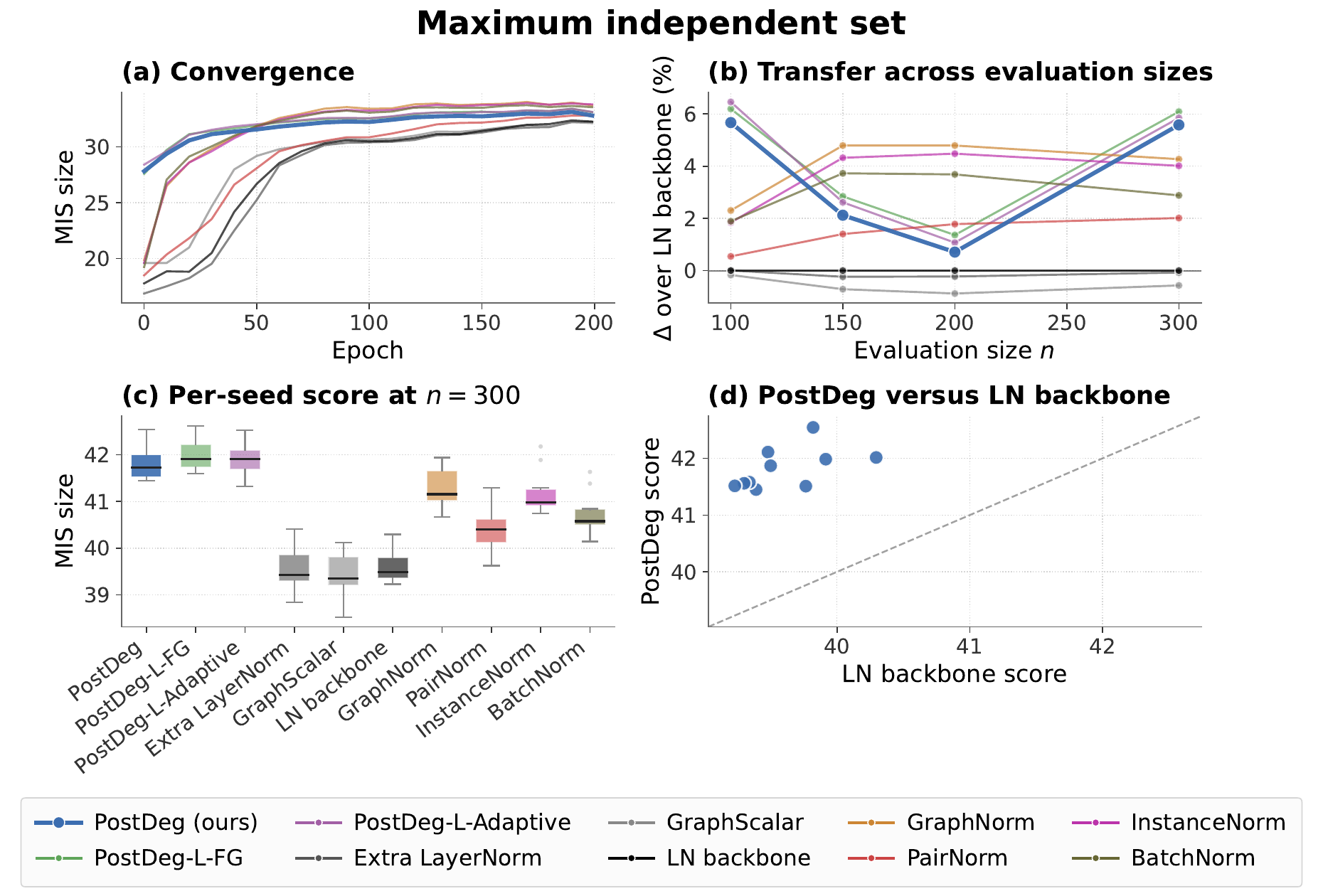}
  \caption{Maximum independent set per-task diagnostic (mean $\pm$ seed std, 10 seeds). \textbf{(a)}~per-method convergence. \textbf{(b)}~relative improvement over the LN backbone across evaluation sizes. \textbf{(c)}~per-seed score distribution at $n=300$. \textbf{(d)}~paired per-seed scores, PostDeg against the LN backbone (dashed identity line).}
  \label{fig:appendix-full-mis}
\end{figure}

Epidemic sits at the low-heterogeneity end of our tasks, and at the largest evaluation size the methods cluster together. Figure~\ref{fig:appendix-full-epidemic} shows each method within a $0.005$-wide band on infected fraction, and the per-seed scatter of PostDeg against the LN backbone straddles the diagonal. When the degree distribution is close to uniform, the post-LN scale becomes a near-constant graphwise multiplier and the placement effect attenuates.

\begin{figure}[htbp]
  \centering
  \includegraphics[width=\linewidth]{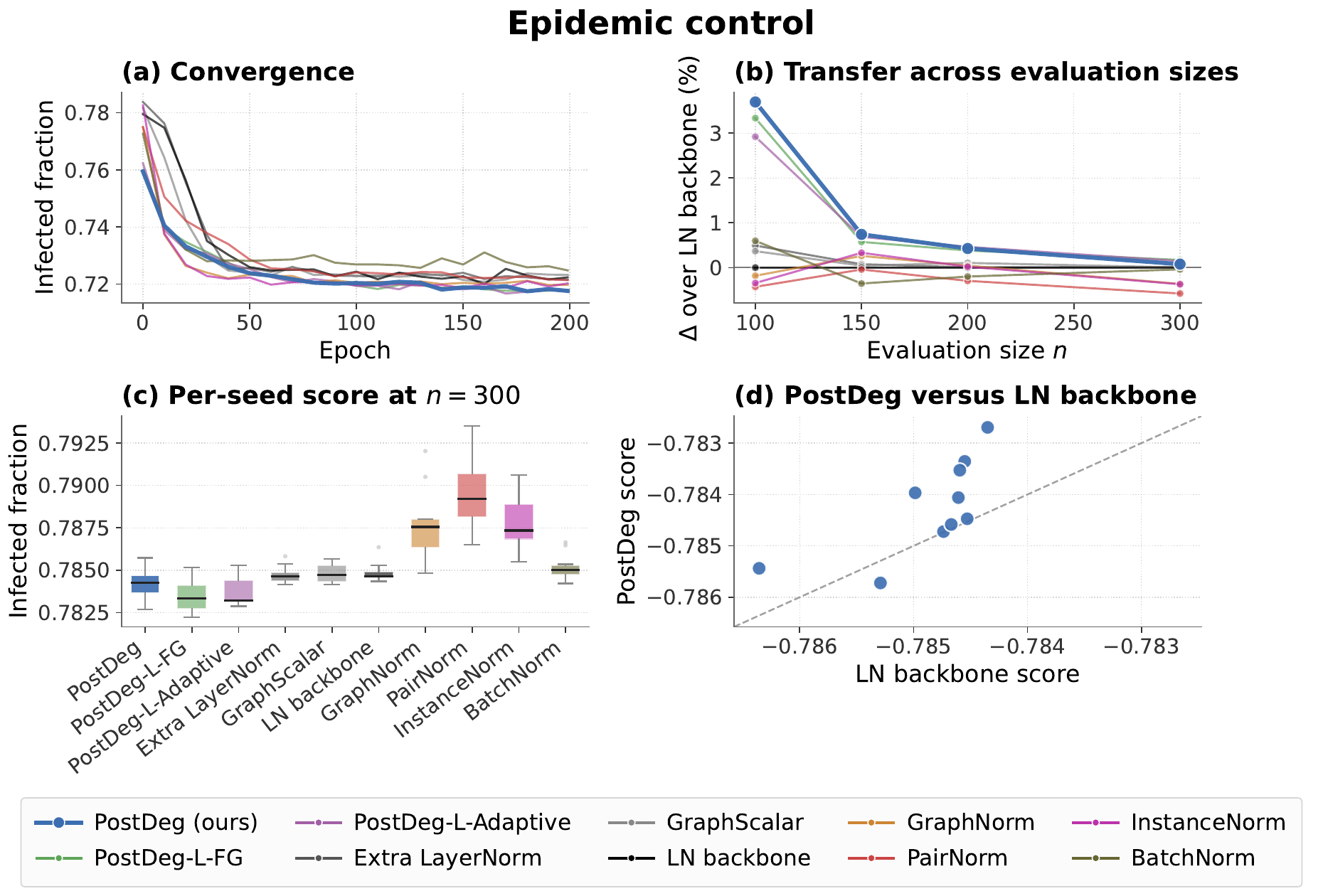}
  \caption{Epidemic containment per-task diagnostic (mean $\pm$ seed std, 10 seeds). \textbf{(a)}~per-method convergence. \textbf{(b)}~relative improvement over the LN backbone across evaluation sizes. \textbf{(c)}~per-seed score distribution at $n=300$. \textbf{(d)}~paired per-seed scores, PostDeg against the LN backbone (dashed identity line), which straddles the diagonal.}
  \label{fig:appendix-full-epidemic}
\end{figure}

\FloatBarrier
\subsection{Equivalence and ablations}
\label{app:stats}

\paragraph{Placement-alone check.} If placement alone drove the gain, post-LN operators of different capacity would be paired-equivalent. We test this by TOST, both pooled and stratified by graph family, and find that it holds on near-regular graphs but fails on heavy-tailed ones.

\paragraph{Ablation hierarchy.}
Figure~\ref{fig:appendix-ablation-hierarchy} arranges the learned superset and its ablations by parameter count. In descending capacity, the order is PostDeg-L-Adaptive, PostDeg-L-FG, PostDeg, GraphScalar, LN backbone. At the pooled level the first three cluster within seed noise, while GraphScalar and the LN backbone sit at the bottom. This clustering separates once graphs are stratified. On heavy-tailed BA graphs the learned exponent moves the numbers above parameter-free PostDeg (Table~\ref{tab:tost-family}).

\begin{figure}[htbp]
  \centering
  \includegraphics[width=\linewidth]{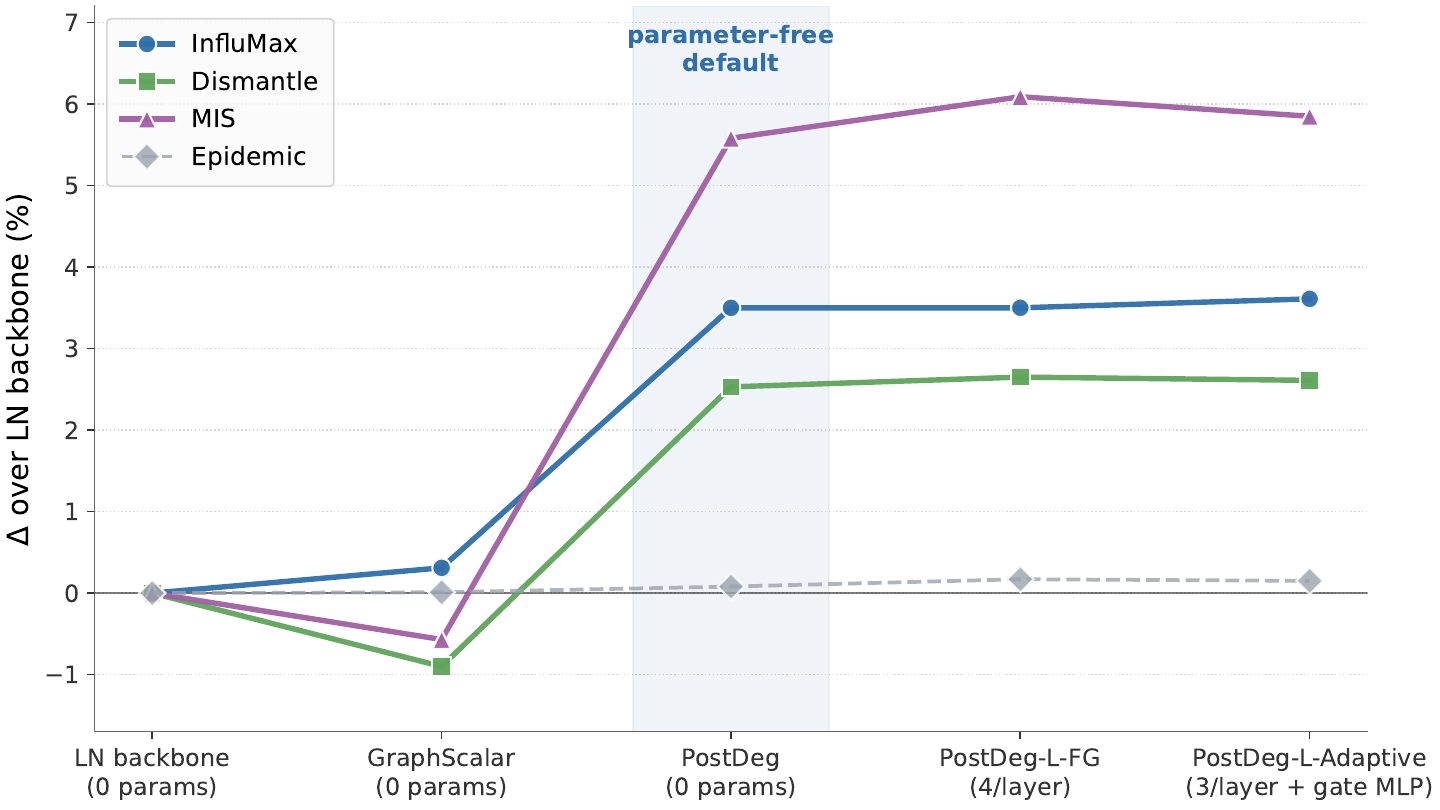}
  \caption{Per-task $\Delta\%$ over the LN backbone at the largest evaluation size across the ablation hierarchy, with operator capacity annotated under each method. Parameter-free PostDeg matches the higher-capacity learned variants, a plateau to its right, while GraphScalar, which drops the per-node degree term, matches the backbone. Epidemic (dashed) is near zero at this largest evaluation size, though at the smallest evaluation size it reaches $+3.7\%$ (Table~\ref{tab:size-profile}).}
  \label{fig:appendix-ablation-hierarchy}
\end{figure}

\paragraph{TOST equivalence tests, pooled and family-stratified.}
On the pooled/SBM evaluations, all four tasks produce a TOST $p$-value below $0.05$ at the $\pm 1\%$ margin (Table~\ref{tab:tost}), so PostDeg and the learned superset are paired-equivalent there. This pooled equivalence does \emph{not} hold once MIS is stratified by graph family (Table~\ref{tab:tost-family}). On heavy-tailed BA graphs the learned exponent beats parameter-free PostDeg by more than twice the margin, unanimously across seeds, while equivalence holds only on the near-regular SBM family. The pooled TOST therefore shows only that the \emph{form} stops mattering on near-regular graphs. It does not show that the form never matters. On the BA family the learned exponent does beat PostDeg, but that family is cross-distribution for MIS and the learned variant also carries weight spectral normalization, so we read the BA margin as confounded by that shift and the added weight normalization.

\begin{table}[htbp]
\centering
\caption{Two one-sided tests (TOST) for paired equivalence between PostDeg-L-Adaptive (the learned PostDeg ablation) and PostDeg, one row per task. Margins are taken as $\pm 1\%$, $\pm 2\%$, $\pm 5\%$ of the PostDeg mean. $p_{\text{TOST}}$ values below $0.05$ reject inequivalence.}
\label{tab:tost}
\small
\setlength{\tabcolsep}{6pt}
\begin{tabular}{lccccc}
\toprule
Task & paired mean diff & 90\% CI & $p_{\text{TOST}}$ ($\pm 1\%$) & $p_{\text{TOST}}$ ($\pm 2\%$) & $p_{\text{TOST}}$ ($\pm 5\%$) \\
\midrule
InfluMax & +0.06080 & [-0.16901, +0.29061] & 0.001 & 0.000 & 0.000 \\
Dismantle & +0.00020 & [-0.00060, +0.00099] & 0.000 & 0.000 & 0.000 \\
Epidemic & -0.00054 & [-0.00136, +0.00028] & 0.000 & 0.000 & 0.000 \\
MIS & +0.10600 & [-0.07711, +0.28911] & 0.006 & 0.000 & 0.000 \\
\bottomrule
\end{tabular}
\end{table}

\begin{table}[htbp]
\centering
\caption{Family-stratified equivalence on MIS (parameter-free PostDeg versus\ the learned exponent PostDeg-L-FG), recomputed per seed from the completed-run evaluation logs. On heavy-tailed BA graphs the learned exponent beats PostDeg by more than twice the $\pm1\%$ equivalence margin on 10 of 10 seeds. Equivalence holds only on near-regular SBM. $p=0.002$ is the two-sided $n{=}10$ Wilcoxon floor, attained by unanimous seed agreement. MIS trains on SBM, so the BA rows are cross-distribution (out-of-distribution) evaluation. The margin is also confounded by the weight spectral normalization the learned variant carries and PostDeg does not (\S\ref{sec:variants}), so the exponent alone does not account for it.}
\label{tab:tost-family}
\small
\begin{tabular}{lccccc}
\toprule
MIS eval & PostDeg & PostDeg-L-FG & rel.\ diff & wins & Wilcoxon $p$ \\
\midrule
BA $n{=}100$ & 25.80 & 26.36 & $+2.18\%$ & 10/10 & $0.002$ \\
BA $n{=}150$ & 39.34 & 40.30 & $+2.45\%$ & 10/10 & $0.002$ \\
BA $n{=}200$ & 53.03 & 54.33 & $+2.45\%$ & 10/10 & $0.002$ \\
SBM $n{=}300$ & 41.81 & 42.02 & $+0.48\%$ & \phantom{0}9/10 & $0.006$ \\
\bottomrule
\end{tabular}
\end{table}

\paragraph{Wilcoxon signed-rank tests (PostDeg-L-Adaptive versus\ representative baselines).}
The Wilcoxon table is the complementary one-sided test. It tests whether PostDeg-L-Adaptive and each baseline are statistically distinguishable, paired by seed. The PostDeg-L-Adaptive row against PostDeg is non-significant on each task, which is consistent with the equivalence result above. The comparisons against LN backbone and GraphScalar are significant across all tasks, which confirms that PostDeg and its learned variants improve over the LN-backbone and graphwise-scalar baselines.

\begin{table}[htbp]
\centering
\caption{Wilcoxon signed-rank test: PostDeg-L-Adaptive versus\ baselines at largest evaluation size, paired by seed. The PostDeg-L-Adaptive row against PostDegs are not significant at $\alpha=0.05$. This is a non-rejection at $n=10$ and does not establish equivalence (see Table~\ref{tab:tost} for TOST equivalence tests).}
\label{tab:significance}
\small
\setlength{\tabcolsep}{8pt}
\begin{tabular}{llcc}
\toprule
Experiment & Baseline & $p$-value & Significance \\
\midrule
InfluMax & LN backbone & 0.0020 & $^{**}$ \\
 & PostDeg & 0.8457 & n.s. \\
 & GraphScalar & 0.0020 & $^{**}$ \\
 & BatchNorm & 0.0020 & $^{**}$ \\
\midrule
Dismantle & LN backbone & 0.0020 & $^{**}$ \\
 & PostDeg & 0.9219 & n.s. \\
 & GraphScalar & 0.0020 & $^{**}$ \\
 & BatchNorm & 0.0020 & $^{**}$ \\
\midrule
Epidemic & LN backbone & 0.0098 & $^{**}$ \\
 & PostDeg & 0.3223 & n.s. \\
 & GraphScalar & 0.0039 & $^{**}$ \\
 & BatchNorm & 0.0098 & $^{**}$ \\
\midrule
MIS & LN backbone & 0.0020 & $^{**}$ \\
 & PostDeg & 0.3750 & n.s. \\
 & GraphScalar & 0.0020 & $^{**}$ \\
 & BatchNorm & 0.0020 & $^{**}$ \\
\bottomrule
\multicolumn{4}{l}{\small\textit{Note:} $^{***}$~$p{<}0.001$, $^{**}$~$p{<}0.01$, $^{*}$~$p{<}0.05$, n.s.\ = not significant.}
\end{tabular}
\end{table}

\paragraph{Effect sizes and per-seed wins.}
Cohen's $d$ values relative to the LN backbone are in Table~\ref{tab:cohens-d}. They are large on three tasks because between-seed variance is low at our seed budget, so we also report the per-seed reward gaps in Table~\ref{tab:per-seed-reward} as a scale-free complement.

\begin{table}[htbp]
\centering
\caption{Cohen's $d$ against LN backbone at the largest evaluation size. Positive values favor the method (sign-flipped for lower-is-better tasks). These $d$ values come from the near-deterministic training pipeline, with between-seed variability below $0.002$ on the learned exponent (Table~\ref{tab:per-seed-beta}), and do not indicate a large practical effect. We read them alongside the raw per-seed rewards in Table~\ref{tab:per-seed-reward}.}
\label{tab:cohens-d}
\small
\setlength{\tabcolsep}{6pt}
\begin{tabular}{lcccc}
\toprule
Method & InfluMax & Dismantle & Epidemic & MIS \\
\midrule
PostDeg & +7.28 & +6.38 & +0.79 & +6.40 \\
PostDeg-L-FG & +7.01 & +8.22 & +1.58 & +6.79 \\
PostDeg-L-Adaptive & +7.06 & +8.01 & +1.59 & +6.56 \\
Extra LayerNorm & +0.11 & -0.20 & +0.21 & -0.07 \\
GraphScalar & +0.57 & -3.01 & +0.07 & -0.50 \\
GraphNorm & +5.88 & -1.67 & -1.81 & +4.43 \\
PairNorm & +1.54 & -2.21 & -2.90 & +1.92 \\
InstanceNorm & +4.60 & -3.03 & -2.39 & +3.86 \\
BatchNorm & +1.98 & -2.88 & -0.42 & +2.88 \\
\bottomrule
\end{tabular}
\end{table}

\FloatBarrier

On each of the three degree-sensitive node-selection tasks, PostDeg and its learned variants each win on $10$ of $10$ paired seeds against the LN backbone, and $9$ of $10$ on Epidemic (visible per seed in Table~\ref{tab:per-seed-reward}).

\paragraph{Result.} TOST rejects inequivalence between PostDeg and the learned superset at $\pm 1\%$ on all tasks. The FG and adaptive variants settle at nearly identical parameters (Table~\ref{tab:dsn_params}).

\FloatBarrier
\subsection{Post-LayerNorm multipliers reach the score head}
\label{app:mechanism}

\paragraph{Theory check.} Theorem~\ref{thm:degree-separation-app} yields a strict scale ordering and Corollary~\ref{cor:loglog-app} a log-log slope of $-\gamma$, and we test both per task and per layer.

\paragraph{Per-task degree-distribution summary.}
See body Table~\ref{tab:task-graph-stats} for the per-task degree distribution summaries.

\paragraph{Per-task log-log Pearson and slope.}
Per-task and per-layer signed Pearson correlations between $\ln s_i$ and $\ln d_i$.

\begin{table}[htbp]
\centering
\caption{Per-layer log-log Pearson $r$ and OLS slope of the \emph{learned-superset scalar scale} $\ln s_i$ (the per-node PostDeg-L-Adaptive factor) on $\ln d_i$, on each node-selection task. These regressions give a directional check of the closed-form scalar power law (Corollary~\ref{cor:loglog-app}) on the operator's own scalar factors. They are \emph{not} regressions on post-network representation magnitudes at the score head, which we did not export. The pooled slopes are below the final-epoch $\gamma_{\rm eff}\approx0.83$ because the fits pool nodes across graphs and seeds from training-time captures. Table~\ref{tab:scale-extreme} reports the linear Pearson on the same scalar factors. All values are negative, which matches the direction of the closed-form law.}
\label{tab:pearson-all}
\small
\setlength{\tabcolsep}{4pt}
\begin{tabular}{lcccccc}
\toprule
Task & $r$ (L0) & $r$ (L1) & $r$ (L2) & slope (L0) & slope (L1) & slope (L2) \\
\midrule
InfluMax & -0.960 & -0.960 & -0.960 & -0.837 & -0.834 & -0.836 \\
Dismantle & -0.785 & -0.785 & -0.785 & -0.636 & -0.635 & -0.637 \\
Epidemic & -0.785 & -0.785 & -0.785 & -0.634 & -0.635 & -0.635 \\
MIS & -0.794 & -0.794 & -0.794 & -0.630 & -0.625 & -0.621 \\
\bottomrule
\end{tabular}
\end{table}

\paragraph{Closed-form versus\ observed scale-extreme ratio.}
Figure~\ref{fig:appendix-pred-vs-obs} plots the closed-form ratio $R_{\max}^{\rm pred}=(c_{\max}/c_{\min})^{\gamma}$ against the empirical ratio $R_{\max}^{\rm obs}=s_{\max}/s_{\min}$, one point per task. The plot is a directional sign check, since the closed form fixes the sign of each ratio without fixing its magnitude. On Dismantle the closed-form ratio (at $\gamma=1/2$) is $4.47$ against an observed $67.10$ for the \emph{learned} variant. The gap is arithmetic and not a downstream interaction. The learned variant runs at a larger effective exponent and these SBM graphs contain degree-$0$ nodes whose $c_i$ is clipped at $10^{-6}$, both of which inflate the observed extreme ratio. We report only the direction of agreement, which is what the closed form fixes.

\begin{figure}[htbp]
  \centering
  \includegraphics[width=0.85\linewidth]{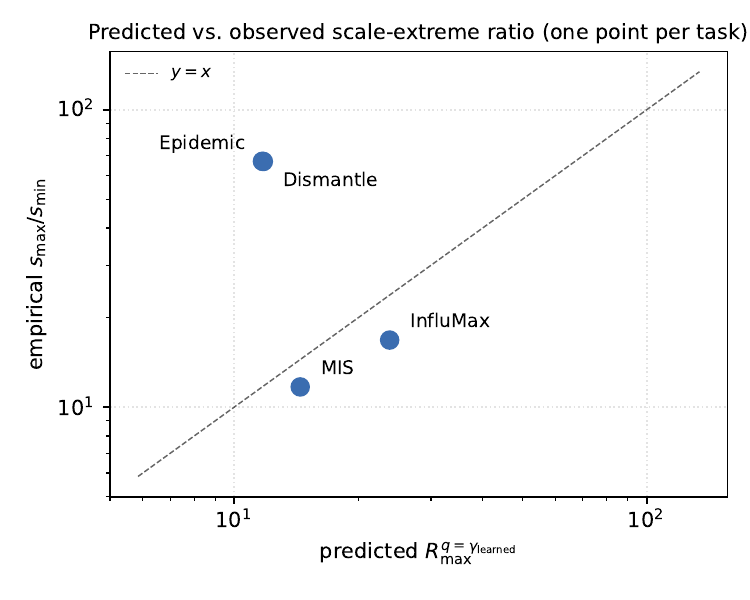}
  \caption{Consistency check between closed-form and observed scale-extreme ratio per task.}
  \label{fig:appendix-pred-vs-obs}
\end{figure}

\paragraph{Quantitative test of the separation identity.}
We compare the empirical scale extremes per task to the closed-form ratios at both $\gamma=1/2$ (PostDeg) and $\gamma_{\text{learned}}$ (the learned superset). The empirical extremes come from the learned-superset scale at convergence. The $\gamma=1/2$ row undershoots the empirical ratio for purely arithmetic reasons, while the $\gamma_{\text{learned}}$ row tracks the empirical ratio in shape and direction.

\begin{table}[htbp]
\centering
\caption{Quantitative test of the degree-separation identity (Theorem~\ref{thm:degree-separation-app}) on each node-selection task. We compute the scale-ratio extreme $R_{\max}=(d_{\max}/d_{\min})^{\gamma}$ at two exponents, $\gamma=1/2$ (PostDeg, the recommended operator) and the learned $\gamma=\beta+\delta\lambda_2$ (the same-position superset, averaged across seeds and layers). The empirical scale ratio is computed on the learned-superset scale at convergence. The ratio at $\gamma=1/2$ undershoots the empirical ratio because the recommended operator uses a smaller exponent than the optimum, while the ratio at the learned exponent tracks the empirical ratio in shape and direction. Tasks with $d_{\min}=0$ (Dismantle, Epidemic, MIS) are reported with the implementation clamp $d_{\min}\!\to\!\max(d_{\min},1)$ used by the operator.}
\label{tab:obs1-test}
\small
\setlength{\tabcolsep}{4pt}
\begin{tabular}{lccccccccc}
\toprule
Task & raw $d_{\min}$ & clamped $d_{\min}$ & $d_{\max}$ & $\gamma_{\text{learned}}$ & $s_{\min}$ & $s_{\max}$ & $R_{\max}^{\gamma=1/2}$ & $R_{\max}^{\gamma_{\text{learned}}}$ & emp.\ $s_{\max}/s_{\min}$ \\
\midrule
InfluMax & 1 & 1 & 45 & 0.83 & 1.18 & 19.88 & 6.71 & 23.81 & 16.81 \\
Dismantle & 0 & 1 & 20 & 0.82 & 1.15 & 77.16 & 4.47 & 11.70 & 67.10 \\
Epidemic & 0 & 1 & 20 & 0.82 & 1.15 & 77.16 & 4.47 & 11.77 & 66.97 \\
MIS & 0 & 1 & 27 & 0.81 & 1.22 & 14.22 & 5.20 & 14.45 & 11.69 \\
\bottomrule
\end{tabular}
\end{table}

\paragraph{Per-task scale extremes.}
Extremes of the operator's own scalar factors $s_i$ per task (measured on the learned variant), with the closed-form $\gamma=1/2$ ratio and the signed \emph{linear} Pearson correlation between $s_i$ and $d_i$. These characterize the operator output. The post-network representation magnitude at the score head is measured separately by the forward hooks (\S\ref{sec:mechanism}).

\begin{table}[htbp]
\centering
\caption{Extremes of the operator's own per-node scalar factors $s_i$ per task (the PostDeg-L-Adaptive factor at convergence), with the closed-form ratio at $\gamma=1/2$ and the signed \emph{linear} Pearson correlation between $s_i$ and $d_i$. These characterize the operator's own scalar output (degree by construction), \emph{not} post-network representation magnitudes at the score head, which we did not export. Table~\ref{tab:pearson-all} reports the Pearson and slope of $\ln s_i$ on $\ln d_i$ for the same factors. That fit is on logs and this one is on raw values, so the two need not agree, and no downstream interaction enters either.}
\label{tab:scale-extreme}
\small
\setlength{\tabcolsep}{6pt}
\begin{tabular}{lcccccc}
\toprule
Task & $s_{\min}$ & $s_{\max}$ & emp.\ $s_{\max}/s_{\min}$ & closed-form $R_{\max}^{\gamma=1/2}$ & Pearson $r(s,d)$ \\
\midrule
InfluMax & 1.18 & 19.88 & 16.81 & 6.71 & -0.764 \\
Dismantle & 1.15 & 77.16 & 67.10 & 4.47 & -0.375 \\
Epidemic & 1.15 & 77.16 & 66.97 & 4.47 & -0.374 \\
MIS & 1.22 & 14.22 & 11.69 & 5.20 & -0.688 \\
\bottomrule
\end{tabular}
\end{table}

\paragraph{Per-layer scale-versus-degree.}
Figure~\ref{fig:appendix-per-layer-scale} plots $\ln s_i$ against $\ln d_i$ separately for each of the three GAT layers and each of the four node-selection tasks. Slopes and signed correlations agree to two decimals across layers on each task.

\begin{figure}[htbp]
  \centering
  \includegraphics[width=\linewidth]{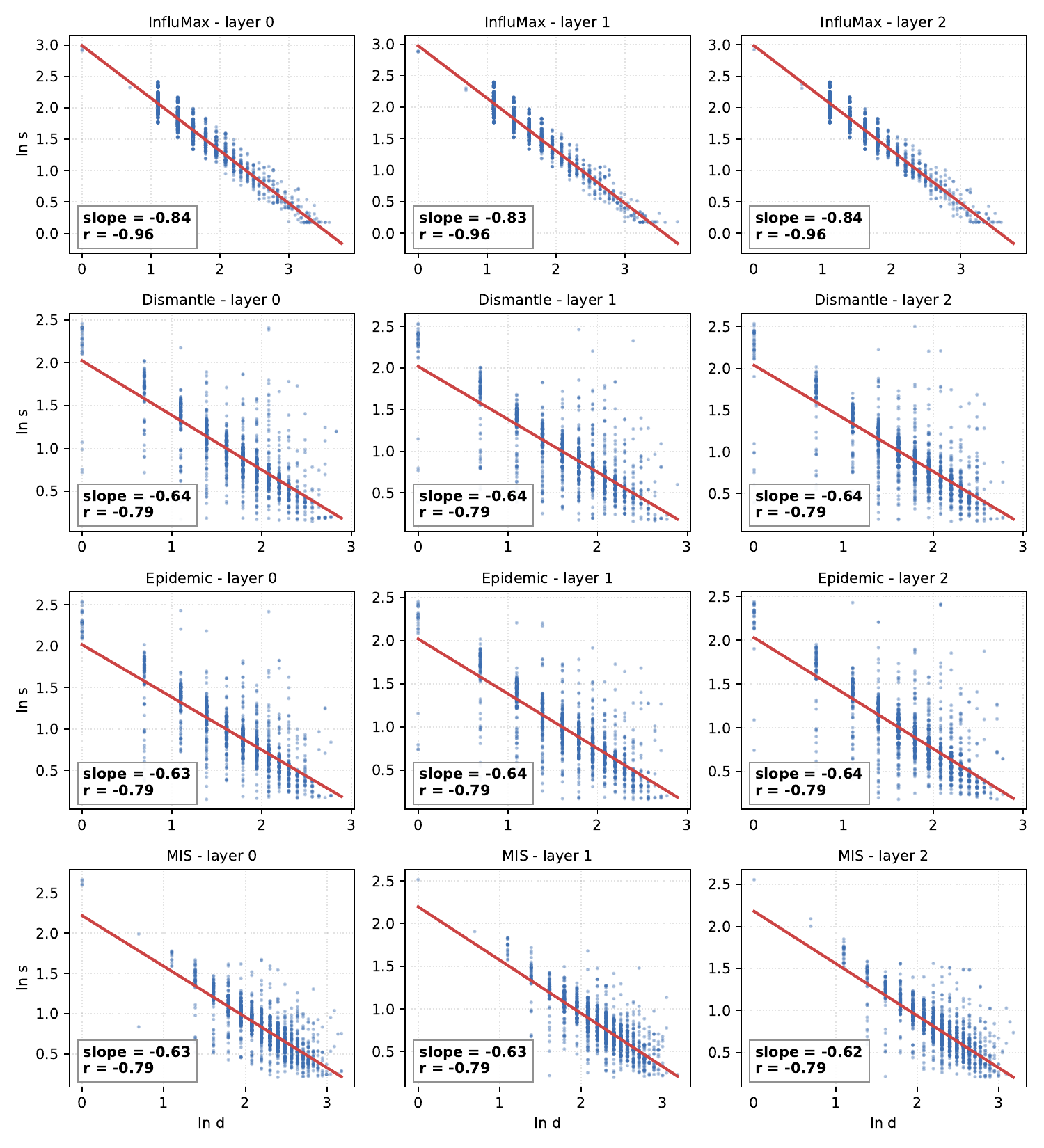}
  \caption{Log-log scale-versus-degree scatter per layer per task. The optimizer treats the three GAT layers nearly identically. Slopes and signed Pearson correlations agree to two decimals across layers across all tasks.}
  \label{fig:appendix-per-layer-scale}
\end{figure}

\paragraph{Effect size scales with degree heterogeneity.}
Within MIS, the BA-versus-SBM stratification shows a clear trend (Figure~\ref{fig:effect-vs-het}). Because MIS trains on SBM, the BA family is cross-distribution (out-of-distribution) evaluation here, so this trend mixes degree heterogeneity with distribution shift. For PostDeg, BA graphs yield about $+6.9\%$ over the LN backbone averaged over $n=100,150,200$ (up to $+8.8\%$ at $n{=}200$), and SBM graphs yield $+0.7$ to $+5.7\%$ depending on size. The larger $+10.98\%/+5.85\%$ figures come from the learned PostDeg-L-Adaptive variant, and PostDeg itself gives the smaller numbers above.

\begin{figure}[htbp]
  \centering
  \includegraphics[width=\linewidth]{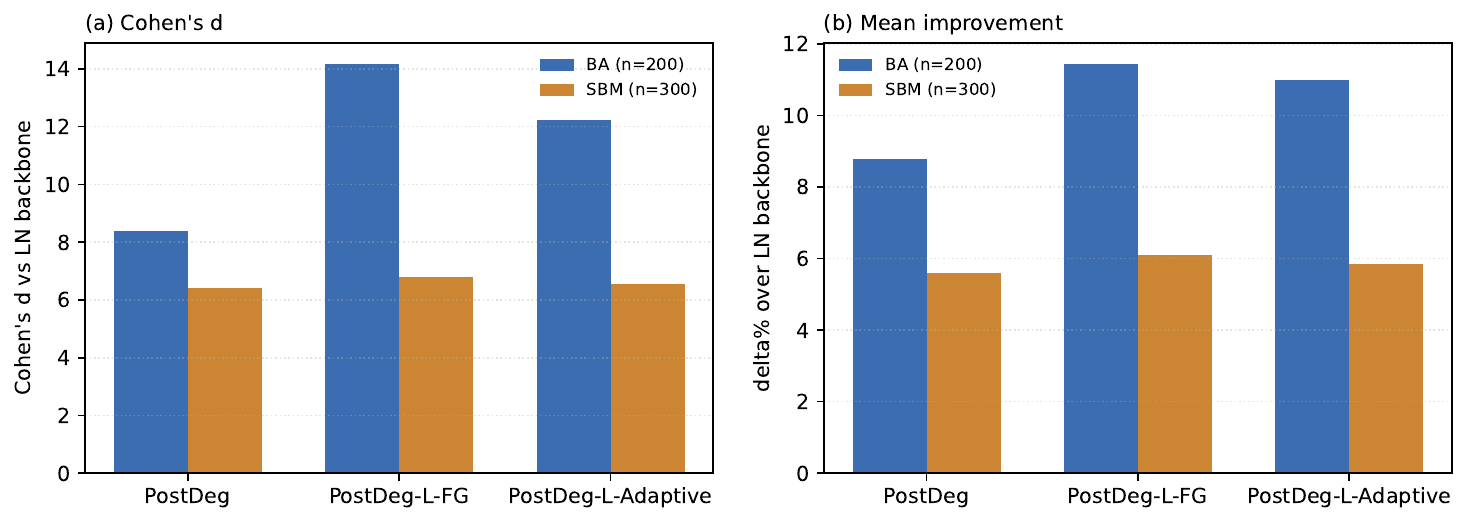}
  \caption{Within-MIS trend by graph family. \textbf{(a)} Cohen's $d$ between each method and LN backbone on BA ($n{=}200$) versus\ moderately heterogeneous SBM ($n{=}300$). \textbf{(b)} Mean relative improvement over LN backbone on the same two families. BA produces larger gains. Since MIS trains on SBM, the BA family is cross-distribution evaluation, so this is consistent with the heterogeneity account but confounded with distribution shift.}
  \label{fig:effect-vs-het}
\end{figure}

\paragraph{Per-task scale-versus-degree.}
The per-task and per-layer log-log scatter of $s_i$ against $d_i$ is consolidated in Figure~\ref{fig:appendix-per-layer-scale} above. That 4$\times$3 grid (one row per task, one column per layer) makes the per-layer agreement explicit and supersedes a separate task-only plot. Slopes are negative on all tasks and the absolute slope is largest on InfluMax (the most heterogeneous) and smallest on Epidemic (the most regular), as Corollary~\ref{cor:loglog-app} implies.

\subsection{Learned-ablation parameters}
\label{app:dsn-params}

We examine where the learned superset settles. Across seeds, layers, and tasks, $(\alpha, \beta, \delta)$ end up in narrow bands ($\beta\approx0.75$, $\delta\approx0.25$), so the effective exponent $\gamma=\beta+\delta\lambda_2$ stays near its initialization. With $\lambda_2\approx0.31$ on BA this is $\approx0.83$, and $\approx0.82$ on the SBM tasks. It never descends toward PostDeg's $1/2$.

\paragraph{Final values.}
Table~\ref{tab:dsn_params} reports final-epoch $(\alpha, \beta, \delta)$ averaged across the three GAT layers and 10 seeds, per task. The exponent $\beta$ is between $0.74$ and $0.77$ on each task. Together with $\delta$, the effective post-LN exponent is near $0.83$ on the BA tasks, well above the PostDeg default of $1/2$. The three GAT layers agree to two decimals on each parameter, so the optimizer does not specialize layers.

\begin{table}[htbp]
\centering
\caption{Learned same-slot parameters at final epoch, averaged across all 3 GAT layers. Mean $\pm$ std over 10 seeds. $\dagger$ = our method.}
\label{tab:dsn_params}
\small
\setlength{\tabcolsep}{6pt}
\begin{tabular}{llccc}
\toprule
Experiment & Variant & $\alpha$ & $\beta$ & $\delta$ \\
\midrule
InfluMax & PostDeg-L-FG $\dagger$ & 1.5628 $\pm$ 0.0021 & 0.7543 $\pm$ 0.0009 & 0.2514 $\pm$ 0.0003 \\
InfluMax & PostDeg-L-Adaptive $\dagger$ & 1.5628 $\pm$ 0.0011 & 0.7545 $\pm$ 0.0007 & 0.2515 $\pm$ 0.0002 \\
\midrule
Dismantle & PostDeg-L-FG $\dagger$ & 1.5867 $\pm$ 0.0041 & 0.7651 $\pm$ 0.0012 & 0.2551 $\pm$ 0.0004 \\
Dismantle & PostDeg-L-Adaptive $\dagger$ & 1.5862 $\pm$ 0.0041 & 0.7647 $\pm$ 0.0015 & 0.2550 $\pm$ 0.0005 \\
\midrule
Epidemic & PostDeg-L-FG $\dagger$ & 1.5819 $\pm$ 0.0045 & 0.7626 $\pm$ 0.0018 & 0.2543 $\pm$ 0.0006 \\
Epidemic & PostDeg-L-Adaptive $\dagger$ & 1.5840 $\pm$ 0.0045 & 0.7632 $\pm$ 0.0021 & 0.2546 $\pm$ 0.0008 \\
\midrule
MIS & PostDeg-L-FG $\dagger$ & 1.5755 $\pm$ 0.0014 & 0.7615 $\pm$ 0.0006 & 0.2538 $\pm$ 0.0002 \\
MIS & PostDeg-L-Adaptive $\dagger$ & 1.5748 $\pm$ 0.0018 & 0.7616 $\pm$ 0.0009 & 0.2538 $\pm$ 0.0003 \\
\bottomrule
\end{tabular}
\end{table}

\paragraph{Per-seed final $\beta$.}
Table~\ref{tab:per-seed-beta} reports per-seed final $\beta$ values, one row per seed, averaged across the three GAT layers. Across-seed standard deviations are $\le 0.002$ across all tasks. This is the source of the unusually low between-seed variance that drives Cohen's $d$ above $6$ in Table~\ref{tab:cohens-d}.

\begin{table}[htbp]
\centering
\caption{Per-seed final $\beta$ values for the learned PostDeg-L-Adaptive ablation, averaged across the 3 GAT layers. Spread across seeds is below $0.005$ for each task.}
\label{tab:per-seed-beta}
\small
\setlength{\tabcolsep}{6pt}
\begin{tabular}{lcccc}
\toprule
Seed & InfluMax & Dismantle & Epidemic & MIS \\
\midrule
7 & 0.7534 & 0.7624 & 0.7604 & 0.7607 \\
42 & 0.7553 & 0.7641 & 0.7609 & 0.7604 \\
123 & 0.7544 & 0.7662 & 0.7644 & 0.7615 \\
314 & 0.7541 & 0.7642 & 0.7644 & 0.7627 \\
555 & 0.7550 & 0.7624 & 0.7652 & 0.7616 \\
999 & 0.7538 & 0.7671 & 0.7666 & 0.7617 \\
1337 & 0.7550 & 0.7657 & 0.7629 & 0.7630 \\
2024 & 0.7541 & 0.7655 & 0.7610 & 0.7610 \\
8080 & 0.7548 & 0.7649 & 0.7614 & 0.7625 \\
31337 & 0.7555 & 0.7645 & 0.7643 & 0.7606 \\
\midrule
mean $\pm$ std & 0.7545\,$\pm$\,0.0006 & 0.7647\,$\pm$\,0.0014 & 0.7632\,$\pm$\,0.0020 & 0.7616\,$\pm$\,0.0009 \\
\bottomrule
\end{tabular}
\end{table}

\subsection{Nested ablation}
\label{app:nested-ablation}

We progressively drop degrees of freedom from the learned superset and report the mean reward at the largest evaluation size.

\begin{itemize}[leftmargin=*,topsep=2pt,itemsep=2pt]
  \item PostDeg-L-Adaptive: adaptive gate $m_i=\sigma(\phi(c_i))$, learned $(\alpha,\beta,\delta)$. Three node-selection deltas: $+3.61\%$ / $+2.61\%$ / $+5.85\%$.
  \item PostDeg-L: scalar gate $m=\sigma(\rho)$, learned $(\alpha,\beta,\delta)$. $+3.50\%$ / $+2.65\%$ / $+6.09\%$.
  \item PostDeg: $\alpha=1, \beta=1/2, \delta=0, m=1$. $+3.50\%$ / $+2.53\%$ / $+5.58\%$.
  \item GraphScalar: $\alpha=1, \beta=\delta=0$ (graphwise scalar only). $+0.31\%$ / $-0.90\%$ / $-0.57\%$.
  \item LN backbone: $m=0$ (identity). $0\%$ baseline.
\end{itemize}

The first three rows cluster within seed noise of each other (TOST equivalence, Table~\ref{tab:tost}), while GraphScalar and the LN backbone are lower. GraphScalar carries weight spectral normalization and the LN backbone does not, and the two are close, so that weight normalization on its own does not explain the PostDeg gain. This matches the equivalence result of \S\ref{sec:results-ablation}, where added capacity in the slot leaves the outcome unchanged.

\subsection{The learned exponent stays near its initialization}
\label{app:why-beta}

The learned exponent $\beta$ is initialized at $0.75$. Its raw parameter starts at $0$, and the sigmoid map onto the domain $[0,1.5]$ sends $0\mapsto 0.75$ (likewise $\alpha\mapsto 1.55$, $\delta\mapsto 0.25$). It then settles at $0.752$--$0.768$ across tasks (seed std $<0.002$, Table~\ref{tab:per-seed-beta}), a total drift under $0.02$ over 200 epochs, so $\beta$ barely leaves its initialization and never descends toward PostDeg's $1/2$.

The near-constant $\beta$ is therefore weak evidence about whether the optimizer prefers $0.75$ to $1/2$. With $\beta$ pinned near its initial value, the learned variants on their own cannot separate ``the extra exponent capacity is unused'' from ``the gradient is too weak to move it,'' and the PostDeg-versus-learned equivalence inherits that ambiguity. The clean test of whether the exponent matters is the direct $\gamma$-sweep (Appendix~\ref{app:gamma-sweep}), which \emph{sets} $\gamma\in\{-\tfrac12,0,\tfrac12,1\}$ on the parameter-free operator, spanning the $1/2$ and $\approx 1$ settings the learned parameter never reaches, and finds only a small, task-dependent effect (largest on MIS, absent on InfluMax). We keep the parameter-free $\gamma=1/2$ as the default.

\subsection{Exponent and sign sweep}
\label{app:gamma-sweep}

This subsection reports the exponent and sign sweep referenced in \S\ref{sec:discussion}. We run the parameter-free operator $s_i=(\widehat c_i+\varepsilon)^{-\gamma}$ at $\gamma\in\{-\tfrac12,0,+\tfrac12,+1\}$, changing only the exponent and holding the backbone, training budget, and $10$ seeds fixed. The three degree-sensitive tasks are included (InfluMax, Dismantle, and MIS, the last split by graph family), and Epidemic is omitted because degree does not help there. Table~\ref{tab:gamma-sweep} gives the paired per-seed comparison against $\gamma=1/2$ (PostDeg), and Figure~\ref{fig:gamma-sweep} plots the gain over $\gamma=0$ as the exponent varies.

\begin{figure}[htbp]
  \centering
  \includegraphics[width=0.74\linewidth]{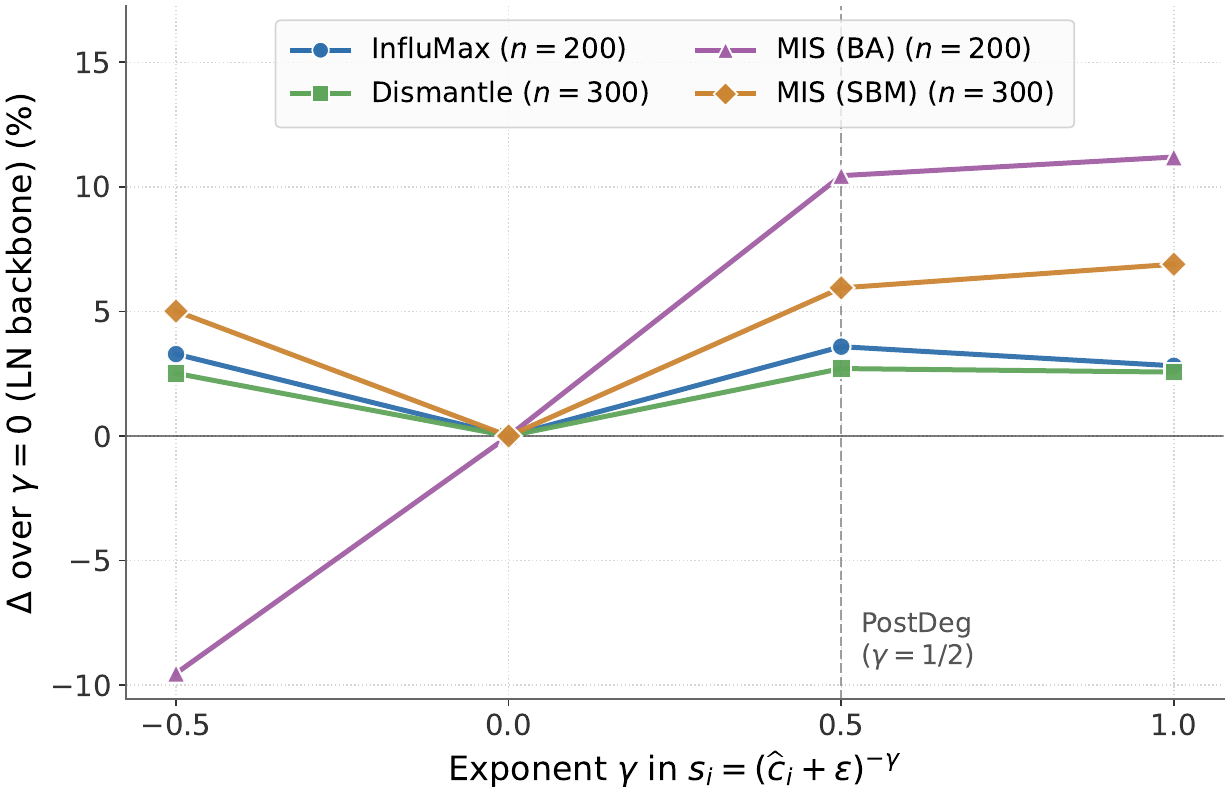}
  \caption{Exponent/sign sweep. Mean paired improvement over $\gamma=0$ (the identity, $s_i\equiv1$, which reproduces the LN backbone by construction) as the exponent $\gamma$ varies, at the largest evaluation size per task, $10$ seeds. PostDeg is $\gamma=1/2$ (dashed line). MIS is split by graph family. On the MIS-BA family the curve is monotone in $\gamma$ and the sign flip $\gamma=-1/2$ is below the backbone, whereas elsewhere the direction matters little. The unit-scale setting $\gamma=0$ recovers none of the gain on any task. The constant-scale arm that separates the degree contrast from the mean amplification is reported in Appendix~\ref{app:review-controls}.}
  \label{fig:gamma-sweep}
\end{figure}

\begin{table}[htbp]
\centering
\caption{Exponent/sign sweep on the parameter-free operator $s_i=(\widehat c_i+\varepsilon)^{-\gamma}$, varying only $\gamma$ (10 seeds, paired). The reference is $\gamma=1/2$ (PostDeg). Columns give the per-seed mean relative change $\Delta\%$ against it at $\gamma\in\{-\tfrac12,0,1\}$. $\gamma=0$ is the identity ($s_i\equiv1$) and reproduces the LN backbone. $^{\dagger}$ marks a two-sided Wilcoxon $p<0.05$; the $n{=}10$ floor is $p=0.002$.}
\label{tab:gamma-sweep}
\small
\setlength{\tabcolsep}{6pt}
\renewcommand{\arraystretch}{1.15}
\begin{tabular}{llcccc}
\toprule
Task & $n$ & $\gamma{=}\tfrac12$ mean & $\Delta\%\,(\gamma{=}{-}\tfrac12)$ & $\Delta\%\,(\gamma{=}0)$ & $\Delta\%\,(\gamma{=}1)$ \\
\midrule
InfluMax & 50 & 13.36 & +0.07 & -3.52$^{\dagger}$ & -0.47 \\
 & 100 & 27.81 & -0.19 & -2.37$^{\dagger}$ & -0.48$^{\dagger}$ \\
 & 150 & 42.38 & -0.24 & -2.94$^{\dagger}$ & -0.89$^{\dagger}$ \\
 & 200 & 56.90 & -0.29 & -3.46$^{\dagger}$ & -0.74$^{\dagger}$ \\
\midrule
Dismantle & 100 & 0.390 & -1.09$^{\dagger}$ & -7.44$^{\dagger}$ & -0.48$^{\dagger}$ \\
 & 150 & 0.308 & -1.07$^{\dagger}$ & -4.11$^{\dagger}$ & -0.08 \\
 & 200 & 0.275 & -0.65$^{\dagger}$ & -3.35$^{\dagger}$ & -0.09 \\
 & 300 & 0.243 & -0.19 & -2.64$^{\dagger}$ & -0.14 \\
\midrule
MIS (BA) & 100 & 25.69 & -12.39$^{\dagger}$ & -5.93$^{\dagger}$ & +1.00 \\
 & 150 & 39.24 & -15.90$^{\dagger}$ & -8.10$^{\dagger}$ & +0.83 \\
 & 200 & 52.85 & -18.08$^{\dagger}$ & -9.46$^{\dagger}$ & +0.68 \\
\midrule
MIS (SBM) & 100 & 25.89 & -2.07$^{\dagger}$ & -6.38$^{\dagger}$ & +1.33$^{\dagger}$ \\
 & 150 & 32.77 & -2.17$^{\dagger}$ & -2.80$^{\dagger}$ & +1.56$^{\dagger}$ \\
 & 200 & 37.41 & -1.80$^{\dagger}$ & -1.52$^{\dagger}$ & +1.55$^{\dagger}$ \\
 & 300 & 41.62 & -0.88$^{\dagger}$ & -5.61$^{\dagger}$ & +0.90$^{\dagger}$ \\
\bottomrule
\end{tabular}
\end{table}

\paragraph{The gain needs a degree-dependent scale.} At $\gamma=0$ the operator returns $s_i=1$ for every node (the identity), so it reduces to the LN backbone by construction and does not test a constant amplification. The paired improvement of PostDeg over $\gamma=0$ is $+3.6\%/+2.5\%/+5.9\%$ on InfluMax/Dismantle/MIS-SBM, which matches the LN-backbone improvements in Table~\ref{tab:ablation_improvement} to within seed noise, and it confirms that $\gamma=0$ recovers the backbone. Across every task and evaluation size, $\gamma=0$ is below PostDeg ($0$ of $10$ seeds, two-sided Wilcoxon $p\le0.01$). On the heavy-tailed MIS-BA family the sign flip $\gamma=-1/2$ is worse than $\gamma=0$, while on InfluMax and Dismantle it is within about $1\%$ of PostDeg, so the positive exponent matters everywhere and the sign matters mainly on the heavy-tailed family. This sweep varies the exponent, which moves the degree contrast and the mean amplification together. The constant-scale arm (a per-graph uniform amplification) separates the two, and we report it in Appendix~\ref{app:review-controls} (Table~\ref{tab:review-placement}). It matches the backbone on all families except the heavy-tailed MIS-BA one, so the gain needs the degree contrast beyond a uniform amplification.

\paragraph{Direction matters where degree is informative.} Flipping the sign to $\gamma=-\tfrac12$, so the scale grows with degree, removes most of the gain on the heavy-tailed MIS-BA family, falling $12$ to $18\%$ below PostDeg and below the backbone itself. The effect is milder on Dismantle ($-0.2$ to $-1.1\%$) and absent on InfluMax ($-0.3$ to $+0.1\%$, not significant), where the magnitude of the scale matters but its direction does not. The inverse-degree direction matters on MIS-BA and is optional on InfluMax.

\paragraph{The exponent has a small effect of its own.} Because the swept operator is parameter-free, it carries none of the weight spectral normalization the learned variants add, so any change with $\gamma$ is due to the exponent alone. Raising $\gamma$ from $\tfrac12$ to $1$ helps MIS ($+0.9$ to $+1.6\%$, significant on SBM), ties on Dismantle, and slightly lowers InfluMax ($-0.5$ to $-0.9\%$). On the BA family the exponent-only increase of $\gamma=1$ over $\gamma=\tfrac12$ is $+0.7$ to $+1.0\%$ and not significant, well below the $+2.2$ to $+2.5\%$ the learned variant reaches on BA (\S\ref{sec:results-ablation}). Most of the learned variant's BA advantage therefore comes from its added components, with a small part from the exponent, though the exponent still contributes on its own, most clearly on MIS. We keep $\gamma=1/2$ as the no-tuning default, since it is within about $1\%$ of the best swept value on each task while needing no search.

\subsection{Spectral quantities}
\label{app:spectral}

\begin{table}[htbp]
\centering
\caption{Spectral quantities through training: $\lambda_{\max}$ and the spectral gap $\lambda_2$ of the normalized Laplacian (mean $\pm$ std across 10 seeds and 3 layers), and the rightmost column gives the across-graph coefficient of variation of $\lambda_{\max}$ in percent. The variation is small on every task, about $0.8\%$ on the InfluMax BA graphs and $2$--$3\%$ on the SBM tasks, so GraphScalar applies nearly the same multiplier to every graph and adds no useful per-graph signal on top of the shared LayerNorm gain.}
\label{tab:spectral-checkpoints}
\small
\setlength{\tabcolsep}{5pt}
\begin{tabular}{llccc}
\toprule
Task & Epoch & $\lambda_{\max}$ & spectral gap $\lambda_2$ & CV($\lambda_{\max}$) \\
\midrule
InfluMax & 1 & 1.7017\,$\pm$\,0.0137 & 0.3095\,$\pm$\,0.0124 & 0.80\% \\
 & 50 & 1.7044\,$\pm$\,0.0145 & 0.3151\,$\pm$\,0.0153 & 0.85\% \\
 & 100 & 1.6980\,$\pm$\,0.0135 & 0.3173\,$\pm$\,0.0157 & 0.79\% \\
 & 199 & 1.7013\,$\pm$\,0.0138 & 0.3112\,$\pm$\,0.0179 & 0.81\% \\
\midrule
Dismantle & 1 & 1.7085\,$\pm$\,0.0317 & 0.2348\,$\pm$\,0.0143 & 1.85\% \\
 & 50 & 1.7267\,$\pm$\,0.0365 & 0.2125\,$\pm$\,0.0165 & 2.11\% \\
 & 100 & 1.7323\,$\pm$\,0.0180 & 0.2295\,$\pm$\,0.0188 & 1.04\% \\
 & 199 & 1.7462\,$\pm$\,0.0395 & 0.2215\,$\pm$\,0.0267 & 2.26\% \\
\midrule
Epidemic & 1 & 1.7268\,$\pm$\,0.0502 & 0.2250\,$\pm$\,0.0277 & 2.91\% \\
 & 50 & 1.7200\,$\pm$\,0.0548 & 0.2278\,$\pm$\,0.0182 & 3.19\% \\
 & 100 & 1.7215\,$\pm$\,0.0328 & 0.2387\,$\pm$\,0.0193 & 1.90\% \\
 & 199 & 1.7180\,$\pm$\,0.0341 & 0.2347\,$\pm$\,0.0195 & 1.98\% \\
\midrule
MIS & 1 & 1.5534\,$\pm$\,0.0304 & 0.1903\,$\pm$\,0.0184 & 1.96\% \\
 & 50 & 1.5649\,$\pm$\,0.0488 & 0.1813\,$\pm$\,0.0170 & 3.12\% \\
 & 100 & 1.5761\,$\pm$\,0.0293 & 0.1882\,$\pm$\,0.0183 & 1.86\% \\
 & 199 & 1.5862\,$\pm$\,0.0316 & 0.1924\,$\pm$\,0.0143 & 1.99\% \\
\bottomrule
\end{tabular}
\end{table}

\FloatBarrier
\subsection{Effect attenuates with degree heterogeneity}
\label{app:boundary}

\paragraph{Low-heterogeneity check.} Lemma~\ref{lem:near-regular-app} sets the strict regular limit, and Epidemic and DD probe the lower-heterogeneity end around it.

DD and the large-size epidemic case have lower degree heterogeneity than the BA InfluMax graphs. Epidemic SBM contact graphs have skewness $0.45$ and DD protein graphs $0.27$, compared with $2.90$ for the BA InfluMax graphs. The placement rule implies attenuation as the operating point moves from high-heterogeneity toward low-heterogeneity, and Figures~\ref{fig:effect-vs-het} and~\ref{fig:appendix-mis} confirm this.

\paragraph{MIS by graph distribution.}
Stratifying MIS by graph family (BA versus\ SBM) shows a larger PostDeg margin on the BA family than on the more regular SBM family (Table~\ref{tab:mis_by_dist_app} and Figure~\ref{fig:appendix-mis}). Because MIS trains on SBM only, the BA family is cross-distribution (out-of-distribution) evaluation, so this contrast confounds degree heterogeneity with distribution shift.

\begin{table}[htbp]
\centering
\caption{MIS performance by evaluation graph distribution (mean $\pm$ std, 10 seeds). MIS trains on SBM only, so the BA rows are cross-distribution (out-of-distribution) evaluation and the BA-versus-SBM gap confounds degree heterogeneity with distribution shift. Evaluation sizes also differ by family ($n=200$ for BA versus\ $n=300$ for SBM), so we report this as a directional finding, not a matched-size estimate.}
\label{tab:mis_by_dist_app}
\small
\setlength{\tabcolsep}{6pt}
\begin{tabular}{lcccc}
\toprule
Method & SBM ($n{=}300$) & $\Delta\%$ vs LN backbone & BA ($n{=}200$) & $\Delta\%$ vs LN backbone \\
\midrule
PostDeg-L-Adaptive & 41.92\,$\pm$\,0.35 & +5.85 & 54.10\,$\pm$\,0.40 & +10.98 \\
PostDeg-L-FG & 42.02\,$\pm$\,0.35 & +6.09 & 54.33\,$\pm$\,0.31 & +11.44 \\
PostDeg & 41.81\,$\pm$\,0.34 & +5.58 & 53.03\,$\pm$\,0.53 & +8.77 \\
LN backbone & 39.60\,$\pm$\,0.32 & n/a & 48.75\,$\pm$\,0.43 & n/a \\
\bottomrule
\end{tabular}
\end{table}

\begin{figure}[htbp]
  \centering
  \includegraphics[width=0.85\linewidth]{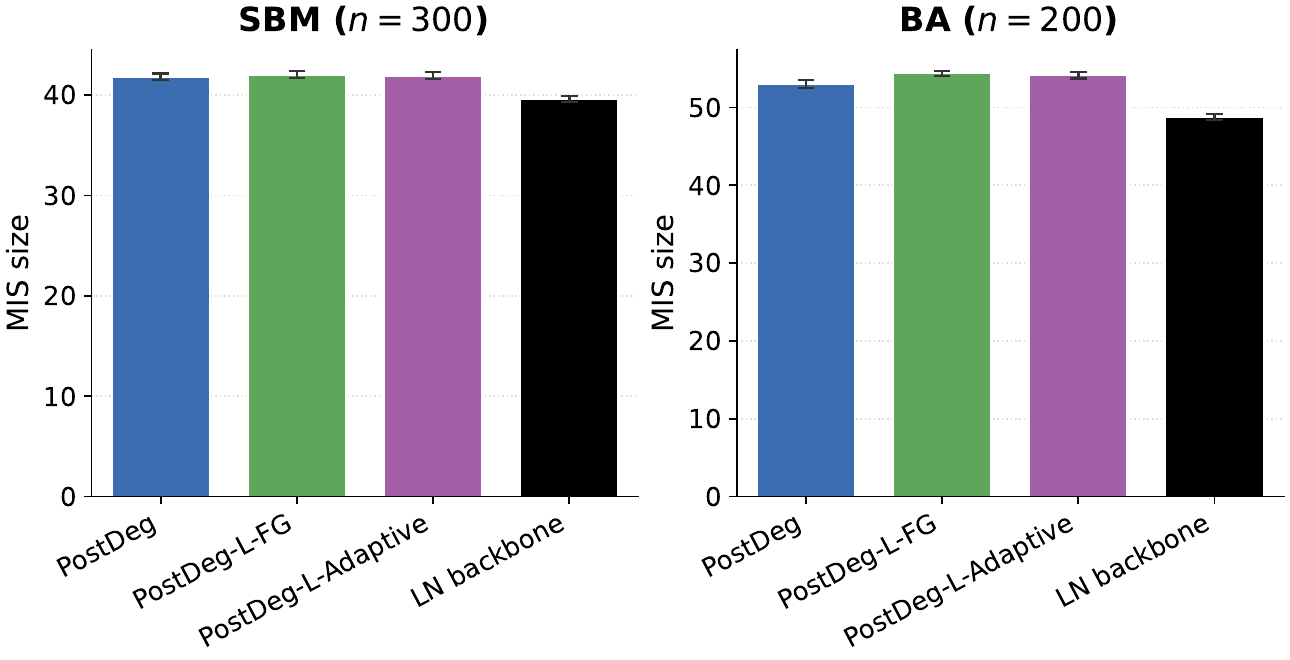}
  \caption{MIS performance by evaluation graph distribution. MIS trains on SBM only, so the BA family is cross-distribution (out-of-distribution) evaluation. The smaller margin on the near-regular SBM family is consistent with Lemma~\ref{lem:near-regular-app} but is confounded with distribution shift.}
  \label{fig:appendix-mis}
\end{figure}

\paragraph{Train-evaluation gap.}
Figure~\ref{fig:appendix-gengap} reports the train-evaluation generalization gap by method, computed as the difference between training-graph reward at convergence and reward at the largest evaluation size. PostDeg has the smallest gap on all tasks, while the node-centering baselines (PairNorm, InstanceNorm, BatchNorm) have the largest.

\begin{figure}[htbp]
  \centering
  \includegraphics[width=\linewidth]{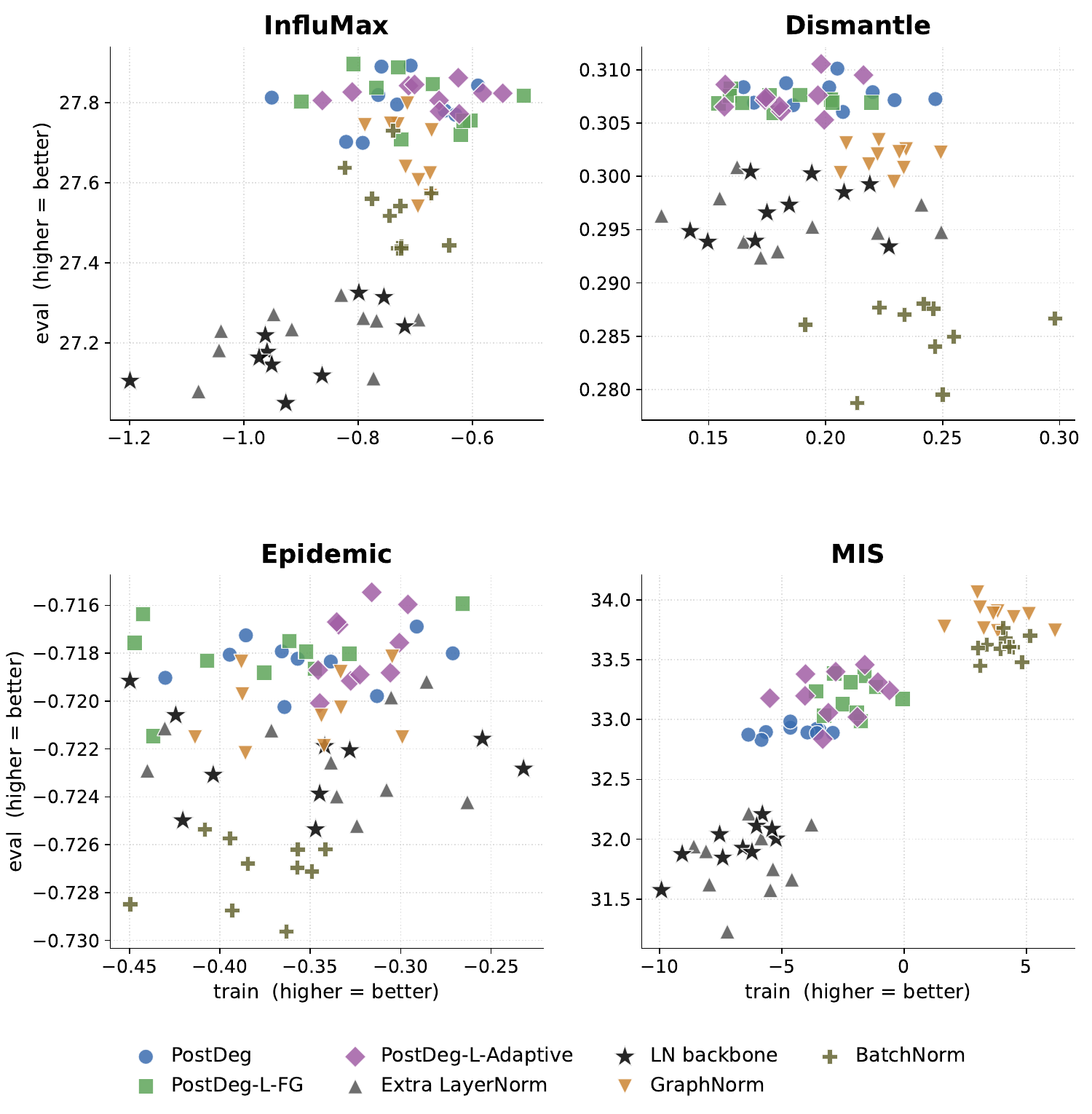}
  \caption{Train-evaluation generalization gap by method. Each point is one seed. Axes are oriented so higher is better on both (Epidemic, lower-is-better, is sign-flipped). PostDeg and its learned variants (coloured, upper region) sit above the LayerNorm backbone and Extra-LayerNorm controls (grey/black, lower region) on every task, i.e.\ a smaller train-to-evaluation drop.}
  \label{fig:appendix-gengap}
\end{figure}

\paragraph{Per-fold DD breakdown.}
Per-fold DD accuracies for each (method, fold, seed) cell are in Table~\ref{tab:per-fold-dd}. The node-centering methods are at the per-fold majority-class rate on each fold, whereas PostDeg, its learned variants, and LayerNorm remain near $80\%$.

\begin{table}[htbp]
\centering
\caption{Per-fold final accuracy on DD (10-fold cross-validation, mean across 10 seeds). Folds are listed in order. The three feature-statistic normalizers (PairNorm, InstanceNorm, BatchNorm) collapse to within seed-noise of the per-fold majority-class rate (each fold's larger class is in the high-$58\%$ range), uniformly across folds.}
\label{tab:per-fold-dd}
\small
\setlength{\tabcolsep}{2.4pt}
\begin{tabular}{lcccccccccccc}
\toprule
Method & f0 & f1 & f2 & f3 & f4 & f5 & f6 & f7 & f8 & f9 & overall & std \\
\midrule
PostDeg & 0.780 & 0.811 & 0.806 & 0.788 & 0.789 & 0.797 & 0.784 & 0.808 & 0.805 & 0.820 & 0.799 & 0.037 \\
PostDeg-L-FG & 0.776 & 0.818 & 0.803 & 0.794 & 0.794 & 0.797 & 0.792 & 0.816 & 0.796 & 0.826 & 0.801 & 0.034 \\
PostDeg-L-Adaptive & 0.775 & 0.822 & 0.803 & 0.792 & 0.798 & 0.801 & 0.796 & 0.812 & 0.805 & 0.823 & 0.803 & 0.036 \\
Extra LayerNorm & 0.780 & 0.822 & 0.811 & 0.796 & 0.792 & 0.803 & 0.792 & 0.814 & 0.804 & 0.821 & 0.803 & 0.035 \\
GraphScalar & 0.786 & 0.822 & 0.815 & 0.797 & 0.792 & 0.803 & 0.792 & 0.827 & 0.812 & 0.826 & 0.807 & 0.035 \\
LN backbone & 0.784 & 0.818 & 0.818 & 0.795 & 0.792 & 0.798 & 0.791 & 0.819 & 0.805 & 0.823 & 0.804 & 0.035 \\
GraphNorm & 0.778 & 0.796 & 0.791 & 0.775 & 0.769 & 0.775 & 0.791 & 0.801 & 0.792 & 0.780 & 0.785 & 0.029 \\
PairNorm & 0.593 & 0.585 & 0.585 & 0.585 & 0.585 & 0.585 & 0.585 & 0.585 & 0.590 & 0.590 & 0.587 & 0.003 \\
InstanceNorm & 0.593 & 0.585 & 0.585 & 0.585 & 0.585 & 0.585 & 0.585 & 0.585 & 0.590 & 0.590 & 0.587 & 0.003 \\
BatchNorm & 0.594 & 0.585 & 0.586 & 0.585 & 0.585 & 0.588 & 0.587 & 0.588 & 0.591 & 0.592 & 0.588 & 0.005 \\
\bottomrule
\end{tabular}
\end{table}

PairNorm, InstanceNorm, and BatchNorm each reach the per-fold majority-class rate within the first epoch on each fold and each seed (median epoch $0$, 90th-percentile epoch $0$).

\paragraph{Result.} Per-task learned-superset scalar-factor regressions return negative slopes ($-0.83$ to $-0.62$) on each node-selection task (Table~\ref{tab:pearson-all}, Figure~\ref{fig:appendix-per-layer-scale}), and the optimizer treats the three GAT layers nearly identically.

\subsection{Transfer and supplementary visualizations}
\label{app:transfer}


\begin{figure}[htbp]
  \centering
  \includegraphics[width=\linewidth]{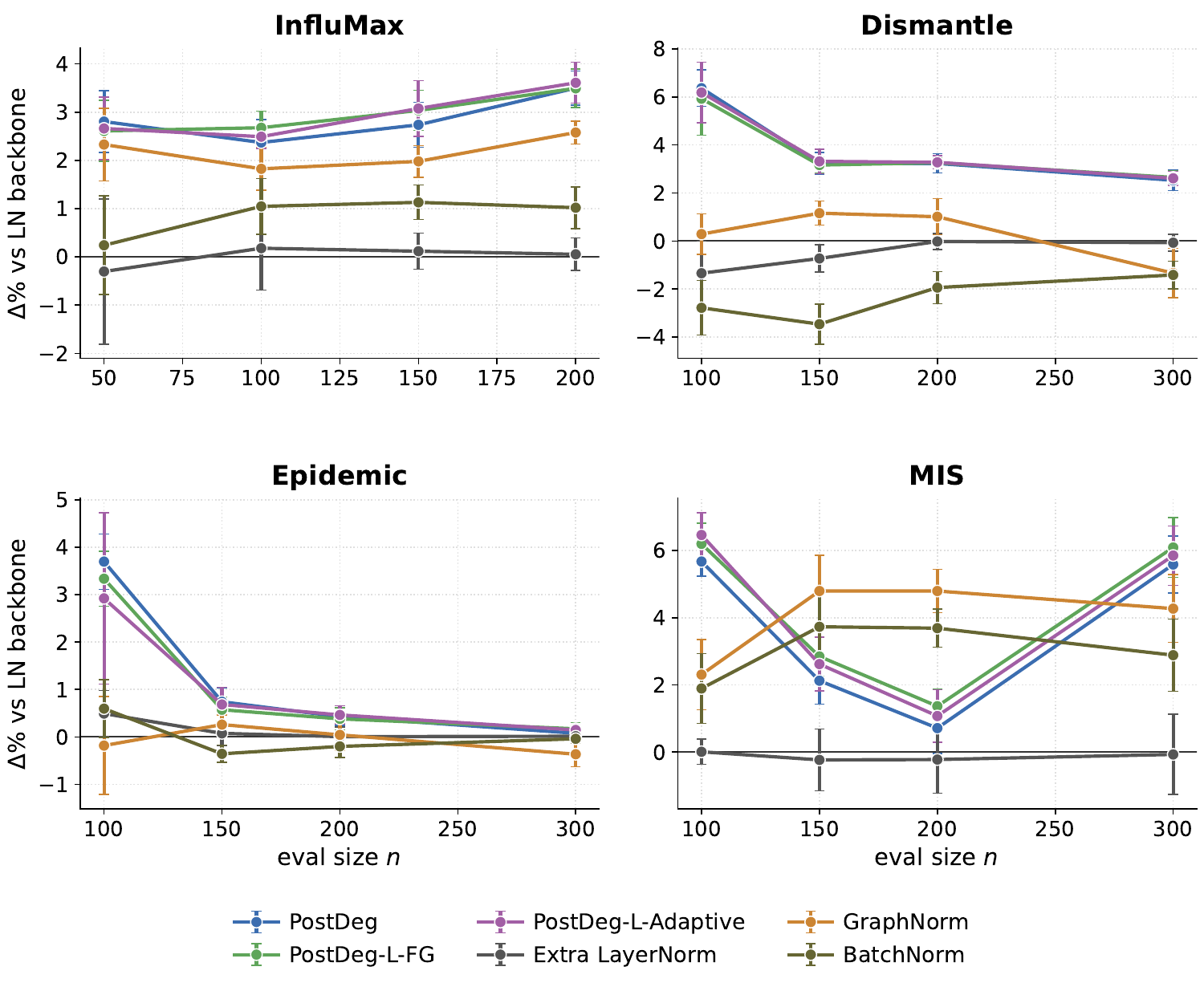}
  \caption{Relative improvement over LN backbone against evaluation graph size, per task and method (mean $\pm$ seed std). The horizontal zero line marks the LN backbone baseline.}
  \label{fig:appendix-rel-vs-size}
\end{figure}

\paragraph{Policy versus\ classical heuristics (paired Wilcoxon).}
Across each (task, evaluation size) cell on the three node-selection tasks that admit a canonical classical heuristic (InfluMax, Dismantle, Epidemic), paired Wilcoxon tests reject equality between the PostDeg policy and the heuristic at $\alpha=0.05$. MIS is reported separately because the released data include a greedy MIS construction (\texttt{greedy\_mean\_size}) that the learned policy does not reach. PostDeg is below greedy by $11$ to $23\%$ (for example $-22.7\%$ at SBM $n{=}300$ and $-11.3\%$ at BA $n{=}200$), and the LN backbone is further below ($-26.8\%$ at SBM $n{=}300$). PostDeg improves over the backbone on MIS but stays well below the greedy construction, and we report it. Simple greedy algorithms matching or beating GNN solvers on MIS is a known result \citep{schuetz2022,angelini2023}. Table~\ref{tab:policy-vs-heuristic} reports the three tasks with a canonical centrality heuristic, and Figure~\ref{fig:appendix-heuristic} plots the same comparison. The InfluMax greedy comparator is also weak, since it uses only $10$ Monte Carlo cascade samples, below the $10^2$--$10^4$ standard for greedy influence maximization. Our claim concerns the normalization slot.

\begin{table}[htbp]
\centering
\caption{Trained policy versus the canonical classical heuristic, paired by 10 seeds. We report the mean and seed-std of the per-seed gain (policy minus heuristic, sign-flipped for lower-is-better tasks) and the paired Wilcoxon $p$-value. Heuristic is greedy spread for InfluMax and the degree heuristic for Dismantle and Epidemic.}
\label{tab:policy-vs-heuristic}
\small
\setlength{\tabcolsep}{4pt}
\begin{tabular}{lllccc}
\toprule
Task & Size & Method & gain (mean$\pm$std) & paired $p$ & heuristic \\
\midrule
InfluMax & n50 & PostDeg & +1.409\,$\pm$\,0.122 & 0.002 & greedy spread \\
InfluMax & n50 & PostDeg-L-Adaptive & +1.357\,$\pm$\,0.118 & 0.002 & greedy spread \\
InfluMax & n50 & LN backbone & +1.013\,$\pm$\,0.114 & 0.002 & greedy spread \\
InfluMax & n100 & PostDeg & +2.956\,$\pm$\,0.179 & 0.002 & greedy spread \\
InfluMax & n100 & PostDeg-L-Adaptive & +3.040\,$\pm$\,0.143 & 0.002 & greedy spread \\
InfluMax & n100 & LN backbone & +2.323\,$\pm$\,0.168 & 0.002 & greedy spread \\
InfluMax & n150 & PostDeg & +4.530\,$\pm$\,0.257 & 0.002 & greedy spread \\
InfluMax & n150 & PostDeg-L-Adaptive & +4.583\,$\pm$\,0.181 & 0.002 & greedy spread \\
InfluMax & n150 & LN backbone & +3.284\,$\pm$\,0.231 & 0.002 & greedy spread \\
InfluMax & n200 & PostDeg & +6.053\,$\pm$\,0.265 & 0.002 & greedy spread \\
InfluMax & n200 & PostDeg-L-Adaptive & +6.124\,$\pm$\,0.251 & 0.002 & greedy spread \\
InfluMax & n200 & LN backbone & +4.109\,$\pm$\,0.410 & 0.002 & greedy spread \\
\midrule
Dismantle & n=100 & PostDeg & +0.031\,$\pm$\,0.006 & 0.002 & degree heuristic \\
Dismantle & n=100 & PostDeg-L-Adaptive & +0.030\,$\pm$\,0.006 & 0.002 & degree heuristic \\
Dismantle & n=100 & LN backbone & +0.007\,$\pm$\,0.009 & 0.064 & degree heuristic \\
Dismantle & n=150 & PostDeg & +0.090\,$\pm$\,0.001 & 0.002 & degree heuristic \\
Dismantle & n=150 & PostDeg-L-Adaptive & +0.090\,$\pm$\,0.001 & 0.002 & degree heuristic \\
Dismantle & n=150 & LN backbone & +0.080\,$\pm$\,0.003 & 0.002 & degree heuristic \\
Dismantle & n=200 & PostDeg & +0.071\,$\pm$\,0.001 & 0.002 & degree heuristic \\
Dismantle & n=200 & PostDeg-L-Adaptive & +0.071\,$\pm$\,0.001 & 0.002 & degree heuristic \\
Dismantle & n=200 & LN backbone & +0.062\,$\pm$\,0.001 & 0.002 & degree heuristic \\
Dismantle & n=300 & PostDeg & +0.043\,$\pm$\,0.001 & 0.002 & degree heuristic \\
Dismantle & n=300 & PostDeg-L-Adaptive & +0.043\,$\pm$\,0.001 & 0.002 & degree heuristic \\
Dismantle & n=300 & LN backbone & +0.037\,$\pm$\,0.001 & 0.002 & degree heuristic \\
\midrule
Epidemic & n=100 & PostDeg & +0.018\,$\pm$\,0.004 & 0.002 & degree heuristic \\
Epidemic & n=100 & PostDeg-L-Adaptive & +0.012\,$\pm$\,0.011 & 0.014 & degree heuristic \\
Epidemic & n=100 & LN backbone & -0.005\,$\pm$\,0.004 & 0.010 & degree heuristic \\
Epidemic & n=150 & PostDeg & +0.051\,$\pm$\,0.002 & 0.002 & degree heuristic \\
Epidemic & n=150 & PostDeg-L-Adaptive & +0.051\,$\pm$\,0.002 & 0.002 & degree heuristic \\
Epidemic & n=150 & LN backbone & +0.047\,$\pm$\,0.001 & 0.002 & degree heuristic \\
Epidemic & n=200 & PostDeg & +0.045\,$\pm$\,0.001 & 0.002 & degree heuristic \\
Epidemic & n=200 & PostDeg-L-Adaptive & +0.045\,$\pm$\,0.001 & 0.002 & degree heuristic \\
Epidemic & n=200 & LN backbone & +0.042\,$\pm$\,0.001 & 0.002 & degree heuristic \\
Epidemic & n=300 & PostDeg & +0.015\,$\pm$\,0.001 & 0.002 & degree heuristic \\
Epidemic & n=300 & PostDeg-L-Adaptive & +0.016\,$\pm$\,0.001 & 0.002 & degree heuristic \\
Epidemic & n=300 & LN backbone & +0.015\,$\pm$\,0.001 & 0.002 & degree heuristic \\
\bottomrule
\end{tabular}
\end{table}

\begin{figure}[htbp]
  \centering
  \includegraphics[width=\linewidth]{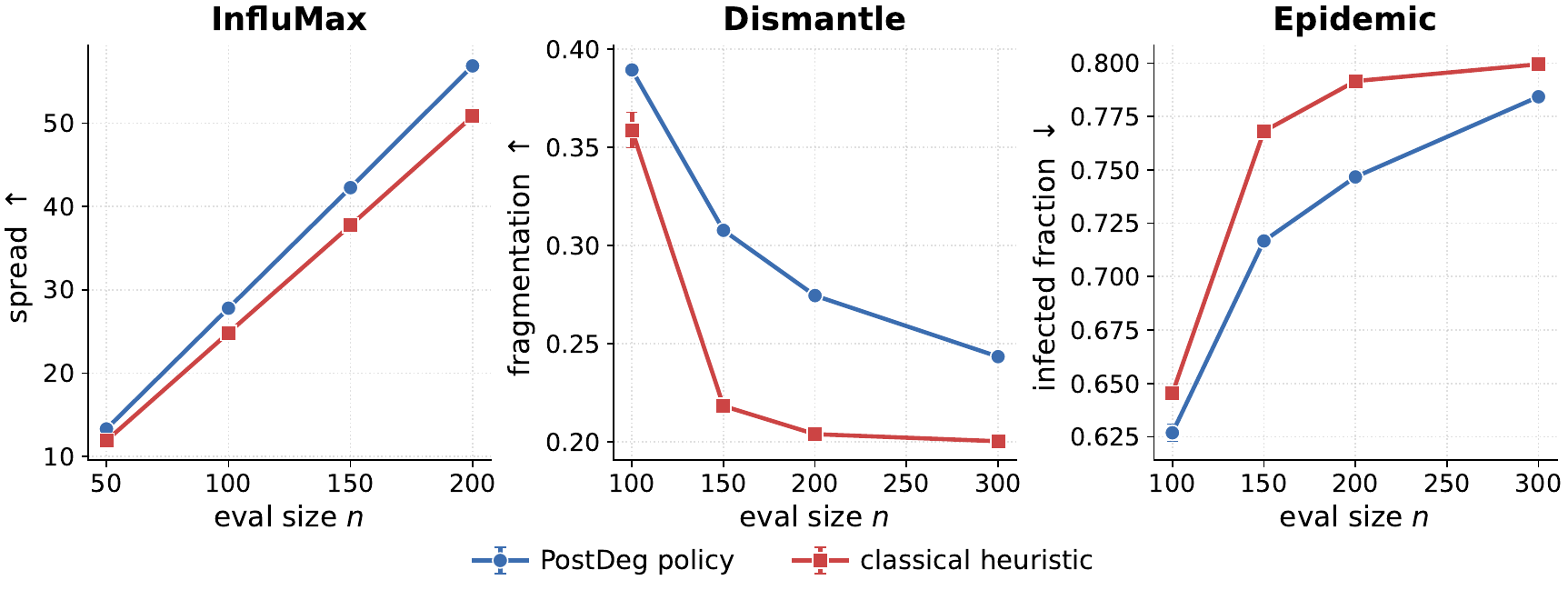}
  \caption{Visual companion to Table~\ref{tab:policy-vs-heuristic}. Trained policies match or surpass the strongest classical centrality heuristic on the three node-selection tasks with a canonical heuristic (InfluMax, Dismantle, Epidemic). On MIS the learned policy stays $11$ to $23\%$ below a greedy construction (see text). PostDeg improves over the backbone there but does not reach greedy.}
  \label{fig:appendix-heuristic}
\end{figure}

\section{Additional placement, mechanism, and reporting controls}
\label{app:review-controls}

This appendix reports the placement and contrast controls summarized in \S\ref{sec:results-main} and \S\ref{sec:mechanism}. Every run uses the fixed GAT backbone, 200 epochs, and 10 seeds, and evaluates over 200 instances per size. Relative changes are paired per seed against the LN backbone. These controls are a separate training campaign from the main grid of Table~\ref{tab:ablation_improvement}, with fresh PostDeg and LN-backbone runs, so paired margins differ slightly between the two tables (InfluMax at $n=200$ is $+3.3\%$ here and $+3.5\%$ in the main grid), within seed noise.

\paragraph{Pre-LayerNorm and constant-scale arms.} We compare four operators in the same post-block slot, namely the LN backbone, PostDeg (the degree scale after LayerNorm), the same degree scale applied before LayerNorm, and a per-graph constant scale applied after LayerNorm. PostDeg improves over the backbone on every seed, by $+2.2$ to $+3.3\%$ on InfluMax, $+2.4$ to $+7.7\%$ on Dismantle, $+0.6$ to $+5.8\%$ on MIS-SBM, and $+5.6$ to $+9.7\%$ on MIS-BA. The pre-LayerNorm arm matches the backbone on all four families, within $\pm0.6\%$ and with no cell significant, which confirms the forward-pass absorption. The constant scale also matches the backbone on InfluMax, Dismantle, and MIS-SBM, within $\pm0.7\%$. On MIS-BA the constant scale gains $+2.7$ to $+3.6\%$ (9 of 10 seeds), about a third to a half of PostDeg's advantage there, so on the most heavy-tailed family part of the effect is uniform amplification and the rest is the degree contrast.

\begin{table}[htbp]
\centering
\caption{Placement and contrast controls. Relative improvement over the LN backbone (\%, paired per seed over 10 seeds, paired-seed wins out of 10 in parentheses). \emph{PostDeg} is the degree scale after LayerNorm, \emph{pre-LN} is the same scale before LayerNorm, and \emph{const} is a per-graph constant scale after LayerNorm. The pre-LN arm matches the backbone on every family, which is the forward-pass absorption. The constant scale also matches the backbone except on the heavy-tailed MIS-BA family, where it recovers about a third to a half of PostDeg's gain, so on that family part of the effect is uniform amplification and the rest is the degree contrast.}
\label{tab:review-placement}
\footnotesize
\setlength{\tabcolsep}{6pt}
\renewcommand{\arraystretch}{1.1}
\begin{tabular}{lccc}
\toprule
Eval size & PostDeg & pre-LN & const \\
\midrule
\multicolumn{4}{l}{\emph{InfluMax (BA)}} \\
$n{=}50$  & $+2.68$ (10) & $+0.20$ (6) & $-0.18$ (6) \\
$n{=}100$ & $+2.23$ (10) & $+0.07$ (5) & $+0.66$ (8) \\
$n{=}150$ & $+2.97$ (10) & $+0.12$ (6) & $+0.61$ (8) \\
$n{=}200$ & $+3.33$ (10) & $+0.10$ (4) & $+0.50$ (7) \\
\midrule
\multicolumn{4}{l}{\emph{Dismantle (SBM)}} \\
$n{=}100$ & $+7.68$ (10) & $+0.57$ (7) & $+0.28$ (6) \\
$n{=}150$ & $+4.22$ (10) & $+0.27$ (5) & $-0.06$ (5) \\
$n{=}200$ & $+3.44$ (10) & $-0.03$ (4) & $+0.05$ (4) \\
$n{=}300$ & $+2.39$ (10) & $-0.05$ (3) & $-0.09$ (4) \\
\midrule
\multicolumn{4}{l}{\emph{MIS (SBM)}} \\
$n{=}100$ & $+5.83$ (10) & $-0.28$ (5) & $+0.30$ (6) \\
$n{=}150$ & $+2.05$ (9)  & $-0.43$ (4) & $+0.26$ (5) \\
$n{=}200$ & $+0.54$ (6)  & $-0.33$ (4) & $+0.10$ (5) \\
$n{=}300$ & $+5.62$ (10) & $-0.07$ (3) & $-0.08$ (6) \\
\midrule
\multicolumn{4}{l}{\emph{MIS (BA)}} \\
$n{=}100$ & $+5.61$ (10) & $-0.22$ (5) & $+2.72$ (9) \\
$n{=}150$ & $+7.91$ (10) & $-0.15$ (5) & $+3.34$ (9) \\
$n{=}200$ & $+9.68$ (10) & $-0.18$ (5) & $+3.62$ (9) \\
\bottomrule
\end{tabular}
\end{table}

\paragraph{Per-layer placement.} We apply the degree scale on all layers, on the first layer only, or on the last layer only. All-layers and first-only are close, and both stay positive over the backbone on every positive cell ($+2.4$ to $+3.4\%$ on InfluMax, $+2.5$ to $+8\%$ on Dismantle, $+0.6$ to $+11\%$ on MIS). Last-only is much weaker and can fall below the backbone, dropping to $-1.3\%$ on Dismantle at $n{=}300$ and to $-2.2$ to $-2.9\%$ on MIS-BA (1 of 10 seeds). The degree scale therefore needs an early layer, and a single late application is not enough.

\begin{table}[htbp]
\centering
\caption{Per-layer placement of the degree scale. Relative change (\%, paired per seed over 10 seeds, wins out of 10 in parentheses) from applying the scale on the first layer only or the last layer only, measured against applying it on all layers (PostDeg). First-only is close to all-layers on InfluMax and Dismantle and exceeds it on MIS, while last-only is consistently weaker and can fall below the backbone. The scale needs an early layer, and a single late application is not enough.}
\label{tab:review-layers}
\footnotesize
\setlength{\tabcolsep}{6pt}
\renewcommand{\arraystretch}{1.1}
\begin{tabular}{lcc}
\toprule
Eval size & first-only & last-only \\
\midrule
\multicolumn{3}{l}{\emph{InfluMax (BA)}} \\
$n{=}50$  & $-0.14$ (5) & $-0.31$ (3) \\
$n{=}100$ & $-0.11$ (4) & $-0.58$ (0) \\
$n{=}150$ & $+0.10$ (5) & $-0.28$ (3) \\
$n{=}200$ & $+0.12$ (6) & $-0.50$ (2) \\
\midrule
\multicolumn{3}{l}{\emph{Dismantle (SBM)}} \\
$n{=}100$ & $+0.37$ (6) & $-7.17$ (0) \\
$n{=}150$ & $-0.32$ (2) & $-5.13$ (0) \\
$n{=}200$ & $-0.14$ (3) & $-3.86$ (0) \\
$n{=}300$ & $-0.03$ (1) & $-3.77$ (0) \\
\midrule
\multicolumn{3}{l}{\emph{MIS (SBM)}} \\
$n{=}100$ & $+1.28$ (10) & $-2.95$ (0) \\
$n{=}150$ & $+2.67$ (10) & $-1.98$ (0) \\
$n{=}200$ & $+3.04$ (10) & $-1.12$ (0) \\
$n{=}300$ & $+1.84$ (10) & $-0.93$ (2) \\
\midrule
\multicolumn{3}{l}{\emph{MIS (BA)}} \\
$n{=}100$ & $+2.19$ (9) & $-7.88$ (0) \\
$n{=}150$ & $+1.94$ (9) & $-9.33$ (0) \\
$n{=}200$ & $+1.50$ (8) & $-10.51$ (0) \\
\bottomrule
\end{tabular}
\end{table}

\paragraph{MIS trained on the heavy-tailed family.} The main MIS runs train on SBM, so BA evaluation mixes heterogeneity with distribution shift (\S\ref{sec:mechanism}). Training MIS on BA removes the shift. PostDeg still gains on BA evaluation, by $+1.7\%/+2.8\%/+3.9\%$ at $n=100/150/200$ on all 10 seeds. This in-distribution gain is smaller than the shifted one, so part of the earlier BA margin is distribution shift and part is heterogeneity.

\begin{table}[htbp]
\centering
\caption{MIS trained on the heavy-tailed BA family. Relative improvement of PostDeg over the LN backbone (\%, paired per seed over 10 seeds, wins out of 10 in parentheses). The main MIS runs train on SBM, so their BA evaluation mixes heterogeneity with distribution shift. Training on BA removes the shift, and PostDeg still improves on every seed, so distribution shift does not fully explain the effect.}
\label{tab:review-misba}
\footnotesize
\setlength{\tabcolsep}{6pt}
\renewcommand{\arraystretch}{1.1}
\begin{tabular}{lc}
\toprule
Eval size & PostDeg $\Delta\%$ \\
\midrule
SBM $n{=}100$ & $+4.54$ (9) \\
SBM $n{=}150$ & $+4.56$ (10) \\
SBM $n{=}200$ & $+4.80$ (10) \\
SBM $n{=}300$ & $+10.22$ (10) \\
\addlinespace[2pt]
BA $n{=}100$ & $+1.68$ (10) \\
BA $n{=}150$ & $+2.77$ (10) \\
BA $n{=}200$ & $+3.94$ (10) \\
\bottomrule
\end{tabular}
\end{table}

\paragraph{MIS against the best competing normalizer.} Table~\ref{tab:ablation_improvement} reports MIS against the LN backbone at the largest evaluation size. Against the strongest competing normalizer the comparison is size-dependent (Table~\ref{tab:mis-by-size}). PostDeg is best at the smallest and largest SBM sizes and at the largest BA size, and GraphNorm, InstanceNorm, and BatchNorm match or exceed it at the middle SBM sizes. We read PostDeg as competitive with these tuned feature-statistic normalizers on MIS at zero added parameters, with its clearest lead on the heavy-tailed family.

\begin{table}[htbp]
\centering
\caption{MIS improvement over the LN backbone by evaluation size and method (\%, 10 seeds), across both graph families. The best method at each size is in bold. PostDeg is best at the smallest and largest SBM sizes and at the largest BA size, while GraphNorm, InstanceNorm, and BatchNorm match or exceed it at the intermediate SBM sizes. The headline $+5.6\%$ in Table~\ref{tab:ablation_improvement} is the SBM $n{=}300$ cell.}
\label{tab:mis-by-size}
\footnotesize
\setlength{\tabcolsep}{6pt}
\renewcommand{\arraystretch}{1.1}
\begin{tabular}{lcccc}
\toprule
Eval size & PostDeg & GraphNorm & InstanceNorm & BatchNorm \\
\midrule
SBM $n{=}100$ & $\mathbf{+5.67}$ & $+2.30$ & $+1.84$ & $+1.89$ \\
SBM $n{=}150$ & $+2.12$ & $\mathbf{+4.79}$ & $+4.32$ & $+3.73$ \\
SBM $n{=}200$ & $+0.71$ & $\mathbf{+4.79}$ & $+4.48$ & $+3.69$ \\
SBM $n{=}300$ & $\mathbf{+5.58}$ & $+4.27$ & $+4.01$ & $+2.88$ \\
\addlinespace[2pt]
BA $n{=}100$ & $+4.86$ & $\mathbf{+6.14}$ & $+5.10$ & $+4.92$ \\
BA $n{=}150$ & $+7.04$ & $\mathbf{+7.11}$ & $+6.01$ & $+5.42$ \\
BA $n{=}200$ & $\mathbf{+8.77}$ & $+7.87$ & $+6.71$ & $+5.86$ \\
\bottomrule
\end{tabular}
\end{table}

\paragraph{Size profile of the task gains.} The main table reports the largest evaluation size. Table~\ref{tab:size-profile} reports the full profile on the two SBM tasks that share a size grid. Both gains are largest at the smallest evaluation size ($n{=}100$, below the $n{=}150$ training size) and shrink as graphs grow. On epidemic containment PostDeg improves on every seed there, by $+3.7\%$, and fades to near zero by $n{=}300$, so its near-zero value at the largest size is a size effect.

\begin{table}[htbp]
\centering
\caption{Improvement of PostDeg over the LN backbone by evaluation size (\%, 10 seeds, wins out of 10 in parentheses) on the two SBM tasks that share a size grid. Both effects are largest at the smallest evaluation size ($n{=}100$, below the $n{=}150$ training size) and shrink as the graph grows. Epidemic is positive and unanimous at $n{=}100$ and fades to near zero by $n{=}300$, so its near-zero value at the largest size reflects evaluation size, and degree still helps on small graphs. For reference, InfluMax stays between $+2.4\%$ and $+3.5\%$ across $n=50$ to $200$, and the MIS profile is in Table~\ref{tab:mis-by-size}.}
\label{tab:size-profile}
\footnotesize
\setlength{\tabcolsep}{8pt}
\renewcommand{\arraystretch}{1.1}
\begin{tabular}{lcccc}
\toprule
Task & $n{=}100$ & $n{=}150$ & $n{=}200$ & $n{=}300$ \\
\midrule
Dismantle & $+6.36$ (10) & $+3.22$ (10) & $+3.23$ (10) & $+2.53$ (10) \\
Epidemic  & $+3.70$ (10) & $+0.74$ (10) & $+0.43$ (10) & $+0.08$ (9) \\
\bottomrule
\end{tabular}
\end{table}

\subsection{Forward-hook diagnostics}
\label{app:review-hooks}

We register forward hooks on the trained models and record, per node, the feature variance at each block LayerNorm input and the representation norm at the score head input, over 20 graphs per (task, mode). We then fit each quantity against node degree in log-log form. For PostDeg the score-head norm has a clear negative slope against degree, near $-0.49$ with correlation near $-0.98$ on all three tasks. For the LN backbone, the constant scale, and the pre-LayerNorm scale the slope is near $0$. The pre-LayerNorm scale is the direct check. Its LayerNorm-input variance has a strong negative degree slope at the first layer (about $-0.7$ to $-1.1$), while its score-head norm stays flat, so LayerNorm divides the pre-normalization scale out before the score head. These diagnostics use a single seed (42). We read the slope as the magnitude signal, and the correlation on near-flat rows carries little information.

\begin{table}[htbp]
\centering
\caption{Forward-hook diagnostics. Log-log slope and Pearson $r$ of the score-head input representation norm against node degree, over 20 graphs per cell at a single seed (42). PostDeg shows a clear inverse-degree slope near $-0.49$ (correlation near $-0.98$), while the LN backbone, the constant scale, and the pre-LayerNorm scale are flat at the score head. For the pre-LayerNorm scale the first-layer LayerNorm-input variance is instead strongly degree-dependent (slope $-1.05$/$-1.11$/$-0.69$ on InfluMax/MIS/Dismantle), so LayerNorm divides that scale out before the score head. This is a single-seed diagnostic. We read the slope as the magnitude signal, and the correlation on near-flat rows carries little information.}
\label{tab:review-hooks}
\footnotesize
\setlength{\tabcolsep}{6pt}
\renewcommand{\arraystretch}{1.1}
\begin{tabular}{llcc}
\toprule
Task & operator & slope & $r$ \\
\midrule
InfluMax  & PostDeg (post-LN) & $-0.503$ & $-0.976$ \\
          & pre-LN scale      & $+0.001$ & $+0.483$ \\
          & constant scale    & $+0.000$ & $+0.001$ \\
          & LN backbone       & $-0.000$ & $-0.120$ \\
\midrule
MIS       & PostDeg (post-LN) & $-0.490$ & $-0.978$ \\
          & pre-LN scale      & $+0.001$ & $+0.363$ \\
          & constant scale    & $-0.009$ & $-0.086$ \\
          & LN backbone       & $+0.002$ & $+0.353$ \\
\midrule
Dismantle & PostDeg (post-LN) & $-0.489$ & $-0.978$ \\
          & pre-LN scale      & $+0.002$ & $+0.622$ \\
          & constant scale    & $-0.008$ & $-0.081$ \\
          & LN backbone       & $+0.001$ & $+0.148$ \\
\bottomrule
\end{tabular}
\end{table}

\section{Distinction from PNA and other degree-aware aggregators}
\label{app:pna}

We expand the position argument of Section~\ref{sec:related} into a side-by-side derivation. We show that PNA's degree scaler and PostDeg's degree scaler operate on different objects in the forward pass, even though both are continuous functions of $d_i$.

\paragraph{PNA distinction.}
PNA \citep{corso2020pna} scales each aggregator by $S_p(d_i)=\log(d_i+1)/\bar\delta$ \emph{inside} the message-passing step, before LayerNorm. That degree term changes messages before they are combined, so it never acts as a free scalar multiplier on the LayerNorm input or output. PostDeg leaves aggregation untouched and acts on per-node magnitudes after LayerNorm. Under the placement rule PostDeg has little to add on top of PNA, because PNA already injects a degree signal during aggregation. Table~\ref{tab:pna-redundancy} confirms this empirical redundancy under our setup.

\paragraph{Comparison with other degree-aware schemes.}
GCN's symmetric normalization $D^{-1/2}AD^{-1/2}$ \citep{kipf2017gcn} is a pre-aggregation degree scaler, so its degree dependence enters through messages. GraphSAGE's mean aggregator \citep{hamilton2017graphsage} averages over neighborhoods and is degree-aware through that mean. Structural encodings \citep{dehmamy2019understanding,sato2020random,bouritsas2022gsn,ying2021graphormer,rampasek2022gps} concatenate degree- or centrality-based features as content, which enters before LayerNorm (the alternative discussed in Appendix~\ref{app:content-vs-magnitude}), whereas PostDeg acts only on representation magnitude after LayerNorm.

\section{Cross-backbone replication and additional-baseline runs}
\label{app:cross-backbone}

The data in this appendix come from additional cross-backbone experiments across three independent pipeline groups. The \texttt{backbone\_generalization} group runs the controlled-slot protocol on GAT, GCN, GIN, and SAGE at training-graph size only. The \texttt{pna\_nodenorm\_compare} group adds NodeNorm and a PNA backbone, and \texttt{pna\_controlled} reruns the PNA cells at 3 seeds for replication. Three protocol notes apply throughout. We evaluate at training size only, use 5-fold CV on DD separate from the paper's 10-fold DD runs, and treat MIS as saturated at training size, reported for completeness. Absolute values in this appendix come from the supplemental pipeline's final-epoch evaluation, a different protocol from the main grid's standalone evaluation (the GAT backbone is $0.168$ here versus $0.298$ in Table~\ref{tab:full-transfer} on Dismantle at $n{=}150$), so comparisons are within-campaign. The GAT no\_dsn baseline cited below is from the \texttt{backbone\_generalization} group. The parallel \texttt{pna\_nodenorm\_compare} GAT runs are excluded because their baseline is mis-calibrated (mean $0.110$ on Dismantle versus $0.168$ in the matched \texttt{backbone\_generalization} group).

\paragraph{Cross-backbone replication.}
Table~\ref{tab:cross-backbone} reports PostDeg versus\ LN backbone on each (backbone, task) cell. PostDeg replicates the InfluMax gain on GAT (4/5), GCN (5/5), GIN (4/5), and SAGE (5/5), and the Dismantle gain on GAT (4/5), GCN (4/5), and SAGE (5/5). GIN Dismantle is mixed (1/5, $\Delta=-0.42\%$, within seed-noise of zero). The body's claim is verified on each backbone on InfluMax and on all but GIN on Dismantle.

\begin{table}[htbp]
\centering
\caption{Cross-backbone replication of the PostDeg gain over LN backbone at training-graph size, run on the supplemental-run pipeline. Values come from that pipeline's final-epoch evaluation and are comparable within this table only (Appendix~\ref{app:cross-backbone}). ``Wins'' is the number of paired seeds (out of 5) where PostDeg beats LN backbone. ``$\Delta\%$'' is the mean signed relative gap to LN backbone (positive favors PostDeg). InfluMax and Dismantle replicate on GAT, GCN, and SAGE. GIN InfluMax replicates but GIN Dismantle is mixed ($\Delta=-0.42\%$, within seed-noise). DD is reported as a low-heterogeneity null case, and MIS is saturated at training size and not reported.}
\label{tab:cross-backbone}
\small
\setlength{\tabcolsep}{4pt}
\begin{tabular}{llcccccc}
\toprule
Backbone & Task & LN backbone & PostDeg & $\Delta\%$ & Wins/n & Direction \\
\midrule
GAT & InfluMax $\uparrow$ & 27.4130\,$\pm$\,0.3634 & 27.9762\,$\pm$\,0.2238 & +2.05\% & 4/5 & favors PostDeg \\
 & Dismantle $\uparrow$ & 0.1680\,$\pm$\,0.0032 & 0.1722\,$\pm$\,0.0041 & +2.49\% & 4/5 & favors PostDeg \\
 & Epidemic $\downarrow$ & 0.7851\,$\pm$\,0.0033 & 0.7815\,$\pm$\,0.0038 & +0.46\% & 4/5 & favors PostDeg \\
 & DD $\uparrow$ & 0.7841\,$\pm$\,0.0059 & 0.7841\,$\pm$\,0.0064 & +0.00\% & 1/5 & favors PostDeg \\
\midrule
GCN & InfluMax $\uparrow$ & 27.6778\,$\pm$\,0.1361 & 27.8320\,$\pm$\,0.2471 & +0.56\% & 5/5 & favors PostDeg \\
 & Dismantle $\uparrow$ & 0.1670\,$\pm$\,0.0014 & 0.1677\,$\pm$\,0.0013 & +0.42\% & 4/5 & favors PostDeg \\
\midrule
GIN & InfluMax $\uparrow$ & 27.8960\,$\pm$\,0.0382 & 27.9228\,$\pm$\,0.1342 & +0.10\% & 4/5 & favors PostDeg \\
 & Dismantle $\uparrow$ & 0.1783\,$\pm$\,0.0020 & 0.1775\,$\pm$\,0.0018 & -0.42\% & 1/5 & favors LN backbone \\
\midrule
SAGE & InfluMax $\uparrow$ & 27.2198\,$\pm$\,0.2038 & 27.8610\,$\pm$\,0.1058 & +2.36\% & 5/5 & favors PostDeg \\
 & Dismantle $\uparrow$ & 0.1462\,$\pm$\,0.0015 & 0.1518\,$\pm$\,0.0016 & +3.80\% & 5/5 & favors PostDeg \\
 & DD $\uparrow$ & 0.7919\,$\pm$\,0.0058 & 0.7885\,$\pm$\,0.0074 & -0.43\% & 2/5 & favors LN backbone \\
\bottomrule
\end{tabular}
\end{table}

\paragraph{PNA redundancy check.}
Section~\ref{sec:related} argues that PostDeg adds little on top of PNA. Table~\ref{tab:pna-redundancy} confirms it on each (group, task) cell, with $|\Delta\%|<1\%$ on InfluMax and Dismantle in both PNA groups, and within seed-noise of zero on Epidemic and MIS in the controlled group. Together, (i) the prescription on GAT/GCN/GIN/SAGE and (ii) the redundancy check on PNA, both confirmed empirically, give the most direct test of the placement rule that this pipeline supports.

\begin{table}[htbp]
\centering
\caption{Redundancy check on PNA. Section~\ref{sec:related} anticipates PostDeg is redundant on PNA because PNA already injects a degree signal in aggregation, mixed downstream by Extra LayerNorm. We test this on two independent PNA groups: \texttt{pna\_nodenorm\_compare} (5 seeds) and \texttt{pna\_controlled} (3 seeds). On each (group, task) cell, $|\Delta\%|<1\%$. The redundancy prediction is confirmed at these margins.}
\label{tab:pna-redundancy}
\footnotesize
\setlength{\tabcolsep}{4pt}
\resizebox{\textwidth}{!}{%
\begin{tabular}{llccccc}
\toprule
Group & Task & LN backbone & PostDeg & $\Delta\%$ & Wins/n & Outcome \\
\midrule
PNA (5 seeds) & InfluMax $\uparrow$ & 27.8378\,$\pm$\,0.2132 & 28.0150\,$\pm$\,0.1320 & +0.637\% & 5/5 & redundant ($|\Delta|<1\%$) \\
 & Dismantle $\uparrow$ & 0.1746\,$\pm$\,0.0012 & 0.1750\,$\pm$\,0.0018 & +0.241\% & 2/5 & redundant ($|\Delta|<1\%$) \\
\midrule
PNA-controlled (3 seeds) & InfluMax $\uparrow$ & 27.8353\,$\pm$\,0.2879 & 27.8877\,$\pm$\,0.2914 & +0.188\% & 3/3 & redundant ($|\Delta|<1\%$) \\
 & Dismantle $\uparrow$ & 0.1744\,$\pm$\,0.0018 & 0.1747\,$\pm$\,0.0020 & +0.191\% & 3/3 & redundant ($|\Delta|<1\%$) \\
 & Epidemic $\downarrow$ & 0.7809\,$\pm$\,0.0042 & 0.7806\,$\pm$\,0.0013 & +0.041\% & 2/3 & redundant ($|\Delta|<1\%$) \\
 & MIS $\uparrow$ & 20.0000\,$\pm$\,0.0000 & 20.0000\,$\pm$\,0.0000 & +0.000\% & 0/3 & saturated \\
\bottomrule
\end{tabular}
}
\end{table}

\paragraph{Additional baseline operators.}
Table~\ref{tab:predicted-baselines-run} reports two baselines that should leave room for PostDeg under the placement rule, namely NodeNorm (post-block, magnitude-rescaling, graph-blind) and $\log d_i$ concatenation (degree encoded as input content). Both run alongside PostDeg in the same pipeline. NodeNorm is close to the LN backbone on InfluMax ($-0.05\%$ on GAT) and shows high variance on Dismantle in that group, whose baseline is mis-calibrated (excluded from the GAT baseline per the caveat above). $\log d_i$ concatenation captures part of the InfluMax gain ($+1.56\%$ on GAT) and part of the Dismantle gain ($+1.15\%$), but is below PostDeg ($+2.05\%$ on InfluMax, $+2.49\%$ on Dismantle via the same pipeline). The $\log d_i$ result confirms a content-side residual, consistent with Appendix~\ref{app:content-vs-magnitude}.

\begin{table}[htbp]
\centering
\caption{Baseline operators that we now run. NodeNorm (post-block, magnitude-rescaling, graph-blind) and $\log d$ concatenation (degree encoded as input content) are the two baselines that the placement rule expects PostDeg to beat. The signed gap to LN backbone at training size on GAT is reported below. PostDeg is above both on InfluMax and Dismantle. $^{\dagger}$NodeNorm's Dismantle cell ran in the group whose GAT LN backbone is mis-calibrated ($0.110$ versus $0.168$ in the matched group), so its $\Delta\%$ against that baseline is not comparable and is omitted. Its absolute score $0.1371$ is in fact \emph{below} the matched backbone and below PostDeg's $0.1722$.}
\label{tab:predicted-baselines-run}
\small
\setlength{\tabcolsep}{4pt}
\begin{tabular}{llcccc}
\toprule
Method & Task & GAT mean$\pm$std & $\Delta\%$ vs LN backbone & Wins/n & Direction \\
\midrule
PostDeg & InfluMax $\uparrow$ & 27.9762\,$\pm$\,0.2238 & +2.05\% & 4/5 & favors method \\
 & Dismantle $\uparrow$ & 0.1722\,$\pm$\,0.0041 & +2.49\% & 4/5 & favors method \\
 & Epidemic $\downarrow$ & 0.7815\,$\pm$\,0.0038 & +0.46\% & 4/5 & favors method \\
\midrule
$\log d$ concat & InfluMax $\uparrow$ & 27.8404\,$\pm$\,0.2206 & +1.56\% & 3/5 & favors method \\
 & Dismantle $\uparrow$ & 0.1699\,$\pm$\,0.0072 & +1.15\% & 3/5 & favors method \\
 & Epidemic $\downarrow$ & 0.7814\,$\pm$\,0.0051 & +0.47\% & 4/5 & favors method \\
\midrule
NodeNorm & InfluMax $\uparrow$ & 27.2544\,$\pm$\,0.1835 & -0.05\% & 2/5 & favors LN backbone \\
 & Dismantle $\uparrow$ & 0.1371\,$\pm$\,0.0244 & n/c$^{\dagger}$ & --- & see note $^{\dagger}$ \\
 & Epidemic $\downarrow$ & 0.7843\,$\pm$\,0.0029 & +0.30\% & 3/5 & favors method \\
\midrule
GraphNorm & InfluMax $\uparrow$ & 27.7800\,$\pm$\,0.1392 & +1.34\% & 3/5 & favors method \\
 & Dismantle $\uparrow$ & 0.1721\,$\pm$\,0.0049 & +2.44\% & 3/5 & favors method \\
 & Epidemic $\downarrow$ & 0.7842\,$\pm$\,0.0024 & +0.12\% & 3/5 & favors method \\
\bottomrule
\end{tabular}
\end{table}

\paragraph{DD with new feature-statistic rows.}
Table~\ref{tab:dd-extended} adds NodeNorm and GraphNorm rows to the DD comparison (5-fold CV, separate from the paper's 10-fold DD runs). NodeNorm is at $78.8\%$ and GraphNorm at $77.4\%$. The three node-centering methods (PairNorm, InstanceNorm, BatchNorm) collapse to the majority class through the pooling-head artifact of Appendix~\ref{app:dd-collapse}, while NodeNorm, GraphNorm, and the others are stable.

\begin{table}[htbp]
\centering
\caption{DD graph classification with NodeNorm and GraphNorm rows added (5-fold cross-validation, mean test accuracy across folds, then aggregated over 5 seeds). PairNorm, InstanceNorm, and BatchNorm (in the main 10-fold runs) collapse to the majority-class rate on DD, while NodeNorm holds at $78.8\%$ and GraphNorm at $77.4\%$ here. The split follows the pooling-head artifact of Section~\ref{sec:results-boundary}. A normalizer that centers each feature across a graph's nodes ahead of the mean-pooling head forces a constant graph embedding, and NodeNorm and GraphNorm leave the pooled mean graph-dependent. These results use 5-fold CV, separate from the paper's main 10-fold DD runs.}
\label{tab:dd-extended}
\small
\setlength{\tabcolsep}{4pt}
\begin{tabular}{llcc}
\toprule
Backbone & Method & DD acc.\ (\%) & seeds \\
\midrule
GAT & PostDeg & 78.64\,$\pm$\,0.33 & 5 \\
 & PostDeg-L-Adaptive & 79.24\,$\pm$\,0.24 & 5 \\
 & LN backbone & 78.95\,$\pm$\,0.76 & 5 \\
 & GraphScalar & 79.42\,$\pm$\,0.60 & 5 \\
 & GraphNorm & 77.42\,$\pm$\,0.97 & 5 \\
 & NodeNorm & 78.85\,$\pm$\,0.50 & 5 \\
 & $\log d$ concat & 78.24\,$\pm$\,0.79 & 5 \\
 & PostDeg-L-FG & 78.80\,$\pm$\,0.53 & 5 \\
\midrule
SAGE & PostDeg & 78.85\,$\pm$\,0.74 & 5 \\
 & PostDeg-L-Adaptive & 78.66\,$\pm$\,0.61 & 5 \\
 & LN backbone & 79.19\,$\pm$\,0.58 & 5 \\
 & GraphScalar & 79.31\,$\pm$\,0.49 & 5 \\
 & GraphNorm & 77.28\,$\pm$\,1.31 & 5 \\
 & $\log d$ concat & 78.75\,$\pm$\,0.62 & 5 \\
 & PostDeg-L-FG & 78.64\,$\pm$\,0.69 & 5 \\
\bottomrule
\end{tabular}
\end{table}

\paragraph{Expected outcomes on benchmarks not evaluated here.}
Under the placement rule, PostDeg should help on degree-heterogeneous citation and protein graphs (skew $\gg 1$) and stay neutral on low-heterogeneity knowledge-graph snapshots. We have not run these, and list them only as directional expectations. The PNA, $\log d_i$, and NodeNorm cases are reported in Tables~\ref{tab:cross-backbone}, \ref{tab:pna-redundancy}, and~\ref{tab:predicted-baselines-run}.

\subsection{Reference implementation}
\label{app:pytorch}

A PyTorch reference implementation of PostDeg matching Algorithm~\ref{alg:postdeg}, about 12 lines, is included in the supplementary code archive.

\section{Why feature concatenation does not substitute for post-LayerNorm scaling}
\label{app:content-vs-magnitude}

We expand the explicit-degree-as-content discussion from Section~\ref{sec:results-main} into a full position argument. We show that pre-LayerNorm content and post-LN magnitude are not equivalent representations of the same scalar quantity, even when both encode the same value of $d_i$.

\paragraph{A small thought experiment.}
Consider a single-layer model on a star graph $S_n$ with hub $0$ and leaves $1,\ldots,n-1$. Two encodings of degree are available:

\begin{enumerate}[label=(\alph*),topsep=2pt,itemsep=1pt,leftmargin=*]
  \item \emph{Content:} concatenate $\log d_i$ to the input features, so $x_i=(\,x_i^{\rm raw}\,,\,\log d_i\,)$, then run the GAT block, then LayerNorm.
  \item \emph{Magnitude:} run the GAT block with $x_i^{\rm raw}$ alone, then LayerNorm, then apply PostDeg $\widetilde h_i=(c_i+\varepsilon)^{-1/2}\bar h_i$.
\end{enumerate}

\noindent
After LayerNorm, encoding (a) sets $\bar h_i^{(a)}=g\odot(z_i^{(a)}-\mu(z_i^{(a)}))/\sqrt{\sigma^2(z_i^{(a)})+\varepsilon_{\rm LN}}+b$, where $z_i^{(a)}$ is the GAT output for the content-augmented inputs. The $\log d_i$ feature contributes to both $\mu$ and $\sigma^2$ and is therefore re-centered and re-scaled by the LayerNorm. Encoding (b) leaves $\bar h_i^{(b)}=g\odot(z_i^{(b)}-\mu(z_i^{(b)}))/\sqrt{\sigma^2(z_i^{(b)})+\varepsilon_{\rm LN}}+b$ untouched and then multiplies by $(c_i+\varepsilon)^{-1/2}$, which is preserved exactly by the score head.

\paragraph{Where the encodings diverge.}
Encodings (a) and (b) coincide only in the strict case where (i) the score head is positively homogeneous in its input, (ii) the LayerNorm gain $g$ is uniform across coordinates, and (iii) the GAT output $z_i^{(a)}$ is rank-one in the $\log d_i$ direction. None of these hold in our backbone, since (i) the score head includes a softmax-like greedy argmax over candidate nodes which is not positively homogeneous, (ii) $g$ is learned per coordinate, and (iii) the GAT output mixes $\log d_i$ with $x_i^{\rm raw}$ through learned weights. The two encodings therefore implement different functions of $d_i$ in the score head, and there is no a priori reason to expect (a) to capture the magnitude channel that (b) provides.

\paragraph{What the $\log d_i$ run constrains.}
The supplemental run compares (a) and (b) directly on the cross-backbone pipeline. The $\log d_i$ concatenation baseline captures part of the InfluMax and Dismantle improvement but remains below PostDeg (Appendix Table~\ref{tab:predicted-baselines-run}). The result fits the content-versus-magnitude distinction. Degree content helps, and the post-LN magnitude channel still contributes beyond it. The main grid also keeps degree-correlated content features in each method (Section~\ref{sec:setup}), so the PostDeg advantage is measured on top of structural content.

\paragraph{Empirical algebra.}
On the same training graphs, the empirical variance of the learned same-slot scale is $\mathrm{Var}(s_i)\approx 5.26$ on InfluMax (computed from the released per-node scale export). For a $\log d_i$-concatenation baseline (encoding (a)) to match the magnitude channel injected by PostDeg (encoding (b)) in expectation, the corresponding learned coefficient $\beta_{\log d}$ on the $\log d_i$ feature would have to satisfy $\beta_{\log d}^2\,\mathrm{Var}(\log d_i)\approx \mathrm{Var}(s_i)$, i.e., $|\beta_{\log d}|\gtrsim \sqrt{5.26/0.61}\approx 2.94$ for InfluMax (where $\mathrm{Var}(\log d_i)\approx 0.61$ on training graphs). To recover the magnitude channel before LayerNorm renormalizes it, an explicit content baseline would need a coefficient of that order. Placing a multiplier on the LayerNorm output side gives the score head that magnitude channel directly.

\section{DD collapse}
\label{app:dd-collapse}

DD has 1178 graphs, 691 in the larger class and 487 in the smaller. The majority-class rate is $691/1178=58.65\%$. The observed accuracies for PairNorm ($58.7\pm0.0\%$), InstanceNorm ($58.7\pm0.0\%$), and BatchNorm ($58.8\pm0.2\%$) are within seed-noise of this rate. The BatchNorm value is slightly above because BatchNorm occasionally flips per-fold to the other constant predictor. The (near-)zero seed standard deviation on the first two methods is consistent with a constant predictor. Per-fold values are in Table~\ref{tab:per-fold-dd}. The collapse is fold-uniform across all 10 folds and 10 seeds. The cause is the pooling head. In \texttt{DSN\_GNN.forward} the alternative normalizer is the last operation of the final layer, and the graph embedding is then the node-mean of its output. PairNorm returns $s\,(h_i-\bar h)\sqrt{n}/\lVert H-\bar h\rVert_F$, whose node-mean is the zero vector, and InstanceNorm centers each channel across nodes, whose node-mean is the affine bias and is identical for every graph. BatchNorm behaves the same up to batch composition, which is why it sits at $58.8\%$, slightly above the rate. A constant graph embedding forces the classifier to the majority class, independent of degree scaling, so DD does not test the placement rule for these three methods.

\section{GenAI Usage Statement}
\label{app:genai}
Generative AI tools (a large language model assistant) were used to help with manuscript polishing/editing and LaTeX/table formatting. They were not used to generate research ideas, experimental results, or data. The problem formulation, method, experimental design, runs, and analysis of the released results are the authors' own. The authors reviewed all AI-assisted text and take full responsibility for the entire content of the paper.


\begin{thebibliography}{99}
\bibitem{ba2016layernorm}
J.~L. Ba, J.~R. Kiros, and G.~E. Hinton.
\newblock Layer normalization.
\newblock \emph{arXiv:1607.06450}, 2016.

\bibitem{vanlaarhoven2017}
T.~van Laarhoven.
\newblock L2 regularization versus batch and weight normalization.
\newblock \emph{arXiv:1706.05350}, 2017.

\bibitem{arora2019}
S.~Arora, Z.~Li, and K.~Lyu.
\newblock Theoretical analysis of auto rate-tuning by batch normalization.
\newblock In \emph{International Conference on Learning Representations (ICLR)}, 2019.

\bibitem{lyle2024}
C.~Lyle, Z.~Zheng, K.~Khetarpal, J.~Martens, H.~van Hasselt, R.~Pascanu, and W.~Dabney.
\newblock Normalization and effective learning rates in reinforcement learning.
\newblock In \emph{Advances in Neural Information Processing Systems (NeurIPS)}, 2024.

\bibitem{schuetz2022}
M.~J.~A. Schuetz, J.~K. Brubaker, and H.~G. Katzgraber.
\newblock Combinatorial optimization with physics-inspired graph neural networks.
\newblock \emph{Nature Machine Intelligence}, 4:367--377, 2022.

\bibitem{angelini2023}
M.~C. Angelini and F.~Ricci-Tersenghi.
\newblock Modern graph neural networks do worse than classical greedy algorithms in solving combinatorial optimization problems like maximum independent set.
\newblock \emph{Nature Machine Intelligence}, 5:29--31, 2023.

\bibitem{barabasi1999}
A.-L. Barab\'asi and R.~Albert.
\newblock Emergence of scaling in random networks.
\newblock \emph{Science}, 286(5439):509--512, 1999.

\bibitem{holland1983}
P.~W. Holland, K.~B. Laskey, and S.~Leinhardt.
\newblock Stochastic blockmodels: First steps.
\newblock \emph{Social Networks}, 5(2):109--137, 1983.

\bibitem{eliasof2024}
M.~Eliasof, B.~Bevilacqua, C.-B. Sch\"onlieb, and H.~Maron.
\newblock {GRANOLA}: Adaptive normalization for graph neural networks.
\newblock In \emph{Advances in Neural Information Processing Systems (NeurIPS)}, 2024.

\bibitem{chen2020learngraphnorm}
Y.~Chen, X.~Tang, X.~Qi, C.-G. Li, and R.~Xiao.
\newblock Learning graph normalization for graph neural networks.
\newblock \emph{Neurocomputing}, 493:613--625, 2022.

\bibitem{wang2026scalevec}
M.~Wang, S.~Zhu, Y.~Fang, B.~Li, K.~Shen, and S.~Zhong.
\newblock Negligible in size, significant in effect: On scale vectors in large language models.
\newblock \emph{arXiv preprint arXiv:2605.26895}, 2026.

\bibitem{bouritsas2022gsn}
G.~Bouritsas, F.~Frasca, S.~P. Zafeiriou, and M.~M. Bronstein.
\newblock Improving graph neural network expressivity via subgraph isomorphism counting.
\newblock \emph{IEEE Transactions on Pattern Analysis and Machine Intelligence}, 45(1):657--668, 2022.

\bibitem{braunstein2016network}
A.~Braunstein, L.~Dall'Asta, G.~Semerjian, and L.~Zdeborov\'{a}.
\newblock Network dismantling.
\newblock \emph{Proceedings of the National Academy of Sciences}, 113(44):12368--12373, 2016.

\bibitem{brody2022gatv2}
S.~Brody, U.~Alon, and E.~Yahav.
\newblock How attentive are graph attention networks?
\newblock In \emph{International Conference on Learning Representations}, 2022.

\bibitem{cai2021graphnorm}
T.~Cai, S.~Luo, K.~Xu, D.~He, T.-Y. Liu, and L.~Wang.
\newblock GraphNorm: A principled approach to accelerating graph neural network training.
\newblock In \emph{International Conference on Machine Learning}, 2021.

\bibitem{cappart2023combinatorial}
Q.~Cappart, D.~Ch{\'e}telat, E.~Khalil, A.~Lodi, C.~Morris, and P.~Veli\v{c}kovi\'{c}.
\newblock Combinatorial optimization and reasoning with graph neural networks.
\newblock \emph{Journal of Machine Learning Research}, 24(130):1--61, 2023.

\bibitem{chen2022touplegdd}
T.~Chen, S.~Yan, J.~Guo, and W.~Wu.
\newblock ToupleGDD: A fine-designed solution of influence maximization by deep reinforcement learning.
\newblock \emph{IEEE Transactions on Computational Social Systems}, 2022.

\bibitem{chung1997spectral}
F.~R.~K. Chung.
\newblock \emph{Spectral Graph Theory}.
\newblock American Mathematical Society, 1997.

\bibitem{corso2020pna}
G.~Corso, L.~Cavalleri, D.~Beaini, P.~Li\`{o}, and P.~Veli\v{c}kovi\'{c}.
\newblock Principal neighbourhood aggregation for graph nets.
\newblock In \emph{Advances in Neural Information Processing Systems}, 2020.

\bibitem{dehmamy2019understanding}
N.~Dehmamy, A.-L. Barab\'{a}si, and R.~Yu.
\newblock Understanding the representation power of graph neural networks in learning graph topology.
\newblock In \emph{Advances in Neural Information Processing Systems}, 2019.

\bibitem{hamilton2017graphsage}
W.~L. Hamilton, R.~Ying, and J.~Leskovec.
\newblock Inductive representation learning on large graphs.
\newblock In \emph{Advances in Neural Information Processing Systems}, 2017.

\bibitem{he2016deep}
K.~He, X.~Zhang, S.~Ren, and J.~Sun.
\newblock Deep residual learning for image recognition.
\newblock In \emph{IEEE Conference on Computer Vision and Pattern Recognition}, pages 770--778, 2016.

\bibitem{ioffe2015batchnorm}
S.~Ioffe and C.~Szegedy.
\newblock Batch normalization: Accelerating deep network training by reducing internal covariate shift.
\newblock In \emph{International Conference on Machine Learning}, 2015.

\bibitem{joshi2022learning}
C.~K. Joshi, Q.~Cappart, L.-M. Rousseau, and T.~Laurent.
\newblock Learning the travelling salesperson problem requires rethinking generalization.
\newblock \emph{Constraints}, 27(1--2):70--98, 2022.

\bibitem{karalias2020erdos}
N.~Karalias and A.~Loukas.
\newblock Erd\H{o}s goes neural: An unsupervised learning framework for combinatorial optimization on graphs.
\newblock In \emph{Advances in Neural Information Processing Systems}, 2020.

\bibitem{kempe2003maximizing}
D.~Kempe, J.~Kleinberg, and \'{E}.~Tardos.
\newblock Maximizing the spread of influence through a social network.
\newblock In \emph{Proceedings of the 9th ACM SIGKDD International Conference on Knowledge Discovery and Data Mining}, 2003.

\bibitem{khalil2017learning}
E.~B. Khalil, H.~Dai, Y.~Zhang, B.~Dilkina, and L.~Song.
\newblock Learning combinatorial optimization algorithms over graphs.
\newblock In \emph{Advances in Neural Information Processing Systems}, 2017.

\bibitem{kipf2017gcn}
T.~N. Kipf and M.~Welling.
\newblock Semi-supervised classification with graph convolutional networks.
\newblock In \emph{International Conference on Learning Representations}, 2017.

\bibitem{kool2019attention}
W.~Kool, H.~van Hoof, and M.~Welling.
\newblock Attention, learn to solve routing problems!
\newblock In \emph{International Conference on Learning Representations}, 2019.

\bibitem{liang2023resnorm}
L.~Liang, Z.~Xu, Z.~Song, I.~King, Y.~Qi, and J.~Ye.
\newblock Tackling long-tailed distribution issue in graph neural networks via normalization.
\newblock \emph{IEEE Transactions on Knowledge and Data Engineering}, 2023. (arXiv:2206.08181).

\bibitem{ling2023deepim}
C.~Ling, J.~Jiang, J.~Wang, M.~T. Thai, R.~Xue, J.~Song, M.~Qiu, and L.~Zhao.
\newblock Deep graph representation learning and optimization for influence maximization.
\newblock In \emph{International Conference on Machine Learning}, 2023.

\bibitem{morris2020tudataset}
C.~Morris, N.~M. Kriege, F.~Bause, K.~Kersting, P.~Mutzel, and M.~Neumann.
\newblock TUDataset: A collection of benchmark datasets for learning with graphs.
\newblock \emph{arXiv:2007.08663}, 2020.

\bibitem{press2017weighttying}
O.~Press and L.~Wolf.
\newblock Using the output embedding to improve language models.
\newblock In \emph{Conference of the European Chapter of the Association for Computational Linguistics}, 2017.

\bibitem{rampasek2022gps}
L.~Rampasek, M.~Galkin, V.~P. Dwivedi, A.~T. Luu, G.~Wolf, and D.~Beaini.
\newblock Recipe for a general, powerful, scalable graph transformer.
\newblock In \emph{Advances in Neural Information Processing Systems}, 2022.

\bibitem{salimans2016weightnorm}
T.~Salimans and D.~P. Kingma.
\newblock Weight normalization: A simple reparameterization to accelerate training of deep neural networks.
\newblock In \emph{Advances in Neural Information Processing Systems}, 2016.

\bibitem{sato2020random}
R.~Sato, M.~Yamada, and H.~Kashima.
\newblock Random features strengthen graph neural networks.
\newblock In \emph{SIAM International Conference on Data Mining}, 2020.

\bibitem{ulyanov2016instance}
D.~Ulyanov, A.~Vedaldi, and V.~Lempitsky.
\newblock Instance normalization: The missing ingredient for fast stylization.
\newblock \emph{arXiv:1607.08022}, 2016.

\bibitem{velickovic2018gat}
P.~Veli\v{c}kovi\'c, G.~Cucurull, A.~Casanova, A.~Romero, P.~Li\`o, and Y.~Bengio.
\newblock Graph attention networks.
\newblock In \emph{International Conference on Learning Representations}, 2018.

\bibitem{wu2019demonet}
J.~Wu, J.~He, and J.~Xu.
\newblock DEMO-Net: Degree-specific graph neural networks for node and graph classification.
\newblock In \emph{Proceedings of the 25th ACM SIGKDD International Conference on Knowledge Discovery and Data Mining}, 2019.

\bibitem{xie2023degree}
M.~Xie, X.-X. Zhan, C.~Liu, and Z.-K. Zhang.
\newblock An efficient adaptive degree-based heuristic algorithm for influence maximization in hypergraphs.
\newblock \emph{Information Processing \& Management}, 60(2):103161, 2023.

\bibitem{xiong2020transformerln}
R.~Xiong, Y.~Yang, D.~He, K.~Zheng, S.~Zheng, C.~Xing, H.~Zhang, Y.~Lan, L.~Wang, and T.-Y. Liu.
\newblock On layer normalization in the transformer architecture.
\newblock In \emph{International Conference on Machine Learning}, 2020.

\bibitem{ying2021graphormer}
C.~Ying, T.~Cai, S.~Luo, S.~Zheng, G.~Ke, D.~He, Y.~Shen, and T.-Y. Liu.
\newblock Do transformers really perform badly for graph representation?
\newblock In \emph{Advances in Neural Information Processing Systems}, 2021.

\bibitem{zhang2019rmsnorm}
B.~Zhang and R.~Sennrich.
\newblock Root mean square layer normalization.
\newblock In \emph{Advances in Neural Information Processing Systems}, 2019.

\bibitem{zhao2020pairnorm}
L.~Zhao and L.~Akoglu.
\newblock PairNorm: Tackling oversmoothing in GNNs.
\newblock In \emph{International Conference on Learning Representations}, 2020.

\bibitem{zhou2020dgn}
K.~Zhou, X.~Huang, Y.~Li, D.~Zha, R.~Chen, and X.~Hu.
\newblock Towards deeper graph neural networks with differentiable group normalization.
\newblock In \emph{Advances in Neural Information Processing Systems}, 2020.

\bibitem{zhou2021nodenorm}
K.~Zhou, Y.~Dong, K.~Wang, W.~S. Lee, B.~Hooi, H.~Xu, and J.~Feng.
\newblock Understanding and resolving performance degradation in deep graph convolutional networks.
\newblock In \emph{Proceedings of the 30th ACM International Conference on Information \& Knowledge Management}, 2021.

\end{thebibliography}
\end{document}